\documentclass{article}
\PassOptionsToPackage{varqu}{inconsolata}
\usepackage{colm2024_conference}

\usepackage{latexsym}
\usepackage{fontspec}
\usepackage[utf8]{inputenc}
\usepackage{microtype}
\usepackage{inconsolata}

\usepackage{tikz}
\usepackage{arydshln}
\usepackage{algpseudocode}
\algrenewcommand\textproc{\text}
\usepackage{makecell}
\usepackage{booktabs}
\usepackage{color}
\usepackage{colortbl}
\usepackage{multirow}
\usepackage{enumitem}
\usepackage{makecell}
\usepackage{threeparttable}
\usepackage{amsmath,amsfonts,mathtools}
\usepackage{pifont}
\usepackage{mathrsfs}
\usepackage{float}
\usepackage{graphicx}
\usepackage{xcolor}
\usepackage{enumitem}
\usepackage{tabularx, booktabs}

\usepackage{tcolorbox}
\tcbuselibrary{breakable}

\definecolor{mygreen}{HTML}{00B050}

\definecolor{myorange}{HTML}{ED7D31}
\definecolor{lightgray}{gray}{0.95}
\definecolor{rowgray}{gray}{0.97}
\definecolor{avgblue}{RGB}{210,230,250}  
\definecolor{headergray}{RGB}{160,160,160} 

\usepackage{array}
\usepackage{mathtools}
\usepackage{multirow}
\usepackage{subcaption}
\usepackage{tikz}
\renewcommand{\arraystretch}{1.1}
\definecolor{color1}{cmyk}{0.216,0.176,0,0}
\definecolor{color2}{cmyk}{0.059,0.235,0.392,0}

\usepackage{graphicx}
\usepackage{tikz}
\usepackage{wrapfig}
\usepackage{algorithm}
\usepackage{algpseudocode}
\algrenewcommand\textproc{\text}
\usepackage{makecell}
\usepackage{booktabs}
\usepackage{pifont}
\usepackage{multirow}
\usepackage{enumitem}
\usepackage{balance}
\usepackage{threeparttable}
\usepackage{amsmath,amsfonts,mathtools} 
\usepackage{longtable}
\usepackage{amsthm}
\newtheorem{theorem}{Theorem}
\usepackage{colortbl}
\usepackage{mathrsfs}
\usepackage{float}
\usepackage{graphicx}
\usepackage{hyperref}
\usepackage{tabularx, booktabs}
\usepackage{tikz}
\usepackage{arydshln}
\usepackage{CJKutf8}

\usepackage{bm}        
\usepackage{amsfonts}
\usepackage{cancel}

\definecolor{uc_color}{rgb}{0.99,0.24,0.63}
\definecolor{hc_color}{rgb}{0.02,0.51,0.51}
\definecolor{tc_color}{rgb}{0.99,0.55,0.09}
\definecolor{qwenpurple}{HTML}{5E45D9}

\renewenvironment{abstract}{%
  \vspace{0.8em}%
  \begin{tcolorbox}[breakable,
    colback=qwenpurple!4,
    colframe=qwenpurple!90!black,
    boxrule=0.8pt,
    arc=3mm,
    left=2.2mm,right=2.2mm,top=1.8mm,bottom=1.8mm,
    title=\centering\large\bfseries Abstract,
    coltitle=qwenpurple!90!black,
    colbacktitle=qwenpurple!10]
}{\end{tcolorbox}\vspace{0.8em}}

\tikzstyle{mybox} = [draw=black, very thick,
    rectangle, rounded corners, inner sep=10pt, inner ysep=13pt]
\tikzstyle{fancytitle} =[fill=black, text=white]

\usepackage{amsmath,amsfonts,bm}

\def\eqref#1{equation~\ref{#1}}

\def\1{\bm{1}}

\def\eps{{\epsilon}}

\def\rb{{\textnormal{b}}}

\def\rf{{\textnormal{f}}}

\def\rt{{\textnormal{t}}}

\DeclareMathAlphabet{\mathsfit}{\encodingdefault}{\sfdefault}{m}{sl}
\SetMathAlphabet{\mathsfit}{bold}{\encodingdefault}{\sfdefault}{bx}{n}

\newcommand{\Pmodel}{P_{\rm{model}}}

\newcommand{\KL}{D_{\mathrm{KL}}}
\newcommand{\Var}{\mathrm{Var}}

\usepackage{amsmath,amssymb,amsthm,bm}
\usepackage{mathtools}
\usepackage{booktabs}
\usepackage{multirow}
\usepackage{graphicx}
\usepackage{hyperref}
\usepackage{cleveref}
\usepackage{url}
\usepackage{xcolor}
\usepackage{tikz}
\usetikzlibrary{positioning,arrows.meta,calc,fit,backgrounds,patterns,decorations.pathreplacing,shapes.geometric,matrix}

\newtheorem{proposition}{Proposition}
\newtheorem{corollary}{Corollary}
\newtheorem{lemma}{Lemma}
\newtheorem{assumption}{Assumption}
\theoremstyle{definition}
\newtheorem{definition}{Definition}
\theoremstyle{remark}
\newtheorem{remark}{Remark}

\newcommand{\hist}{\mathbf{H}}                          
\newcommand{\histlog}{\mathbf{H}'}                      
\newcommand{\histlogt}{\histlog_t}             
\newcommand{\histt}{\hist_t}                   
\newcommand{\Hfull}{\hist_T}                   
\newcommand{\hcomp}{\hist^{\mathrm{comp}}}     
\newcommand{\View}{\textsc{View}}              
\newcommand{\editset}{\mathcal{E}}             

\newcommand{\junct}[1]{J_{#1}}                 
\newcommand{\evictset}[1]{E_{#1}}              
\newcommand{\tpath}{\mathbf{p_t}}            
\newcommand{\spath}{\mathbf{p_s}}            
\newcommand{\Ptarget}{P^{\mathrm{target}}}

\newcommand{\policy}{\pi_\theta}
\newcommand{\policyold}{\pi_{\theta_{\mathrm{old}}}}

\newcommand{\sg}{\mathrm{sg}}

\newcommand{\editor}{\mathrm{Edit}}

\let\Pmodel\undefined
\newcommand{\Pmodel}[1][\theta]{P_{#1}}

\let\KL\undefined
\newcommand{\KL}{D_{\mathrm{KL}}}

\DeclareMathOperator*{\Exp}{\mathbb{E}}
\let\Var\undefined
\DeclareMathOperator{\Var}{Var}

\newcommand{\Ltask}{\mathcal{L}_{\mathrm{task}}}
\newcommand{\Lsft}{\mathcal{L}_{\mathrm{SFT}}}

\newcommand{\Lstar}{\mathcal{L}^{\star}}

\newcommand{\Cset}{\mathcal{C}}     
\newcommand{\Treestruct}{\mathcal{T}}

\newcommand{\impratio}{\rho}

\newcommand{\epsKL}{\varepsilon_{\mathrm{KL}}}
\newcommand{\Order}{\mathcal{O}}

\newcommand{\edit}{\editor}                    
\newcommand{\logits}{\mathbf{Logits}}
\let\eps\undefined
\newcommand{\eps}{\varepsilon}

\newcommand{\cmark}{\ensuremath{\checkmark}}

\newcommand{\xmark}{\ensuremath{\times}}

\title{\raisebox{-0.55cm}{\includegraphics[width=1.65cm,height=1.85cm]{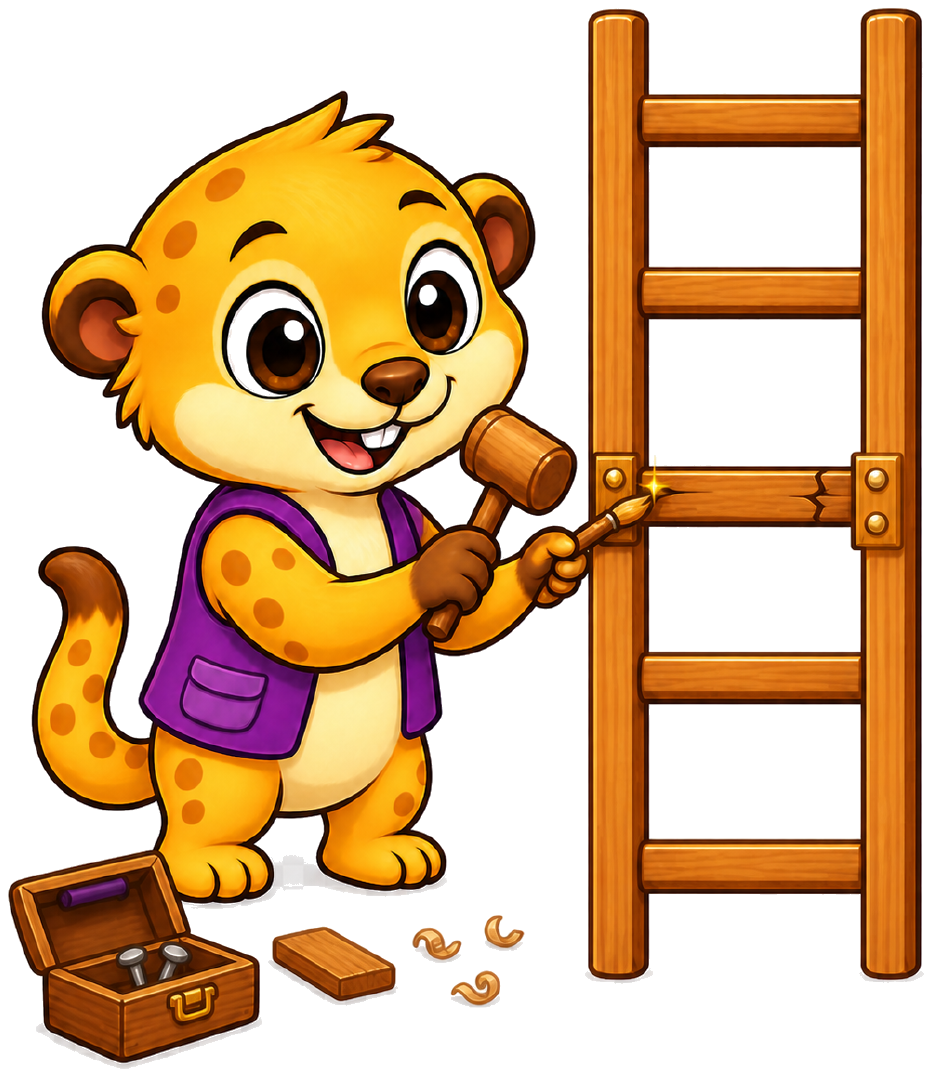}}\hspace{0.16em}%
{\fontsize{13.8}{16}\selectfont MemoryWalker: Stop Training Agents on Contexts They Never Saw}\\[0.3ex]
{\fontsize{10}{12}\selectfont \textcolor{gray}{\textit{``Your Memory-Compressing Harness Makes Training and Inference Inconsistent''}}}}

\author{\bfseries Zinco J, Xunjie Zhu, Shen Huang, Zhenyi Wang, Pengjun Xie, Jieping Ye\\
\includegraphics[height=0.4cm]{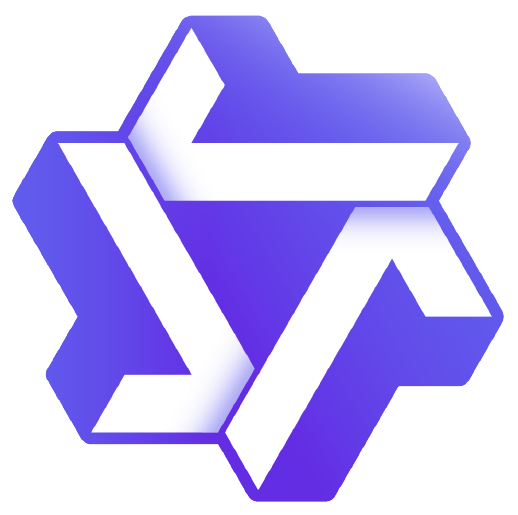} Token Foundry, Alibaba Group}

\begin{document}
\maketitle

\begin{abstract}
Production agent harnesses, including Claude Code and Qwen-Agent, compress an agent's context during rollout to make long-horizon interaction tractable. Training policies under such compression, however, introduces a fundamental conditioning problem: each compression eviction branches the effective interaction history, making the object presented to the learning objective a \textbf{tree} rather than a single sequence. Existing pipelines typically linearize this tree in one of two ways: retaining only the \emph{rightmost root-to-leaf path}, which causes \textbf{time-travel leakage}, or replaying the full \emph{depth-first traversal}, which induces a \textbf{train--inference mismatch}.
Thus, we introduce two exact conditioning-level corrections and prove their gradient equivalence: \textbf{LogitTree}, a segmented $K$-forward traversal of the logits tree, and an equivalent packed \textbf{4D attention mask}. LogitTree requires $K{+}1$ backward passes, whereas the 4D formulation requires both a custom masked-attention kernel and white-box access to eviction records. To avoid these costs, we further propose \textbf{SDCC} (\textbf{\underline{S}}elf-\textbf{\underline{D}}istillation for \textbf{\underline{C}}onditioning \textbf{\underline{C}}onsistency), a training-friendly variational relaxation requiring only a single backward pass. At each eviction junction, SDCC minimizes the forward KL divergence between the compressed student policy and a stop-gradient teacher evaluated on the reconstructed pre-eviction prefix. By Pinsker's inequality, a residual per-junction KL divergence of $\epsKL$ yields an $\Order(\sqrt{\epsKL})$ bound on the resulting train--deployment behavioral gap in total variation. Unlike the exact corrections, SDCC applies without modification to black-box harnesses.
We evaluate our framework with three white-box context editors---TC-RAG, AgentFold, and MemexRL---and two black-box harnesses---Claude Code and OpenCode. Across seven web-search benchmarks, naive compressed-stream training substantially inflates the train--rollout log-probability gap relative to the no-compression floor, with the largest inflation on eviction-heavy batches. In contrast, both exact corrections---LogitTree and the 4D attention mask---remain at the no-compression floor, while SDCC substantially closes the gap. These gains in conditioning consistency are accompanied by lower logit drift and higher rollout rewards than naive compressed-stream training.

\end{abstract}

\section{Introduction}
\label{sec:intro}
\textbf{Long-horizon agents}, including deep-search assistants~\citep{xi2025surveyllmbaseddeepsearch}, coding copilots~\citep{wang2025aiagenticprogrammingsurvey}, browser-based research agents~\citep{shi2025deepresearchsystematicsurvey}, and multi-turn planners~\citep{huang2024understandingplanningllmagents}---can accumulate trajectories spanning tens to hundreds of thousands of tokens. 
Retaining the full interaction history is not only computationally expensive, but can also impair decision making: \ding{182} outdated observations, abandoned plans, and failed attempts introduce substantial contextual noise~\citep{tcrag2024}; \ding{183} state rollback can invalidate the actions and observations associated with a reverted branch~\citep{tcrag2024}; and \ding{184} critical goals and instructions may be diluted or overlooked as the context grows, exacerbating the \emph{lost-in-the-middle} effect~\citep{liu2023lostmiddlelanguagemodels,chen2026cigmeasuringconversationalinformation}. 
Production agent harnesses therefore actively compress and rewrite interaction histories to maintain compact, decision-relevant contexts. For example, Claude Code~\citep{anthropic2024claudecode} and Qwen-Agent~\citep{QwenAgent} maintain bounded context windows and replace evicted prefixes with automatically generated summaries. Other approaches, including TC-RAG~\citep{tcrag2024}, MemexRL~\citep{MemexRL2025} and AgentFold~\citep{agentfold2025,mem1} introduce model-aware memory mechanisms that allow the agent to determine whether, when, what and how to compress, offload, or retrieve information on demand.

Across these context-compression designs, a \textbf{\emph{memory editor}} (hook) rewrites the interaction history into a bounded live context before each decoding step. Although this improves inference-time tractability and decision quality, it creates a subtle problem for post-training (SFT and RL): \textbf{\textit{when the trainer recomputes the log probabilities of rollout tokens, the context under which a token was originally generated may no longer be available}}. 

\textbf{Training on the conditioning tree, not the sequence.}
This missing-context problem arises because harness-level compression breaks the correspondence between token order and conditioning history. At each eviction junction, the original prefix is replaced by a compressed view: tokens preceding the edit were generated under the original prefix, whereas subsequent tokens are generated under its compressed replacement. Repeated edits therefore induce a \emph{tree of conditioning histories}, rather than a single sequence that can be replayed consistently, as illustrated in Figure~1(a).
Existing pipelines~\citep{agentfold2025,mem1,MemexRL2025,tcrag2024} typically serialize this tree in one of two ways. \ding{182} Training on the surviving compressed stream follows only the \emph{rightmost root-to-leaf path}, causing earlier tokens to be reevaluated as if they had been conditioned on summaries created only later---an error we call \emph{time-travel leakage}. \ding{183} Training on the full raw trace instead follows a \emph{depth-first traversal}, causing later tokens to be reevaluated with evicted content that is unavailable during rollout and deployment---a \emph{train--inference conditioning mismatch}. 
\textbf{In both cases, recomputing token logits under the resulting sequence assigns tokens contexts different from those under which they were originally generated, producing substantial train--inference logit inconsistency.}

\textbf{From the conditioning tree to exact solutions.}
To restore the context under which each rollout token was originally generated, we formulate training as a traversal of the conditioning tree and introduce two exact approaches to walk the tree, gradient-equivalent constructions. \ding{182} \textbf{LogitTree} performs a segmented traversal, separately evaluating the $K$ pre-eviction branches, \ding{183} whereas a packed \textbf{4D attention mask} encodes the same branch structure within a single forward pass. As shown in \S\ref{sec:hard}, both constructions yield the same policy gradients and exactly recover the correct rollout conditioning. This exactness, however, comes at a computational or systems cost: LogitTree requires (K{+}1) segmented model evaluations, whereas the packed formulation relies on a custom masked-attention kernel that is difficult to integrate efficiently with fused attention implementations and is restricted to white-box harnesses with access to eviction records.
To recover conditioning consistency without explicitly replaying every branch of the context tree, we further propose \textbf{SDCC} (\textbf{\underline{S}}elf-\textbf{\underline{D}}istillation for \textbf{\underline{C}}onditioning \textbf{\underline{C}}onsistency; \S\ref{sec:soft}). Instead of exactly reconstructing the conditioning context for every rollout token, SDCC trains the policy under the surviving compressed context to match its behavior under the original pre-eviction context. Specifically, at each eviction junction, it minimizes the forward KL divergence from the compressed student policy to a stop-gradient teacher evaluated under the reconstructed pre-eviction prefix. A wake--sleep argument determines the KL direction, and Pinsker's inequality converts a residual per-junction KL divergence of $\epsKL$ into an $\Order(\sqrt{\epsKL})$ bound on the corresponding behavioral discrepancy in total variation. This relaxation requires only a single backward pass per trajectory and extends to black-box harnesses without requiring white-box modification of their attention computation.

\textbf{Contributions.}
Our contributions are fourfold:
\begin{itemize}[leftmargin=*,itemsep=-0.45pt]
\item \textbf{We formalize context compression training as a conditioning-inconsistency problem} (\S\ref{sec:failure_modes}). We prove that the two common rollout compression---\emph{Naive-Compressed} and \emph{Naive-Full}---lead to time-travel leakage and train--inference conditioning mismatch, respectively.

\item \textbf{We propose two exact, gradient-equivalent solutions to walk the context tree} (\S\ref{sec:hard}). \textbf{LogitTree} explicitly traverses the conditioning tree, while a packed \textbf{4D attention mask} represents the same structure in a single forward pass.

\item \textbf{We introduce SDCC, a training-efficient relaxation} (\S\ref{sec:soft}). SDCC distills policy behavior under the original pre-eviction context into the surviving compressed context, requires only one backward pass, and generally applies to white-box and black-box harnesses.

\item \textbf{We validate the framework across compression mechanisms and agent tasks} (\S\ref{sec:experiments}). Across three white-box and two black-box memory editors, on seven web-search benchmarks, our methods reduce train--inference inconsistency and logit drift while improving rollout rewards over naive compressed-stream training.
\end{itemize}

\section{Related Work}
\label{sec:related}
\paragraph{Harness context compression.}
Most production long-horizon agent harnesses include an external editor
that mutates the physical history before the next decoding step.
Claude Code~\citep{anthropic2024claudecode} and
Qwen-Agent~\citep{QwenAgent} maintain a bounded context window and
inject an auto-compact summary in place of the evicted prefix.
MemexRL~\citep{MemexRL2025} and \textsc{MemGPT}~\citep{packer2023memgpt} pages between the live context
and an external store.
TC-RAG and StackPlanner~\citep{tcrag2024,zhang2026stackplannercentralizedhierarchicalmultiagent} manages context as an explicit stack whose
\texttt{pop} evicts the oldest retrieval envelope.
AgentFold and Mem1~\citep{agentfold2025,mem1} periodically fold the
prefix into a self-generated summary.
A complementary line studies working-memory design for agents operating
over long documents, retaining answerable evidence while compacting the
rest of the context~\citep{zhou2026awmanswerableworkingmemory}.
A parallel line trains dedicated compressors, including token-budget
compression with small auxiliary models such as
\textsc{LLM-Lingua}~\citep{jiang2023llmlingua} and
\textsc{RECOMP}~\citep{xu2024recomp}, which act as scheduled deterministic editors.
Across these editor families~\citep{zhang2023h2o,li2024snapkv,xiao2024streamingllm,borgeaud2022retro}, the context available at training time
need not match the live view under which token was originally
decoded. Thus, context the model conditioned on while generating is not, what remains available when training.

\paragraph{Agent RL over edited contexts.}
Recently, policy-gradient RL~\citep{schulman2017ppo,shao2024grpo, zheng2025gspo} and their agentic extensions~\citep{jin2025searchr1,chen2025research,anonymous2025agenticrag} fine-tune LLMs on rollout trajectories, typically treating the edited context as ordinary sequential data. Memory-R1~\citep{yan2026memoryr1enhancinglargelanguage, zhang2026stackplannercentralizedhierarchicalmultiagent} and MemexRL~\citep{MemexRL2025} go further by making memory management itself learnable: the policy can offload spans to external storage and retrieve on demand, thereby optimizing the memory-editing process. However, neither line of work asks \textbf{which context the training gradient should condition on once the rollout history has been rewritten}.
Related notions of train--inference consistency address different failure modes. TITO~\citep{gallouedec2026tito}, whose token-level scheme our LogitTree extends to trajectory trees, studies token-level consistency---whether the token IDs observed during rollout are replayed identically during training---but does not address harness-induced history rewriting.
Our focus is orthogonal: the conditioning context itself changes, which extends block-causal masking for long-context training~\citep{ding2024packing,reid2024gemini,mem1}.

\section{Pitfalls: Two Default Ways to Train on an Edited Rollout}
\label{sec:failure_modes}
Every harness in \S\ref{sec:related} rewrites the agent's context
mid-rollout, posing one question: \textit{\textbf{which context did each token
condition on, and does training match it?}}
An edited rollout is therefore a \emph{trajectory tree}, not a single sequence that can be replayed consistently. Existing pipelines flatten it in one of two ways, which fail in opposite directions yet violate the same token-level conditioning invariant.

\subsection{Train--inference logits drift}
\label{sec:fm:drift}
We first define the \emph{train--inference logits drift} that quantifies this
mismatch. Fix a generated token $y_t$; write
$c_t^{\mathrm{rollout}}$ for the context the policy held when it emitted $y_t$
and $c_t^{\mathrm{train}}$ for the prefix under which the loss later scores it.
The rollout engine logs $\log \Pmodel(y_t \mid c_t^{\mathrm{rollout}})$ at
decode time and the training pass recomputes
$\log \Pmodel(y_t \mid c_t^{\mathrm{train}})$ under the \emph{same} weights;
averaged over the loss-carrying response tokens $\mathcal{M}$, their
disagreement is the \textbf{train--inference logits drift} (\texttt{logdiff}
in \S\ref{sec:experiments}):
\begin{equation}
\mathrm{logdiff} \;=\; \frac{1}{|\mathcal{M}|} \sum_{t \in \mathcal{M}}
\bigl|\, \log \Pmodel(y_t \mid c_t^{\mathrm{train}}) - \log \Pmodel(y_t \mid c_t^{\mathrm{rollout}}) \,\bigr| .
\label{eq:logdiff}
\end{equation}
Sharing one set of weights makes drift immune to sampler staleness: it is
either the small numerical gap between the training and inference kernels
when nothing is compressed, where the two contexts coincide by construction
and the statistic sits at its $0.014$ floor
(\S\ref{sec:experiments}); such residual differences can arise from
floating-point computation and optimization implementation details
\citep{zhang2026precisiontraininginferencemismatchoptimization} --- or a genuine
$c_t^{\mathrm{train}} \neq c_t^{\mathrm{rollout}}$. Drift is thus a test with a
calibrated zero, and the two pitfalls below fail it in opposite directions.

\subsection{Setup: the live view and the trajectory tree}
\label{sec:background}
\label{sec:fm:tree}
We model the agent as a ReAct-style policy~\citep{yao2023react} that generates each action from the physical history
$\histt = (s_0, a_0, o_0, \dots, s_{t-1}, a_{t-1}, o_{t-1})$,
where $s_k$, $a_k$, and $o_k$ denote the thought, tool call, and tool observation at turn $k$, respectively. A memory-editing harness inserts an editor $\textsc{Edit}(\cdot)$ into the rollout loop to perform deletion~\citep{MemexRL2025} or summarization~\citep{agentfold2025,tcrag2024}. Let $\editset_{\le t}$ denote the ordered edits fired by step $t$. The policy conditions on the \textbf{live context view} as:
\begin{equation} \histlogt \;=\; \textsc{View}_t~\!\bigl(\hist_{\le t};\, \editset_{\le t}\bigr), \qquad \editset_{\le t}=\textsc{Edit}(\hist_{\le t}), \qquad a_t \sim \policy(\,\cdot \mid \histlogt\,), \label{eq:live-view} \end{equation}
where $\textsc{View}_t$ applies all edits fired up to step $t$, yielding the partially observable context available to the policy. At the end of the rollout, the harness reports the \emph{final compressed walk}
$\hcomp$. Crucially,
\begin{equation} \underbrace{\View_t~\!\bigl(\hist_{\le t};\, \editset_{\le t}\bigr)}_{\text{live view at decode of } y_t} \;\;\neq\;\; \underbrace{\hcomp[:t]}_{\text{prefix of final walk}}  = \View_T~(\hist_{\le T}; \editset_{\le T}),
\label{eq:live-vs-final} \end{equation}
whenever an edit fires after step $t$: the prefix on the right has been rewritten by edits that did not yet exist when $y_t$ was decoded.

To see why, consider an eviction at position $\junct{k}$ that removes or replaces a span $\evictset{k}$. Tokens generated before $\junct{k}$ were conditioned on a context in which $\evictset{k}$ remained visible, whereas subsequent generation proceeds from the edited context in which $\evictset{k}$ has been removed or replaced. Each edit therefore forks the conditioning history into a \emph{generation-time leg}, which preserves the pre-edit context, and a \emph{compressed spine}, from which the rollout continues. Unrolling $K$ such evictions at positions $\junct{1}<\dots<\junct{K}$, removing spans $\evictset{1},\dots,\evictset{K}$, yields a \emph{trajectory tree} $\Treestruct$ rooted at the initial context. Every $y_t$ is a leaf of this tree: $\tpath(y_t)$ denotes its root-to-leaf generation path, with $\tpath(y_t)=\histlogt$ by definition, while $\spath(y_t)=\hcomp[:t]$ denotes its corresponding prefix in the final compressed walk. \Cref{fig:traj-tree} illustrates a running example with $\evictset{1,2,3}=\{y_5\},\{y_8\},\{y_{11}\}$ at $\junct{1,2,3}=\{7,10,12\}$; the corresponding per-leaf prefixes are reported in \Cref{tab:running-prefixes} and Appendix~\ref{app:tree-notes-fm}.

\begin{figure}[t]
\centering
\begin{minipage}[t]{0.610\linewidth}
\centering
\parbox[t][4ex][t]{\linewidth}{\centering\footnotesize\textbf{(a) Trajectory tree and two flattenings
($K{=}3$ junctions).}}
\resizebox{!}{4.8cm}{%
\begin{tikzpicture}[
  font=\scriptsize,
  tok/.style={draw, line width=0.45pt, rounded corners=1.1pt,
              minimum width=4.6mm, minimum height=3.7mm,
              inner sep=0.4pt, font=\scriptsize},
  query/.style ={tok, fill=blue!10,  draw=blue!60!black},
  shared/.style={tok, fill=black!4,  draw=black!48},
  spine/.style ={tok, fill=green!12, draw=green!50!black},
  evict/.style ={tok, fill=red!7,    draw=red!62!black, densely dashed},
  revive/.style={tok, fill=red!14,   draw=red!62!black, line width=0.6pt},
  answer/.style={tok, fill=green!28, draw=green!55!black,
                 minimum width=5.6mm, font=\scriptsize\bfseries},
  ghost/.style ={tok, draw=none, fill=none, text=black!34},
  sw/.style    ={draw, line width=0.4pt, rounded corners=0.8pt,
                 minimum width=2.6mm, minimum height=2.1mm, inner sep=0pt},
  swquery/.style ={sw, fill=blue!10,  draw=blue!60!black},
  swshare/.style ={sw, fill=black!4,  draw=black!48},
  swspine/.style ={sw, fill=green!12, draw=green!50!black},
  swevict/.style ={sw, fill=red!7,    draw=red!62!black, densely dashed},
  bhead/.style ={font=\scriptsize\bfseries, anchor=west, inner sep=0pt},
  bsub/.style  ={font=\fontsize{6.4}{7.4}\selectfont, text=black!58,
                 anchor=west, inner sep=0pt},
  note/.style  ={font=\fontsize{6.4}{7.4}\selectfont, anchor=west,
                 align=left, inner sep=0pt},
  legend/.style={font=\fontsize{6.4}{6.9}\selectfont, text=black!55,
                 anchor=north west, align=left, inner sep=0pt},
  tlab/.style  ={font=\fontsize{6.4}{7.0}\selectfont, text=black!50,
                 rotate=90, anchor=south, inner sep=0.5pt},
  jlab/.style  ={font=\fontsize{6.4}{7.4}\selectfont,
                 text=orange!58!black, anchor=east, inner sep=0.6pt,
                 fill=white},
  jarr/.style  ={draw=orange!62!black, line width=0.45pt,
                 -{Latex[length=1.1mm]}, shorten >=0.5pt,
                 shorten <=0.5pt},
  jrule/.style ={draw=orange!50!black, line width=0.3pt, densely dashed},
  timeax/.style={draw=black!42, line width=0.4pt, -{Latex[length=1.0mm]}},
  dead/.style  ={draw=black!32, line width=0.3pt},
  skiparc/.style={draw=black!58, line width=0.45pt, densely dotted,
                  -{Latex[length=1.1mm]}, shorten >=0.8pt,
                  shorten <=0.8pt},
  leakarc/.style={draw=red!62!black, line width=0.45pt, densely dashed,
                  -{Latex[length=1.1mm]}, shorten >=0.8pt,
                  shorten <=0.8pt},
]
\def\cx{0.52}      
\def\xa{6.82}      
\def\rA{0}         
\def\rB{-0.72}     
\def\rC{-1.44}     
\def\rD{-2.16}     
\def\yA{-3.04}     
\def\yB{-4.16}     

\node[bhead, text=orange!58!black] (h1) at (-0.23,0.48)
  {Trajectory tree $\Treestruct$};
\node[bsub] at (h1.east)
  {\,: one row per live context $\histlogt$; edit $\junct{k}$ deletes a column};

\draw[timeax] (-0.30,0.17) -- (-0.30,-2.34);
\node[tlab] at (-0.36,-1.08) {rollout time};

\node[query]  at (0*\cx,  \rA) {$s$};
\node[shared] at (1*\cx,  \rA) {$y_1$};
\node[shared] at (2*\cx,  \rA) {$y_2$};
\node[shared] at (3*\cx,  \rA) {$y_3$};
\node[shared] at (4*\cx,  \rA) {$y_4$};
\node[evict]  at (5*\cx,  \rA) {$y_5$};
\node[spine]  at (6*\cx,  \rA) {$y_6$};
\node[spine]  (tA7) at (7*\cx,  \rA) {$y_7$};

\begin{scope}[shift={(7.955*\cx,0.135)}]
  \node[swquery] (kq) at (0,0)         {};
  \node[legend, anchor=west] at (kq.east) {\,prompt};
  \node[swshare] (kc) at (0,-0.235)    {};
  \node[legend, anchor=west] at (kc.east) {\,older context};
  \node[swspine] (kd) at (1.72,0)      {};
  \node[legend, anchor=west] (kdl) at (kd.east) {\,decoded here};
  \node[swevict] (ke) at (1.72,-0.235) {};
  \node[legend, anchor=west] (kel) at (ke.east) {\,decoded, evicted next};
  \node[ghost, minimum width=2.6mm, minimum height=2.1mm,
        font=\fontsize{5.6}{5.6}\selectfont] (kg) at (0,-0.470) {$y_j$};
  \draw[dead] (kg.west) -- (kg.east);
  \node[legend, anchor=west] (kt) at (kg.east)
    {\,no box: gone from the live context};
  \begin{scope}[on background layer]
    \node[draw=black!20, line width=0.3pt, rounded corners=1.2pt,
          fill=white, inner xsep=2pt, inner ysep=1.2pt,
          fit=(kq)(kdl)(kel)(kt)(kg)] {};
  \end{scope}
\end{scope}

\draw[jrule] (-0.06,-0.36) -- (2.25,-0.36);
\node[jlab] at (4.38*\cx,-0.36) {$\junct{1}{=}7$: evict $y_5$};

\node[query]  at (0*\cx,  \rB) {$s$};
\node[shared] at (1*\cx,  \rB) {$y_1$};
\node[shared] at (2*\cx,  \rB) {$y_2$};
\node[shared] at (3*\cx,  \rB) {$y_3$};
\node[shared] at (4*\cx,  \rB) {$y_4$};
\node[ghost]  (gB5) at (5*\cx,  \rB) {$y_5$};
\node[shared] at (6*\cx,  \rB) {$y_6$};
\node[shared] at (7*\cx,  \rB) {$y_7$};
\node[evict]  at (8*\cx,  \rB) {$y_8$};
\node[spine]  at (9*\cx,  \rB) {$y_9$};
\node[spine]  (tB10) at (10*\cx, \rB) {$y_{10}$};
\draw[dead] (gB5.west) -- (gB5.east);
\draw[jarr] (tA7.south) to[out=-96,in=84] (gB5.north);

\draw[jrule] (-0.06,-1.08) -- (3.81,-1.08);
\node[jlab] at (7.38*\cx,-1.08) {$\junct{2}{=}10$: evict $y_8$};

\node[query]  at (0*\cx,  \rC) {$s$};
\node[shared] at (1*\cx,  \rC) {$y_1$};
\node[shared] at (2*\cx,  \rC) {$y_2$};
\node[shared] at (3*\cx,  \rC) {$y_3$};
\node[shared] at (4*\cx,  \rC) {$y_4$};
\node[ghost]  (gC5) at (5*\cx,  \rC) {$y_5$};
\node[shared] at (6*\cx,  \rC) {$y_6$};
\node[shared] at (7*\cx,  \rC) {$y_7$};
\node[ghost]  (gC8) at (8*\cx,  \rC) {$y_8$};
\node[shared] at (9*\cx,  \rC) {$y_9$};
\node[shared] at (10*\cx, \rC) {$y_{10}$};
\node[evict]  at (11*\cx, \rC) {$y_{11}$};
\node[spine]  (tC12) at (12*\cx, \rC) {$y_{12}$};
\draw[dead] (gC5.west) -- (gC5.east);
\draw[dead] (gC8.west) -- (gC8.east);
\draw[jarr] (tB10.south) to[out=-96,in=84] (gC8.north);

\draw[jrule] (-0.06,-1.80) -- (5.37,-1.80);
\node[jlab] at (10.38*\cx,-1.80) {$\junct{3}{=}12$: evict $y_{11}$};

\node[query]  at (0*\cx,  \rD) {$s$};
\node[shared] at (1*\cx,  \rD) {$y_1$};
\node[shared] at (2*\cx,  \rD) {$y_2$};
\node[shared] at (3*\cx,  \rD) {$y_3$};
\node[shared] at (4*\cx,  \rD) {$y_4$};
\node[ghost]  (gD5) at (5*\cx,  \rD) {$y_5$};
\node[shared] at (6*\cx,  \rD) {$y_6$};
\node[shared] at (7*\cx,  \rD) {$y_7$};
\node[ghost]  (gD8) at (8*\cx,  \rD) {$y_8$};
\node[shared] at (9*\cx,  \rD) {$y_9$};
\node[shared] at (10*\cx, \rD) {$y_{10}$};
\node[ghost]  (gD11) at (11*\cx, \rD) {$y_{11}$};
\node[shared] at (12*\cx, \rD) {$y_{12}$};
\node[answer] at (\xa,    \rD) {$y_{\mathrm{ans}}$};
\draw[dead] (gD5.west)  -- (gD5.east);
\draw[dead] (gD8.west)  -- (gD8.east);
\draw[dead] (gD11.west) -- (gD11.east);
\draw[jarr] (tC12.south) to[out=-96,in=84] (gD11.north);

\node[bhead, text=red!58!black] (h2) at (-0.23,\yA+0.44)
  {Pitfall A: Naive-Compressed};
\node[bsub] at (h2.east)
  {\,: scores \emph{every} token on the last row $\hcomp$ --- too
   \emph{short}};

\node[query]  (a0)  at (0*\cx,  \yA) {$s$};
\node[shared] (a1)  at (1*\cx,  \yA) {$y_1$};
\node[shared] (a2)  at (2*\cx,  \yA) {$y_2$};
\node[shared] (a3)  at (3*\cx,  \yA) {$y_3$};
\node[shared] (a4)  at (4*\cx,  \yA) {$y_4$};
\node[ghost]  (ag1) at (5*\cx,  \yA) {$y_5$};
\node[spine]  (a6)  at (6*\cx,  \yA) {$y_6$};
\node[spine]  (a7)  at (7*\cx,  \yA) {$y_7$};
\node[ghost]  (ag2) at (8*\cx,  \yA) {$y_8$};
\node[spine]  (a9)  at (9*\cx,  \yA) {$y_9$};
\node[spine]  (a10) at (10*\cx, \yA) {$y_{10}$};
\node[ghost]  (ag3) at (11*\cx, \yA) {$y_{11}$};
\node[spine]  (a12) at (12*\cx, \yA) {$y_{12}$};
\node[answer] (aya) at (\xa,    \yA) {$y_{\mathrm{ans}}$};
\draw[dead] (ag1.west) -- (ag1.east);
\draw[dead] (ag2.west) -- (ag2.east);
\draw[dead] (ag3.west) -- (ag3.east);

\draw[skiparc] (a4.south) to[out=-62,in=-118] (a6.south);
\node[note, text=black!62] at (6.85*\cx,\yA-0.32)
  {$y_6,y_7$ scored as if $y_5$ had never existed};

\node[bhead, text=red!58!black] (h3) at (-0.23,\yB+0.44)
  {Pitfall B: Naive-Full};
\node[bsub] at (h3.east)
  {\,: revives every evicted column, tape $\Hfull$ --- too
   \emph{long}};

\node[query]  (b0)  at (0*\cx,  \yB) {$s$};
\node[shared] (b1)  at (1*\cx,  \yB) {$y_1$};
\node[shared] (b2)  at (2*\cx,  \yB) {$y_2$};
\node[shared] (b3)  at (3*\cx,  \yB) {$y_3$};
\node[shared] (b4)  at (4*\cx,  \yB) {$y_4$};
\node[revive] (b5)  at (5*\cx,  \yB) {$y_5$};
\node[spine]  (b6)  at (6*\cx,  \yB) {$y_6$};
\node[spine]  (b7)  at (7*\cx,  \yB) {$y_7$};
\node[revive] (b8)  at (8*\cx,  \yB) {$y_8$};
\node[spine]  (b9)  at (9*\cx,  \yB) {$y_9$};
\node[spine]  (b10) at (10*\cx, \yB) {$y_{10}$};
\node[revive] (b11) at (11*\cx, \yB) {$y_{11}$};
\node[spine]  (b12) at (12*\cx, \yB) {$y_{12}$};
\node[answer] (bya) at (\xa,    \yB) {$y_{\mathrm{ans}}$};

\draw[leakarc] (b5.south) .. controls +(0,-0.34) and +(0,-0.34) ..
  (b9.south);
\node[note, text=red!55!black, align=left] at (9.62*\cx,\yB-0.40)
  {$y_9,y_{10}$ still attend to\\$y_5$, dead since $\junct{1}$};

\end{tikzpicture}}
\end{minipage}%
\hfill
\hspace{-10pt}
\begin{minipage}[t]{0.375\linewidth}
\centering
\parbox[t][4ex][t]{\linewidth}{\centering\scriptsize\textbf{(b) Per-token $\Delta\log p$ on untrained Qwen3-4B,
three harnesses on one axis.}}
\vspace{5pt}
\includegraphics[height=4.8cm]{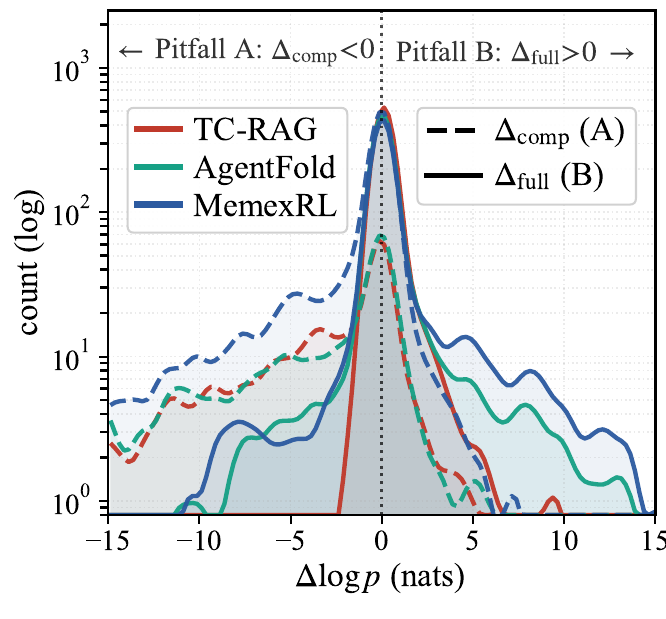}%
\end{minipage}

\caption{\small{\textbf{The two pitfalls.}
\textbf{(a)} The trajectory tree $\Treestruct$ of the $K{=}3$ running
example, drawn as the staircase of \emph{live contexts}: row $k$ lists
every token visible during layer $k$, so a token decoded there
conditions on exactly what lies to its left in its own row, i.e.\ on
its live view $\histlogt=\tpath(y_t)$ (box fills are keyed in the
legend). Between two rows an edit fires --- $\evictset{1}{=}\{y_5\}$ at
$\junct{1}{=}7$, $\evictset{2}{=}\{y_8\}$ at $\junct{2}{=}10$,
$\evictset{3}{=}\{y_{11}\}$ at $\junct{3}{=}12$ (orange arrows) --- so
every later row carries a hole at the evicted column. The two default
flattenings collapse this staircase in opposite directions: strip~A
scores every token on the last row alone, the walk $\hcomp$, so
$y_6,y_7$ are scored on a prefix that already skips $y_5$ (no box at
that column); strip~B revives every evicted column into the tape $\Hfull$,
leaving $y_5,y_8,y_{11}$ visible to tokens decoded long after their
eviction (red arc).
\textbf{(b)} \emph{Train--inference logit inconsistency} for each
generated token of an untrained Qwen3-4B, overlaid across three
white-box harnesses. $\Delta_{\mathrm{comp}}$ (dashed, left tail) is
the logit difference when training replays a token using the final
compressed context, which is missing information present at inference;
$\Delta_{\mathrm{full}}$ (solid, right tail) is the difference when
training replays the full physical trace, which contains information
unavailable at inference. Their per-harness means have opposite signs,
$\mu_{\mathrm{comp}}\in[-4.0,-1.6]$ versus
$\mu_{\mathrm{full}}\in[+0.2,+0.7]$; full aggregation is in
\Cref{tab:q1} (\S\ref{app:q1-full}).}}
\label{fig:traj-tree}%
\label{fig:q1_dist}%
\end{figure}

\subsection{The two pitfalls and the conditioning invariant}
\label{sec:fm:comp}
\label{sec:fm:full}
\label{sec:fm:symmetry}
\label{sec:invariant}

Two sequence representations of $\Treestruct$ are natural, and neither preserves the generation-time path of every token: \ding{182} the final compressed walk $\hcomp$ keeps only the compressed spine, \ding{183} the full physical trace $\Hfull$ concatenates every generated token in decoding order, a depth-first traversal. The former applies edits too early, the latter keeps edited-away content too long.

\paragraph{\ding{182} Pitfall A: Naive-Compressed follows the final compressed path
(time-travel leakage).}
Naive-Compressed trains directly on $\hcomp$. Consequently, each surviving target $y_t$ is scored under its prefix in the final compressed walk, $\spath(y_t)$, rather than under the live view $\tpath(y_t)=\histlogt$ from which it was generated:
\begin{equation}
  \spath(y_t) \;\subsetneq\; \tpath(y_t) \;=\; \histlogt \;=\ \textsc{RootToRightmostLeaf}(\Treestruct).
  \label{eq:pitfall-A-symptom}
\end{equation}
The mismatch arises because $\spath(y_t)$ may already reflect edits that fired only after $y_t$ was decoded. Future compression is thus propagated backward in time, a failure mode we call \emph{time-travel leakage}. In the running
example, this scores $y_6, y_7$ as if $y_5$ never existed --- though
the model conditioned on $y_5$. On an untrained Qwen3-4B, one such
leaf contributes $\Delta_{\mathrm{comp}} = -22.8$ nats to the
per-token log-probability gap; SFT/RL views appear in
Appendix~\ref{app:tree-notes-fm}.

\paragraph{\ding{183} Pitfall B: Naive-Full trains on the depth-first traversal
(train--inference mismatch).}
Naive-Full instead trains on $\Hfull$, which preserves every generated token in decoding order. A target generated after an eviction is therefore reevaluated under the physical prefix $\hist_{<t}$, which may still contain spans that had already been removed from the live view:
\begin{equation}
\textsc{DepthFirstTraversal}(\Treestruct) \;=\ \hist_{<t} \;\supsetneq\; \tpath(y_t) \;\supseteq\; \spath(y_t).
  \label{eq:pitfall-B-symptom}
\end{equation}
Thus, $y_9, y_{10}$ are scored with $y_5, y_8$ still in the prefix even
though each was already evicted; the model learns to exploit signal
unavailable at deployment, with credit-assignment error accumulating
linearly with $K$ (App.~\ref{app:tree-notes-fm},
Eq.~\ref{eq:step-level-bias}). A matched leaf contributes
$\Delta_{\mathrm{full}} = +18.5$ nats: opposite in sign to Pitfall~A,
with comparable magnitude. We refer to this failure as \emph{stale-context leakage}: training exposes the model to information unavailable during rollout and deployment.

\paragraph{\ding{184} Mirror symmetry and the conditioning invariant.}
Pitfall~A conditions on too little context, whereas Pitfall~B conditions on too much (\Cref{tab:bad-symmetry}, App.~\ref{app:tree-notes-fm}). Accordingly, on an untrained Qwen3-4B, $\Delta_{\mathrm{comp}}$ is predominantly negative and $\Delta_{\mathrm{full}}$ predominantly positive across three white-box harnesses (\Cref{fig:q1_dist}). Under eviction-heavy compression, the Naive-Compressed train--inference log-probability gap reaches $0.366$ on AgentFold, ${\approx}26\times$ the no-compression floor of $0.014$ (\S\ref{sec:experiments}).
Both failures violate the same requirement. In the notation of \S\ref{sec:fm:drift}, the loss prefix is $c_t^{\mathrm{train}}=\spath(y_t)$ for Pitfall~A and $\hist_{<t}$ for Pitfall~B, against $c_t^{\mathrm{rollout}}=\histlogt$ in both cases.

\begin{definition}[Conditioning consistency]
\label{def:conditioning-invariant}
A training pipeline is \emph{conditioning-consistent} if, at every step $t$ in every trajectory, using only the edits fired through step $t$, never future ones:
\begin{equation} \;\forall t,\quad c_t^{\mathrm{train}} \;\equiv\; c_t^{\mathrm{rollout}} \;\equiv\; \View_t~\!\bigl(\hist_{\le t};\, \editset_{\le t}\bigr) \;=\; \histlogt\; \tag{\textit{Conditioning Invariant}}, \label{eq:conditioning-invariant} \end{equation}
\end{definition}

\begin{proposition}[Failure of naive training]
\label{prop:naive-violation}
$\hcomp[:t]$ yields Pitfall~A, while $\hist_{<t}$ yields Pitfall~B; neither matches the live view $\histlogt$ from which $y_t$ was decoded. Whenever a nontrivial edit changes the prefix used to evaluate some target $y_t$, both Naive-Compressed and Naive-Full violate the conditioning invariant. \S\ref{sec:hard} gives two exact (bias-zero) fixes, while \S\ref{sec:soft} gives a soft $\Order(\sqrt{\epsKL})$ one.
\end{proposition}

\section{Exact Solutions: Two Ways to Walk the Context Tree}
\input{figures/fig_method_overview}
\label{sec:hard}
To eliminate the conditioning errors introduced by the two pitfalls, we propose two exact solutions that enforce conditioning-invariance. Across the $K{+}1$ branches of the conditioning tree $\Treestruct$, each target $y_t$ is scored on the unique branch whose prefix matches its inference-time live view, $\tpath(y_t)=\histlogt$ (App.~\ref{app:logits-tree}). Both methods therefore recover the correct training context for every target and incur zero conditioning bias. \ding{182} \textbf{LogitTree} (\S\ref{sec:hard:tree}) explicitly materializes and evaluates the branches as separate sequences, \ding{183} whereas the \textbf{4D attention mask} (\S\ref{sec:hard:4d}) packs the same branches into a single forward pass. As established in \Cref{thm:hard_equiv}, they implement the \emph{same tree traversal} at different levels of the training stack (\Cref{fig:method-overview}).

\subsection{\ding{182} LogitTree materializes branches as separate sequences}
\label{sec:hard:tree}
\textbf{LogitTree} (ours; extending the token-level scheme of
\citep{gallouedec2026tito} to trajectory trees) decomposes
$\Treestruct$ into its $K{+}1$ root-to-leaf branches and processes each
branch as a separate sequence:
\begin{equation}
\{\histlog^\pi\} = \textsc{RootToLeaf}(\Treestruct), \qquad
\logits_\pi \;=\; f_\theta(\{\histlog^\pi\}),
\qquad \pi \in \{1, \dots, K+1\},
\label{eq:logittree-def}
\end{equation}
where $\{\histlog^\pi\}$ denotes the sequence on branch $\pi$ and logits is extracted from LLMs  $f(\cdot)$ parameterized by $\theta$. For each
target $y_t$, LogitTree selects the unique branch whose prefix equals
the live view under which $y_t$ was generated,
$\tpath(y_t)=\histlogt$. The loss for $y_t$ is therefore computed from
that branch only, so its training-time context exactly matches its
rollout-time context. In implementation this uniqueness is enforced by a
per-branch \emph{token loss mask}: a token is unmasked on exactly one of
the $K{+}1$ branches --- the one whose prefix is its live view --- and
masked out at every other occurrence, so each token contributes its
policy gradient exactly once and the shared-trunk tokens, which are
physically replicated across all branches, are never counted $K{+}1$
times.

LogitTree places the $K{+}1$ legs in physically disjoint sequences, so standard causal attention restricts each target $y_t$ to its predecessors on the same leg---namely, $\tpath(y_t)$---with no attention across legs. 
In the running example
($\junct{1,2,3}=7,10,12$), the pre-$\junct{1}$ leg containing
$y_6$ and $y_7$ retains $y_5$, which was visible when they were decoded. In contrast, the compressed-spine leg containing $y_{\mathrm{ans}}$ excludes
$y_5$, $y_8$, and $y_{11}$, all of which had already been evicted. Thus, neither time-travel leakage nor stale-context leakage has a valid attention pathway.
Shared-trunk KV caching keeps aggregate memory well below $(K{+}1)\times$; compute stays $\approx (K{+}1)\times$. Per-branch SFT/RL gradient equivalences are in App.~\ref{app:proofs} (\Cref{prop:sft_tree,prop:rl_tree}; the branch layout for the running example is drawn in \Cref{fig:method-overview}(b)); the construction is backbone-agnostic --- it consumes only the per-leg live views the harness already exposes (App.~\ref{app:impl-tradeoffs}).

\subsection{\ding{183} The 4D attention mask packs all branches into one sequence}
\label{sec:hard:4d}
Performing $K{+}1$ backward passes per rollout is prohibitively expensive in deep-search settings. We therefore ask whether the entire conditioning tree can be traversed in a single masked forward pass. Our realization follows the 4D attention-mask formulation used for stack-memory agent trajectories in AgenticRag-R1~\citep{jiang2026agenticrag}: we keep $\Hfull$ packed as one sequence while restricting each query token to attend only to tokens on its own root-to-leaf branch.
Let $\bigl[\tau(j),
\omega(j)\bigr)$ be token $j$'s logical lifetime ($\omega(j) = +\infty$ if
never evicted), we define
\begin{equation}
M^{\mathrm{logical}}_{i,j}
  = \begin{cases}
    0,        & \tau(j)\le i < \omega(j),\\
    -\infty,  & \text{otherwise,}
  \end{cases}
  \qquad M_{i,j} = M^{\mathrm{causal}}_{i,j} \oplus M^{\mathrm{logical}}_{i,j},
  \label{eq:4d-mask-def}
\end{equation}
The causal mask $M^{\mathrm{causal}}$ prevents attention to future tokens, while
$M^{\mathrm{logical}}$ further removes attention edges between
different branches of the conditioning tree. We also reassign position
ids so that the tokens visible to each query form a gap-free logical
sequence~\citep{mem1}. As a result, each query $y_i$ attends to exactly the
root-to-leaf path of $\Treestruct$ under which it was generated, allowing
all branches to be evaluated in a single forward pass
(\Cref{fig:method-overview}(c)).
Throughout, both
materializations share one standing conditioning invariant proposition --- and gradient equivalence
(\Cref{prop:sft_4d}, App.~\ref{app:proofs}), which holds under five
conditions: (i)~exact-visibility encoding, (ii)~row-wise position
reassignment, (iii)~decoupled loss/attention masks, (iv)~no
cross-mask normalization, and (v)~rollout and training sharing the same
$M^{\mathrm{logical}}$. However, linear attention, sparse/top-$k$ selection,
and vLLM/SGLang violate the premise or one of (i)--(v)
(App.~\ref{app:impl-tradeoffs}).

\subsection{\ding{184} The centerpiece: 4D masks and LogitTree are the same walk}
\label{sec:hard:equiv}

\begin{theorem}[Materialization equivalence]
\label{thm:hard_equiv}
Under dense softmax attention, the packed 4D mask and the LogitTree $K$-forward decomposition implement the same traversal of the conditioning tree and compute identical training gradients. Both compute the sum over branches of the per-branch gradient in \Cref{prop:sft_tree} (App.~\ref{app:proofs}).
\end{theorem}

\paragraph{Cost and deployment splits.}
LogitTree requires $K{+}1$ backward passes per rollout, resulting in a $5$--$20\times$ wall-clock overhead in deep-search settings, whereas the 4D formulation requires only one. The 4D formulation, however, imposes a stronger systems requirement: it requires more larger attention-mask and white-box access to both the harness and the model. Implementing~\eqref{eq:4d-mask-def} requires the harness to expose the eviction records that determine $[\tau(j),\omega(j))$, as well as the model to accept a custom $(Batchsize,NumHeads,T,T)$ attention mask. Neither capability is generally available through black-box APIs such as Claude Code or OpenCode. Moreover, constructing and applying a per-query 4D mask has high memory and computational complexity, making it particularly unfriendly for training with very long contexts.
By contrast, LogitTree requires only the live views exposed during rollout and can reconstruct the $K{+}1$ branch inputs by concatenating the corresponding per-leg prompts. Its main limitation is therefore computational rather than architectural.
The concrete 4D realization used in
\S\ref{sec:experiments} is deferred to App.~\ref{app:impl-4d}, a
packaging within the equivalence class of~\eqref{eq:4d-mask-def} that
inherits \Cref{thm:hard_equiv}.

\section{SDCC: Self-Distillation for Conditioning Consistency}
\label{sec:soft}

A compressed rollout leaves two trajectories of the \emph{same interaction}:
the \emph{live rollout trajectory}, whose contexts produced each token, and
the \emph{final compressed replay trajectory} $\hcomp$, which remains after
all edits and is used by Naive-Compressed for training. They agree until a
later edit rewrites a prefix. SDCC aligns their next-token distributions at
exactly those divergences: the student under the compressed replay prefix
matches a stopped-gradient teacher under the original live prefix. It aligns
conditional policy distributions rather than the two text sequences.

The exact methods in \Cref{sec:hard} achieve this alignment by scoring every
target on the branch on which it was generated. They are exact, but require
either $K{+}1$ backward passes or a custom attention kernel.
\emph{Self-Distillation for Conditioning Consistency} (SDCC) is the cheaper
alternative: it retains the usual single backward pass on the compressed
walk and aligns its predictions with the original live trajectory only at
the diverging leaves. SDCC is therefore an approximation to the exact walks,
not another way to materialize the tree.

\subsection{Aligning the live and compressed trajectories}
\label{sec:soft:method}
\label{sec:soft:form}

For a diverging response token $y_p$, write
\begin{equation}
  z_p := \tpath(y_p), \qquad x_p := \spath(y_p) \subsetneq z_p,
  \label{eq:sdcc-contexts}
\end{equation}
where $z_p$ is the live prefix from which $y_p$ was originally decoded and
$x_p$ is its prefix in the final compressed walk. We call $y_p$ a
\emph{diverging leaf}, and let $\Cset$ collect all such leaves in a
trajectory; each pair $(z_p,x_p)$ is one local alignment pair. In the
running example, if $y_5$ is evicted only
after $y_6$ was decoded, then $z_6$ contains $y_5$ whereas $x_6$ does not.
Naive-Compressed trains $y_6$ on $x_6$; the exact methods would instead
score it on $z_6$.

SDCC does not put $z_p$ back into the gradient-carrying training sequence.
Instead, it uses the policy evaluated on $z_p$ as a stopped-gradient
\emph{teacher}, and the policy evaluated on $x_p$ as the gradient-carrying
\emph{student}. Here ``teacher'' does not mean a second model: it is the
current policy with gradients stopped. Thus, at a diverging leaf, SDCC
teaches the compressed-context policy to reproduce the next-token
distribution it had under the original live context (\Cref{fig:method-overview}(d)).

The procedure has three steps. First, run the ordinary task forward pass on
the final compressed walk $\hcomp$; this is the student pass and is the only
pass through which gradients flow. Second, reconstruct the original live
prefix for each diverging leaf and evaluate it without gradients. For the
common case in which one eviction $\evictset{\mathrm{act}}(p)$ is active at
$p$, with $\iota(\evictset{\mathrm{act}}(p))$ its original insertion slot,
the reconstruction is
\begin{equation}
  \tpath(y_p) =
  \underbrace{\hcomp[:\iota(\evictset{\mathrm{act}}(p))]}_{\text{shared trunk}}
  \oplus
  \underbrace{\evictset{\mathrm{act}}(p)}_{\text{re-inserted span}}
  \oplus
  \underbrace{\hcomp[\iota(\evictset{\mathrm{act}}(p)) : p]}_{\text{post-junction suffix}}.
  \label{eq:teacher-reconstruction}
\end{equation}
With overlapping evictions, SDCC re-inserts all spans active at $p$;
Appendix~\ref{app:var:remarks} gives the general construction. Third, add a
KL penalty only at those diverging leaves:
\begin{equation}
  \mathcal{L}_{\mathrm{SDCC}} =
  \underbrace{\Ltask(\theta;\hcomp)}_{\text{ordinary compressed-stream task loss}}
  + \lambda \sum_{p \in \Cset}
  \KL\!\left(
    \underbrace{\Ptarget(\cdot\mid\tpath(y_p))}_{\text{live-context teacher; stop-grad}}
    \,\middle\|\,
    \underbrace{\Pmodel(\cdot\mid\spath(y_p))_p}_{\text{compressed-context student; gradient}}
  \right),
  \label{eq:soft_loss}
\end{equation}
where
$\Ptarget(\cdot\mid\tpath(y_p))=f_{\sg(\theta)}(\tpath(y_p))$ and
$\Pmodel(\cdot\mid\spath(y_p))_p=f_\theta(\hcomp)_p$. The KL is
\emph{forward}: the teacher distribution appears first, so it acts as a
fixed soft target and requires the student to cover outcomes that are likely
under the original live context.

This construction makes the computational trade-off explicit. The student
uses the same compressed input as Naive-Compressed, and the teacher legs are
gradient-free; SDCC therefore needs one backward pass per trajectory. The
penalty is \emph{leaf-gated}: if no context differs, $\Cset$ is empty and
the loss is exactly the usual task loss. Likewise, setting $\lambda=0$
recovers Naive-Compressed exactly, making it SDCC's matched baseline.
Because $\Hfull$ never enters the student input, SDCC does not introduce
Naive-Full's stale-context leakage. It instead softens Naive-Compressed's
time-travel error by matching distributions rather than by exactly replaying
every branch.

The variational derivation of the forward direction, including its
wake--sleep interpretation and the role of $\lambda$, is deferred to
Appendix~\ref{app:variational}. In brief, the forward KL has the same
zero set as the conditioning gap induced by the variational objective, while
giving a stopped-gradient cross-entropy update for the student.

\subsection{What the residual KL guarantees}
\label{sec:soft:bound}

Let $\varepsilon_p := \KL\!\left(
\Ptarget(\cdot\mid\tpath(y_p)) \,\middle\|\,
\Pmodel(\cdot\mid\spath(y_p))_p\right)$ be the residual left by
\Cref{eq:soft_loss} at diverging leaf $p$.

\begin{proposition}[Behavioral Pinsker bound on the conditioning gap]
\label{prop:pinsker-bound}
For every diverging leaf $p$,
\[
\left\|\pi_\theta(\cdot\mid\tpath(y_p))
      -\pi_\theta(\cdot\mid\spath(y_p))\right\|_{\mathrm{TV}}
\le \sqrt{\varepsilon_p/2}.
\]
\end{proposition}

Thus the regularizer controls a behavioral quantity, not merely a logit
diagnostic: the smaller the teacher--student KL at a junction, the smaller
the difference between their next-action distributions. If
$\varepsilon_p=0$, SDCC matches the exact walk's next-token distribution at
that leaf; for nonzero residual it remains an explicitly bounded
approximation. Appendix~\ref{app:variational} proves the proposition and
further gives the variational derivation, a policy-gradient-bias bound,
conditions for a zero-KL solution, and the convergence analysis.

\section{Experiments}
\label{sec:experiments}
We validate the proposed conditioning-inconsistency framework across
multiple model scales, white-box and black-box agent harnesses, and
multiple logit recomputation schemes. 
Across these settings, we consistently observe that harness-level
context editing introduces a measurable train--inference mismatch,
which manifests as elevated logit drift and degraded rollout reward
under naive training. In contrast, conditioning-consistent
recomputation substantially reduces this mismatch and improves
downstream agent performance.

\subsection{Experimental Setup}
\label{sec:setup}
We validate conditioning inconsistency across different backbone
models, agent harnesses, and training-time logit recomputation
schemes.
\par\medskip
\noindent\mbox{\ding{182} \textbf{Backbone Models.} We} use Qwen-family models, with Qwen3-4B and Qwen3.7-Air as the main
backbones.

\ding{183} \textbf{Harnesses.}
For white-box context editors, we use three representative
learnable-action editors:
\textbf{TC-RAG}~\citep{tcrag2024}, which emits \verb|pop()|;
\textbf{AgentFold}~\citep{agentfold2025}, which folds spans past a
length threshold; and
\textbf{MemexRL}~\citep{MemexRL2025}, which offloads spans to an
external store.
A \textbf{Search-R1} no-compression control emits tool calls but never
evicts, anchoring the no-compression drift floor. Knobs and
per-method inputs are given in
\S\ref{app:harness-knobs}--\ref{app:method-inputs}.
For black-box production-style harnesses, we instrument their HTTPS
traffic, recover the per-turn rendered context prefixes, and construct
the corresponding LogitTree views from these observed prefix states.
We evaluate two such harnesses: \textbf{Claude Code} and
\textbf{OpenCode}.

\ding{184} \textbf{Logit Recomputation Schemes.}
We compare replay on the final compressed stream
(\textbf{Naive-Compressed}), replay on the full physical trace
(\textbf{Naive-Full}), two exact conditioning-consistent replay
methods (\textbf{4D-Mask} and \textbf{LogitTree}), and our proposed
single-backward approximation \textbf{SDCC}.

\ding{185} \textbf{Datasets.}
\emph{Train:} a pooled composite-QA corpus of $81{,}638$ instances drawn
from \textsc{RedSearcher}~\citep{chu2026redsearcherscalablecostefficientframework}
and the \textsc{ASearcher} agentic-search training set, whose multi-hop
structure reliably induces long-horizon interaction and frequent context
editing.
\emph{Test:}
\textsc{NQ}~\citep{kwiatkowski2019nq},
\textsc{TriviaQA}~\citep{joshi2017triviaqa},
\textsc{HotpotQA}~\citep{yang2018hotpotqa},
\textsc{2Wiki}~\citep{ho20202wiki},
\textsc{MuSiQue}~\citep{trivedi2022musique},
\textsc{Bamboogle}~\citep{press2023selfask}, and
\textsc{FRAMES}~\citep{krishna2024frames}
($38{,}270$ questions per checkpoint, using the same live pipeline as
training).

\ding{186} \textbf{Search Infrastructure.}
All rollouts use live web tools rather than a local retrieval server.
Web search is backed by the online DashScope text-search API, which
returns up to 10 ranked results per query. For full-page evidence
extraction, we use a two-stage pipeline: retrieved pages are first
fetched via Firecrawl and then summarized by Qwen3-Turbo.

\ding{187} \textbf{Evaluation Metrics.}
We report the following quantities, each with a distinct role.
(1)~\textbf{Recorded-token \texttt{logdiff}}
$=
\big|\log\pi_{\mathrm{train}}(y|c^{\mathrm{train}}) -
\log\pi_{\mathrm{rollout}}(y|c^{\mathrm{rollout}})\big|$
is a per-token conditioning-fidelity diagnostic rather than an
efficacy metric; 
(2)~\textbf{Full-distribution conditioning KL}
$\mathrm{KL}(\pi(\cdot|H^{\mathrm{teacher}})\,\|\,
\pi(\cdot|H^{\mathrm{student}}))$
is the population-level counterpart, used in
\S\ref{sec:exp-q1} to measure the direction and magnitude of the
conditioning mismatch.
(3)~\textbf{SDCC training KL}
$\mathrm{KL}_{\mathrm{SDCC}}$
(\S\ref{sec:exp-grand-matrix}, \S\ref{app:sdcc-dynamics})
is SDCC's leaf-gated consistency regularizer, and is not intended as a
cross-method comparator.
(4)~\textbf{Rollout reward} measures the behavioral consequence of
conditioning mismatch during training and rollout.
(5)~\textbf{Compression depth}
$\mathrm{depth}=\mathbb{E}[c_i\,|\,c_i\!>\!0]$, for $c_i$ the fraction
of rendered context evicted in rollout episode~$i$, is how much an
editor removes when it fires. Its companion factor
$\mathrm{freq}=\Pr[c_i\!>\!0]$ is tabulated per cell in
\S\ref{app:compression-metrics} instead of here: what triggers a firing
is harness-specific, so $\mathrm{freq}$ compares down a column, not across.
(6)~Agentic \textbf{EM} is the primary downstream task metric and
determines the final correctness ranking in
\S\ref{sec:exp-q2}. Full definitions are given in
\S\ref{app:hparams}.

\ding{188} \textbf{Optimization.}
All cells are trained with \textbf{GRPO}: group size $G=16$, learning
rate $10^{-6}$, KL coefficient $\beta_{\mathrm{KL}}=10^{-3}$, and a
live-compression rollout pipeline, so training and deployment use the
same rollout setting. Each cell retains at least $400$ steps; a
frozen dev split of $1{,}325$ questions, disjoint from every test
benchmark, shortlists candidates for full-suite evaluation, and each
reported EM is the best full-suite result among the cell's fully
evaluated checkpoints (\S\ref{app:hparams}).

\subsection{Grand result matrix (overview)}
\label{sec:exp-grand-matrix}

\paragraph{Tree-consistent methods recover the drift floor; the two
Naives inflate with eviction.} Under live model-triggered compression
(\texttt{memory\_offload}, \texttt{pop} emitted by the policy;
AgentFold folds when the rendered context is $>$3k tokens), the
no-compression control sits at $\texttt{logdiff}=0.014$. The
discriminating quantity is the \emph{slope} against eviction frequency
rather than the level: where a cell barely evicts, every method sits at
the floor, baselines included. From each method's least- to its
most-evicting harness, Naive-Compressed moves $30.5\times$
($0.012\!\to\!0.366$ for freq $2.3\!\to\!28.6\%$) and Naive-Full
$9.7\times$ ($0.021\!\to\!0.203$, freq $6.5\!\to\!41.5\%$), whereas 4D
spans a comparably wide frequency range ($3.8\!\to\!49.2\%$) and moves
only $3.7\times$, ending at $0.022$; SDCC moves $6.5\times$ to $0.071$
and LogitTree is flat at $0.012$--$0.013$. In the two non-folding
editors all six tree-consistent cells sit at $0.006$--$0.013$, at or
below the control (\S\ref{sec:invariant}).
\Cref{tab:grand-matrix} reports the full training matrix.
\textsc{n/a}~= undefined.
\texttt{reward} and \texttt{logdiff} are read at each cell's
argmax-reward iteration and compression at run level, by one recipe
applied to every cell. Semantics: \S\ref{app:logdiff-semantics}.
For the untrained seven-benchmark diagnostic, compression depth is the
mean fraction removed conditional on an editor firing. ``Max tok.'' is the
largest single-request input length observed during rollout, rendered with
the same chat template and tool schema used by the adapter. The Base and
Search-R1 controls have depth $0.0$ by construction.

\begin{table*}[t]
\centering
\scriptsize
\setlength{\tabcolsep}{3.2pt}
\renewcommand{\arraystretch}{0.98}
\caption{\textbf{Grand result matrix (Qwen3-4B, live compression, real
web tools).} Five training methods across three white-box editors, plus
two black-box deployed agents. \texttt{logdiff} is measured at each
cell's maximum-reward rollout iteration, as the mean and max within
that iteration rather than over the run. Per-cell
compression frequencies are in \S\ref{app:compression-metrics}. For the
two black-box harnesses, context management is internal and eviction
spans are not exposed, so compression depth is \textsc{n/a}; their
observed maximum context lengths are reported, and junctions are recovered indirectly
(\S\ref{app:impl-blackbox}).}
\label{tab:grand-matrix}
\resizebox{\textwidth}{!}{%
\begin{tabular}{ll|c|cc|cc|ccccccc|c}
\toprule
& & reward $\uparrow$
  & \multicolumn{2}{c|}{\texttt{logdiff} $\downarrow$}
  & \multicolumn{2}{c|}{compression}
  & \multicolumn{7}{c|}{Agentic EM per bench (\%) $\uparrow$}
  & \\
Method & Harness & max
       & mean & max
       & depth (\%) & max tok.
       & NQ & TQA & HQA & 2Wiki & MSQ & Bamb. & Frames
       & avg EM $\uparrow$ \\
\midrule
\multicolumn{15}{c}{\textit{\textbf{Qwen3-4B (without RL)}}} \\
\midrule
\multirow{2}{*}{No-Compression}
  & Base
    & \textsc{n/a} & \textsc{n/a} & \textsc{n/a} & 0.0 & 289
    & 3.2 & 21.6 & 8.4 & 2.6 & 2.5 & 23.2 & 8.0 & 9.9 \\
  & Search-R1
    & \textsc{n/a} & \textsc{n/a} & \textsc{n/a} & 0.0 & 1{,}806
    & 2.3 & 22.2 & 8.7 & 2.5 & 2.7 & 18.4 & 8.4 & 9.3 \\
\midrule
\multirow{5}{*}{Compression}
  & TC-RAG
    & \textsc{n/a} & \textsc{n/a} & \textsc{n/a} & 2.00 & 2{,}117
    & 3.1 & 23.2 & 9.2 & 3.1 & 3.9 & 16.0 & 9.7 & 9.7 \\
  & AgentFold
    & \textsc{n/a} & \textsc{n/a} & \textsc{n/a} & 4.78 & 1{,}971
    & 2.0 & 21.3 & 7.3 & 2.5 & 2.9 & 13.6 & 6.1 & 7.9 \\
  & MemexRL
    & \textsc{n/a} & \textsc{n/a} & \textsc{n/a} & 12.58 & 1{,}670
    & 1.2 & 15.8 & 5.5 & 1.1 & 1.7 & 12.0 & 7.3 & 6.4 \\
  & Claude Code
    & \textsc{n/a} & \textsc{n/a} & \textsc{n/a} & \textsc{n/a} & 33{,}804
    & 22.1 & 50.4 & 23.2 & 29.8 & 7.4 & 37.6 & 13.7 & 26.3 \\
  & OpenCode
    & \textsc{n/a} & \textsc{n/a} & \textsc{n/a} & \textsc{n/a} & 34{,}285
    & 31.0 & 55.7 & 26.6 & 32.7 & 8.4 & 36.6 & 12.1 & 29.0 \\
\midrule
\multicolumn{15}{c}{\textit{\textbf{Qwen3-4B (with RL)}}} \\
\midrule
No-compression & Search-R1
  & 0.697 & \textbf{0.014} & 0.015 & 0.0 & 36{,}370
  & 19.8 & 57.6 & 19.9 & 11.1 & 7.9 & 32.0 & 12.7 & 23.0 \\
\midrule
\multirow{3}{*}{Naive-Full}
  & TC-RAG
    & 0.838 & 0.197 & 0.335 & \textbf{60.5} & 37{,}044
    & 30.1 & 64.4 & 28.2 & 34.1 & 12.9 & 52.0 & 22.8 & 34.9 \\
  & AgentFold
    & 0.600 & 0.203 & 0.333 & 78.2 & 14{,}190
    & 25.6 & 58.8 & 24.0 & 18.7 & 9.1 & 40.8 & 17.1 & 27.7 \\
  & MemexRL
    & 0.734 & 0.021 & 0.051 & 43.2 & \textbf{65{,}630}
    & 29.2 & 63.1 & 28.0 & 31.3 & 9.7 & 48.0 & 15.2 & 32.1 \\
\midrule
\multirow{3}{*}{Naive-Compressed}
  & TC-RAG
    & 0.850 & 0.015 & 0.041 & 54.3 & 38{,}171
    & 30.9 & 64.3 & 31.9 & 46.0 & 11.8 & 50.4 & 21.5 & 36.7 \\
  & AgentFold
    & \textbf{0.734} & 0.366 & 1.095 & 68.4 & \textbf{16{,}708}
    & 30.3 & 64.1 & 25.4 & 24.8 & 12.6 & 56.0 & 18.9 & 33.2 \\
  & MemexRL
    & \textbf{0.750} & 0.012 & 0.014 & 55.0 & 37{,}101
    & 23.9 & 60.5 & 22.9 & 20.5 & 9.7 & 47.2 & 17.6 & 28.9 \\
\midrule
\multirow{3}{*}{\textbf{4D mask (ours)}}
  & TC-RAG
    & 0.800 & \textbf{0.006} & \textbf{0.013} & 44.6 & \textbf{40{,}182}
    & 29.4 & 60.5 & 30.2 & 44.9 & 13.0 & 58.4 & 23.7 & 37.2 \\
  & AgentFold
    & 0.632 & 0.022 & 0.091 & 79.7 & 15{,}383
    & 33.4 & 66.0 & 31.9 & 43.9 & 11.1 & 37.6 & 15.8 & 34.2 \\
  & MemexRL
    & 0.416 & 0.013 & 0.017 & 45.4 & 36{,}528
    & 33.7 & 65.9 & 31.2 & 36.2 & 11.7 & 40.8 & 14.3 & 33.4 \\
\midrule
\multirow{5}{*}{\textbf{LogitTree (ours)}}
  & TC-RAG
    & 0.914 & 0.012 & \textbf{0.013} & 53.0 & 39{,}743
    & 36.2 & 65.6 & 36.1 & 54.1 & 13.9 & 55.2 & 19.3 & 40.1 \\
  & AgentFold
    & 0.629 & \textbf{0.012} & \textbf{0.018} & \textbf{83.3} & 16{,}595
    & \textbf{34.2} & \textbf{67.9} & 34.7 & 46.8 & 12.5 & 46.4 & 19.8 & 37.5 \\
  & MemexRL
    & 0.715 & 0.013 & 0.013 & \textbf{64.1} & 40{,}485
    & \textbf{37.1} & 69.1 & \textbf{42.0} & \textbf{61.5} & \textbf{19.1} & \textbf{64.0} & \textbf{28.5} & \textbf{45.9} \\
  & Claude Code
    & \textbf{0.839} & 0.016 & 0.016 & \textsc{n/a} & 65{,}167
    & 34.8 & 37.0 & 26.5 & 30.8 & \textbf{21.7} & \textbf{58.8} & \textbf{42.0} & 35.9 \\
  & OpenCode
    & \textbf{0.929} & \textbf{0.012} & \textbf{0.012} & \textsc{n/a} & 73{,}589
    & \textbf{34.5} & \textbf{63.2} & 33.7 & 39.2 & 16.0 & 42.8 & 15.9 & 35.0 \\
\midrule
\multirow{5}{*}{\textbf{SDCC (ours)}}
  & TC-RAG
    & \textbf{0.922} & 0.013 & 0.015 & 57.8 & 37{,}590
    & \textbf{37.8} & \textbf{70.2} & \textbf{42.9} & \textbf{63.8} & \textbf{17.7} & \textbf{60.2} & \textbf{29.6} & \textbf{46.0} \\
  & AgentFold
    & 0.713 & 0.071 & 0.156 & 73.3 & 9{,}416
    & 33.9 & 67.1 & \textbf{39.7} & \textbf{57.9} & \textbf{14.4} & \textbf{58.4} & \textbf{25.1} & \textbf{42.4} \\
  & MemexRL
    & \textbf{0.750} & \textbf{0.011} & \textbf{0.012} & 50.3 & 36{,}364
    & 37.0 & \textbf{69.2} & 39.5 & 56.9 & 15.6 & 59.2 & 24.5 & 43.1 \\
  & Claude Code
    & 0.648 & \textbf{0.015} & \textbf{0.015} & \textsc{n/a} & \textbf{82{,}779}
    & \textbf{37.2} & \textbf{67.7} & \textbf{34.1} & \textbf{40.7} & 20.8 & 42.8 & 19.5 & \textbf{37.5} \\
  & OpenCode
    & 0.717 & 0.013 & 0.013 & \textsc{n/a} & \textbf{76{,}807}
    & 32.7 & 62.3 & \textbf{35.4} & \textbf{41.5} & \textbf{17.0} & \textbf{47.7} & \textbf{21.5} & \textbf{36.9} \\
\bottomrule
\end{tabular}%
}
\end{table*}

\paragraph{Structural pattern of \Cref{tab:grand-matrix}.} Exact
walks recover the floor (LogitTree $0.012$--$0.013$, 4D
$0.006$--$0.022$; No-comp anchor $0.014$). The Naives separate on the
same axis: Naive-Compressed reaches $0.366$ on AgentFold, ${\approx}26\times$
the floor. Where a 4D maximum exceeds
LogitTree's it does so at the floor itself ($0.017$ vs.\ $0.013$ on
MemexRL), which does not contradict \Cref{thm:hard_equiv}: the theorem
equates gradients \emph{on the same batch}, whereas each row is an
independently trained run; the two \emph{means} --- the quantity the
invariant constrains --- agree at $0.013$.
SDCC's $\mathrm{KL}_{\mathrm{SDCC}}$ scales with teacher--student
mismatch, spanning three orders of magnitude across harnesses
($10^{-4}$ on MemexRL up to $0.356$ on AgentFold).
Every white-box trained row exceeds the Search-R1 no-compression
control at EM $23.0$ (vs.\ untrained-Base $9.9$). For the black-box
harnesses, SDCC reaches $37.5$ on Claude Code and $36.9$ on OpenCode.

\subsection{Q1: The two pitfalls fail in opposite directions}
\label{sec:exp-q1}

\paragraph{Both pitfalls exhibit their predicted signs on an untrained
model.} On an untrained Qwen3-4B, Naive-Comp under-places
($\Delta_{\mathrm{comp}}\!<\!0$) and Naive-Full-leak over-places
($\Delta_{\mathrm{full}}\!>\!0$) across all three white-box harnesses,
with the predicted sign holding for a majority of tokens in every
cell --- structural rather than
RL-absorbable. Per token,
$\Delta_{\mathrm{comp}} = \log p(y_t|\hcomp[:t]) -
\log p(y_t|\tpath(y_t))$ (diverging leaves) and
$\Delta_{\mathrm{full}} = \log p(y_t|c^{\mathrm{leak}}_t) -
\log p(y_t|\hcomp[:t])$ (post-eviction; $c^{\mathrm{leak}}_t$ reinjects
Naive-Full's training-only content); magnitudes are ranked
$\text{TC-RAG}\!\approx\!\text{AgentFold}\!\gg\!\text{MemexRL}$ for
$\Delta_{\mathrm{comp}}$ and are reversed for $\Delta_{\mathrm{full}}$, interpreted
at the sign level only. \Cref{fig:q1_dist}(b); \S\ref{app:q1-full}.

\subsection{Q2: SDCC $\approx$ LogitTree $\approx$ 4D $\gg$ both Naives}
\label{sec:exp-q2}

\paragraph{On the anchor cell, the exact walks pin the floor and
SDCC's \texttt{logdiff} decreases over the logged window.} On the
low-eviction MemexRL anchor at 4B, both exact walks sit at the no-comp
floor by construction (LogitTree $0.0133$, 4D $0.0140$ vs.\ Search-R1 $0.0135$;
Naive-Comp $0.0237$; cost multipliers in \Cref{tab:main}). On the
eviction-heavy AgentFold cell, SDCC's mean \texttt{logdiff} falls by
$55\%$ over the logged window and crosses below the same-window
Naive-Comp trace, which shows no trend. That comparison is on one
editor at one scale and we attach no confidence interval to it;
whether these conditioning-side
movements translate into learning is decided by EM, not
\texttt{logdiff}. Naive-Compressed is the $\lambda\!=\!0$ limit of
SDCC's objective: setting $\lambda\!=\!0$ reduces the loss to the
standard GRPO objective on the compressed walk $\hcomp$ (same
underlying policy-loss implementation), so the SDCC~vs.~Naive-Comp
comparison is the natural matched-baseline ablation of the junction
KL. \Cref{tab:main} and \Cref{fig:sdcc-convergence} close the argument;
training dynamics appear in \S\ref{app:sdcc-dynamics}.

\subsection{Q3: Cross-harness (white-box and black-box)}
\label{sec:exp-q3}

\paragraph{White-box ordering is editor-dependent.} The gap between the exact
walks and Naive-Comp tracks how much the editor rewrites. It is widest
on length-triggered AgentFold, where LogitTree and 4D fall far below
Naive-Comp ($0.012$/$0.022$ vs.\ $0.366$); it narrows on TC-RAG and
closes on the low-eviction MemexRL cell, where all five methods lie
within $0.011$--$0.021$ and no method is separated from the floor.
Naive-Full is elevated on the two eviction-heavy harnesses, opposite to
Pitfall~B. Per-cell values are in \Cref{tab:grand-matrix}.

\paragraph{Black-box transfer follows the same task-level ordering.}
On Claude Code, SDCC reaches $37.5$ average EM, above LogitTree's $35.9$;
on OpenCode, the corresponding values are $36.9$ and $35.0$.
The observed logit-drift values remain small in both harnesses
(Claude Code: $0.015$ for SDCC versus $0.016$ for LogitTree; OpenCode:
$0.013$ versus $0.012$). Because the two harnesses do not expose their
internal evictions, we compare them by downstream EM and observed maximum
context length rather than compression depth; the latter is therefore
reported as \textsc{n/a} in \Cref{tab:grand-matrix}.

\subsection{Large-scale evaluation on WideSearch}
\label{sec:exp-widesearch}

We additionally evaluate on the \textsc{WideSearch} benchmark~\citep{widesearch2025}
with the \textbf{Claude Code} harness. This setting is distinct from the
Qwen3-4B matrix: it uses \textbf{Qwen3.7-Air}, is trained on
\textbf{256 NVIDIA H100 GPUs}, and evaluates $200$ questions with $4$
independent trials per question ($800$ executions per model). We compare an
SFT checkpoint (\textbf{Base}) against \textbf{LogitTree} after $15$
training steps. Pass@1 is the success rate of the first trial
(``\texttt{\_1}''), while Pass@4 counts a question as successful if any of
its four trials succeeds.

\begin{table}[!htbp]
\centering
\small
\caption{\textbf{WideSearch results with Claude Code (200
questions, four trials each).} Pass@1 uses trial \texttt{\_1}; Pass@4
credits any successful trial. ``Row'' evaluates matching output rows and
``Item'' evaluates individual answer items. P/R denote precision/recall;
all Row/Item metrics are means over $800$ trial-level verifier outputs.}
\label{tab:widesearch}
\resizebox{\linewidth}{!}{%
\begin{tabular}{lcccccccc}
\toprule
Method & Pass@1 & Pass@4 & Row Precision & Row Recall & Row F1 & Item Precision & Item Recall & Item F1 \\
\midrule
Base (SFT) & 4.5 & 7.0 & 42.50 & 37.91 & 38.99 & 71.53 & 63.57 & 65.41 \\
LogitTree (15 steps) & \textbf{5.0} & \textbf{10.0} & \textbf{49.88} & \textbf{43.64} & \textbf{45.27} & \textbf{77.39} & \textbf{67.48} & \textbf{69.99} \\
Relative improvement & \textbf{+11.1\%} & \textbf{+42.9\%} & \textbf{+17.4\%} & \textbf{+15.1\%} & \textbf{+16.1\%} & \textbf{+8.2\%} & \textbf{+6.2\%} & \textbf{+7.0\%} \\
\bottomrule
\end{tabular}
}
\end{table}

LogitTree improves Pass@1 from $9/200$ ($4.5\%$) to $10/200$ ($5.0\%$),
or a relative $11.1\%$ improvement; Pass@4 rises from $14/200$ ($7.0\%$)
to $20/200$ ($10.0\%$), a relative $42.9\%$ improvement. \emph{Row
precision} is the fraction of predicted output rows that match a reference
row, whereas \emph{row recall} is the fraction of reference rows recovered;
they rise by $17.4\%$ and $15.1\%$, respectively. \emph{Item precision}
and \emph{item recall} apply the same definitions to individual answer
items, rising by $8.2\%$ and $6.2\%$. Thus the row- and item-level F1
scores improve by $16.1\%$ and $7.0\%$, respectively. These verifier
metrics complement binary success by giving partial credit for structured
multi-item answers.

\section{Conclusion and Future Work}
\label{sec:conclusion}
Viewing a compressed rollout as a \emph{conditioning tree} reveals a
distinct train--inference mismatch in editor-based agent RL. The two
default replay schemes fail in opposite directions: Naive-Compressed
conditions tokens on prefixes that are too short, while Naive-Full
conditions them on prefixes that are too long. Across different models,
white-box and black-box harnesses, and multiple logit recomputation
schemes, we show that this mismatch appears as measurable logit drift
and degraded rollout reward.
To address it, we present two exact conditioning-consistent solutions,
\textbf{LogitTree} and the \textbf{4D attention mask}, together with a
training-efficient approximation, \textbf{SDCC}. Exact replay restores
the no-compression drift floor, while SDCC substantially narrows the
gap with a single-backward objective.
Future work includes extending the framework to broader compression
policies (e.g., dropping tool observations or using latent summaries),
testing across a broader range of harnesses, and adapting exact tree-consistent
replay to sparse and linear-attention architectures.

\paragraph{Provenance and acknowledgment.}
Our white-box implementation builds on the \textsc{Slime} example stack~\citep{slime_github},
which already provided a TITO-style~\citep{gallouedec2026tito} loss hook
sketching the per-turn segmented objective---one forward per response
segment, scored against a reconstructed rollout-time prefix. We reuse
that scaffold and its interfaces, and thank its authors. What this paper
adds is the conditioning-tree formulation, the equivalence between
explicit branch materialization and the packed 4D mask, the SDCC
relaxation, the rollout-side splitting that makes exact replay
executable, and the measurements of \Cref{tab:grand-matrix}---not the
idea of replaying tokens under the prefix that produced them.

\bibliography{iclr2026_conference}
\bibliographystyle{colm2024_conference}

\appendix

\section*{Part I \; Position in the agentic-RL literature}
\addcontentsline{toc}{section}{Part I: Position in the agentic-RL literature}
\label{app:part-intro}
\label{app:position}
At the time of writing, published baselines for harness-RL training on
QA-style benchmarks (Search-R1 \citep{jin2025searchr1}, ReSearch
\citep{chen2025research}) all train on the physical trajectory $\histt$.
AgenticRag-R1 \citep{jiang2026agenticrag} and MEM1 \citep{mem1} have begun
exploring RL with stack-based or learned memory mechanisms for long-horizon
agents. Recent work in the long-horizon coding-agent literature (e.g., agentic
SWE-bench solvers) has also informally reported that training on edited
transcripts is harmful. To our knowledge, our work is the first to formalize
the conditioning invariant, characterize the failure modes that follow from
violating it, and propose a soft remedy with a provable bias bound.

\section*{Part II \; Method supplements: pitfalls, structure, proofs,
variational derivation, convergence}
\addcontentsline{toc}{section}{Part II: Method supplements}
\label{app:part-method}

\noindent This part supplies the technical scaffolding behind the
development of \S\ref{sec:failure_modes}--\S\ref{sec:soft}:
a worked tool-using rollout in which a single eviction places two
target spans on opposite sides of it, so that both pitfalls and the
mirror symmetry between them can be read off one concrete trajectory
(\S\ref{app:weather-case}); the trajectory-tree object itself
(\S\ref{app:logits-tree}); every proof stated in
\S\ref{sec:hard}--\S\ref{sec:soft}, restated so the appendix reads
self-contained (\S\ref{app:proofs}); and the end-to-end $\beta$-VAE
derivation of SDCC{} with its convergence analysis
(\S\ref{app:variational}, \S\ref{app:sdcc-convergence}).

\section{A worked example: one weather lookup, two pitfalls}
\label{app:weather-case}

\S\ref{sec:failure_modes} states both pitfalls over abstract tokens
$y_t$; here both sit on one rollout with one eviction, so the mirror
symmetry of \Cref{tab:bad-symmetry} is read off rather than derived.
\textbf{The rollout is constructed for exposition: it is not a logged
trajectory, and no quantity in it is a measurement.}

\paragraph{Setup.} TC-RAG's learnable \texttt{pop}
(\S\ref{app:impl-tcrag}) deletes an envelope outright, leaving no
summary behind to carry the forecast forward. Asked about an umbrella
for tomorrow's meeting in Shanghai, the policy calls
\texttt{get\_weather}, receives envelope $o_1$, and writes span $S$ ---
``\emph{rain is likely tomorrow afternoon --- bring an
umbrella}'' --- \textbf{while $o_1$ is in view}. The budget is then
exceeded, \texttt{pop} evicts the largest envelope $o_1$ at junction
$\junct{1}$ ($\evictset{1} = o_1$), and a later \texttt{get\_transit}
turn produces span $T$ \textbf{after the eviction}. Below, \emph{rollout} is the live
view a target was decoded under and \emph{training} the constructed
prefix the objective scores it under; the one row where they disagree
is the whole failure.

\subsection{\ding{182} Pitfall A on $S$: the model is taught that it
knows the weather}
\label{app:weather-pitfall-a}

\begin{tcolorbox}[breakable, colback=red!3, colframe=red!55!black,
  fonttitle=\bfseries, boxrule=0.7pt, arc=2pt,
  left=5pt, right=5pt, top=3pt, bottom=3pt, boxsep=2pt,
  title={Naive-Compressed, time-travel leakage:
         \textnormal{\itshape ``I already know the weather.''}}]
Naive-Compressed trains on the final compressed walk $\hcomp$, so $S$
is scored under $\spath(S)$: $\junct{1}$ fires \emph{after} $S$ is
decoded, yet its effect is already applied
(Eq.~\ref{eq:pitfall-A-symptom}, $\spath(S) \subsetneq \tpath(S)$).

\smallskip
{\footnotesize\renewcommand{\arraystretch}{1.05}%
\begin{tabular}{@{}l@{\hspace{0.7em}}l@{\hspace{1.6em}}c@{\hspace{1.1em}}c@{}}
& & \itshape rollout & \itshape training \\
& & $\tpath(S)$ & $\spath(S)$ \\
\midrule
\texttt{\textcolor{blue!55!black}{[user]}} & \texttt{umbrella for tomorrow's Shanghai meeting?} & \cmark & \cmark \\
\texttt{\textcolor{gray!60!black}{[call]}} & \texttt{get\_weather("Shanghai", +1d)} & \cmark & \cmark \\
\texttt{\textcolor{red!55!black}{[obs~]}} & \texttt{\textcolor{red!55!black}{weather envelope (afternoon rain)}} & \cmark & \textcolor{red!55!black}{\xmark} \\
\texttt{\textcolor{green!45!black}{[tgt~]}} & \texttt{$S$: "...rain tomorrow afternoon..."} & \multicolumn{2}{c}{\normalfont\itshape scored here} \\
\bottomrule
\end{tabular}}

\smallskip
\noindent The tool \emph{call} survives, only its \emph{result} is
gone, so $-\log \Pmodel(S \mid \spath(S))$ demands the forecast from a
prefix that no longer holds it; the only way to cut the term is to move
mass onto ``rain tomorrow afternoon'' as a \emph{prior over Shanghai
weather}. A tool-grounded assertion becomes a from-memory one, and the
deployed model states the forecast before reading the tool.
\textbf{Too little context.}
\end{tcolorbox}

\subsection{\ding{183} Pitfall B on $T$: the model is taught to read a
discarded context}
\label{app:weather-pitfall-b}

\begin{tcolorbox}[breakable, colback=red!3, colframe=red!55!black,
  fonttitle=\bfseries, boxrule=0.7pt, arc=2pt,
  left=5pt, right=5pt, top=3pt, bottom=3pt, boxsep=2pt,
  title={Naive-Full, stale-context leakage:
         \textnormal{\itshape ``I compressed that --- why is it still
         in my context?''}}]
Naive-Full trains on the depth-first tape $\Hfull$, which deletes
nothing, so $T$ is scored under the physical prefix $\hist_{<t}$, which
still holds the $o_1$ its live view had lost
(Eq.~\ref{eq:pitfall-B-symptom}, $\hist_{<t} \supsetneq \tpath(T)$).

\smallskip
{\footnotesize\renewcommand{\arraystretch}{1.05}%
\begin{tabular}{@{}l@{\hspace{0.7em}}l@{\hspace{1.6em}}c@{\hspace{1.1em}}c@{}}
& & \itshape rollout & \itshape training \\
& & $\tpath(T)$ & $\hist_{<t}$ \\
\midrule
\texttt{\textcolor{blue!55!black}{[user]}} & \texttt{umbrella for tomorrow's Shanghai meeting?} & \cmark & \cmark \\
\texttt{\textcolor{gray!60!black}{[call]}} & \texttt{get\_weather("Shanghai", +1d)} & \cmark & \cmark \\
\texttt{\textcolor{red!55!black}{[obs~]}} & \texttt{\textcolor{red!55!black}{weather envelope (afternoon rain)}} & \textcolor{red!55!black}{\xmark} & \cmark \\
\texttt{\textcolor{gray!60!black}{[span]}} & \texttt{$S$: "...rain tomorrow afternoon..."} & \cmark & \cmark \\
\texttt{\textcolor{gray!60!black}{[pop~]}} & \texttt{$\junct{1}$: $\evictset{1}$ = weather envelope} & \cmark & \cmark \\
\texttt{\textcolor{gray!60!black}{[turn]}} & \texttt{get\_transit(...) $\rightarrow$ transit envelope} & \cmark & \cmark \\
\texttt{\textcolor{green!45!black}{[tgt~]}} & \texttt{$T$: "Line 2 runs every four minutes..."} & \multicolumn{2}{c}{\normalfont\itshape scored here} \\
\bottomrule
\end{tabular}}

\smallskip
\noindent The gradient therefore rewards reading the forecast back
\emph{after} the \texttt{pop} --- a move the deployed agent cannot
make, because there the deletion is real. \textbf{Too much context.}
\end{tcolorbox}

\paragraph{The mirror.} On row \texttt{[obs] weather envelope},
Pitfall~A reads $(\cmark,\ \xmark)$ and Pitfall~B reads
$(\xmark,\ \cmark)$: A drops the forecast from a target that used it, B
keeps it for one that had lost it, so no reweighting of a single
serialization repairs both. \S\ref{sec:fm:symmetry} measures this ---
$\Delta_{\mathrm{comp}} = -22.8$ against $\Delta_{\mathrm{full}} =
+18.5$ nats on an untrained Qwen3-4B, opposite in sign and comparable
in magnitude. Those two numbers are measurements; the weather rollout
is not. Both defects vanish under one repair: score every target under
the live view it was decoded from
(Definition~\ref{def:conditioning-invariant}).

\section{The trajectory tree: a unified structural analysis}
\label{app:logits-tree}

This appendix collects, in one place, the tree-theoretic view
underlying every construction in the main text. \S\ref{sec:failure_modes}
introduces the running-example tree $\Treestruct$; \S\ref{sec:invariant}
states the conditioning invariant on its leaves; \S\ref{sec:hard} and
\S\ref{sec:soft} materialize that invariant with a 4D mask, with LogitTree,
and with SDCC. The purpose here is to make the object $\Treestruct$
explicit as a graph, count its components (junctions, leaves, branches),
identify where each method places gradient, and read off the compute
budget from the counts.

\subsection{Definition of the trajectory tree}

\begin{definition}[Trajectory tree]
\label{def:logits-tree}
Fix a rollout with physical trajectory $\Hfull = (x_1,\dots,x_T)$ and
eviction records $\{(J_k, \evictset{k})\}_{k=1}^K$, where the $k$-th
compression fires at physical position $J_k$ and removes the positions
$\evictset{k} \subset \{1,\dots,J_k\}$. The trajectory tree
$\Treestruct = (V, E, \phi)$ has a \emph{root} (the initial user
prompt), a \emph{junction} $\junct{k}$ for each $k$, and a \emph{leaf}
for each generated token $y_t$; edges are labeled by physical-token
spans, so following edges from the root to a node $v$ recovers the
physical prefix that produced $v$. Two colorings
$\phi_{\mathrm{spine}},\, \phi_{\mathrm{leg}}: V \to 2^{\{1,\dots,T\}}$
record the two conditioning sets of interest:
$\phi_{\mathrm{spine}}(v)$ is the prefix visible on the final
compressed walk, with \emph{every} eviction applied, and
$\phi_{\mathrm{leg}}(v)$ is the live view of \eqref{eq:live-view} ---
the prefix visible when $v$ was decoded, carrying exactly the evictions
that had fired by that step and no others. A node decoded at step $t$
is \emph{diverging} iff the two disagree,
\[
\phi_{\mathrm{leg}}(v)\setminus\phi_{\mathrm{spine}}(v)
\;=\; \bigcup_{j:\,J_j > t} \evictset{j} \cap \{1,\dots,t-1\}
\;\neq\; \emptyset ,
\]
i.e.\ iff some eviction that fires \emph{after} $t$ reaches back into
$v$'s prefix.
\end{definition}

Two regions of $\Treestruct$ are non-diverging by construction: nodes
whose prefixes precede every evicted span, and nodes decoded at or
after the final junction $\junct{K}$, for which all evictions have
already been applied and the tree \emph{re-converges} (cf.\ leaf
$y_{\mathrm{ans}}$ in \Cref{tab:running-prefixes}). Nodes decoded
within $[\junct{k-1}, \junct{k})$ form the $k$-th \emph{branch layer},
with $\junct{0} = 0$.

\paragraph{Counts.} Write $R$ for the number of response tokens and $D$
for the number of diverging leaves. Then $\Treestruct$ carries one
root, $K$ junctions, $R$ leaves, and $K{+}1$ root-to-leaf branches, and
each leaf has an eviction mass
$|\phi_{\mathrm{leg}} \setminus \phi_{\mathrm{spine}}|$ --- the tokens
it conditioned on at generation that its final-walk prefix no longer
contains. Only two of these numbers are read off downstream: the
junction depth $K$, which upper-bounds the number of independent
forwards a LogitTree-style method needs, and the diverging-leaf count
$D$, which is exactly the set of positions where SDCC places KL
gradients. \Cref{tab:tree-counts} reports both for the three white-box
harnesses at $T_{\max}\!=\!5$, so tree size can be read directly off
the harness knob. Two regularities matter later: the per-eviction mass
$|\evictset{k}|$ ranks inversely with $K$ (TC-RAG evicts large spans
rarely, MemexRL small slices often), and these are \emph{structural}
capacities under a forced schedule, not live firing rates --- MemexRL
has the deepest tree here but the \emph{lowest} model-triggered
eviction density in \Cref{tab:phase2-longrun}.

\begin{table}[t]
\centering\small
\caption{\textbf{Tree sizes induced by the three white-box harnesses}
on the running example (probe of \Cref{tab:q1}); $K$ =
junctions per rollout, $D$ = diverging leaves per rollout,
$|\evictset{k}|$ = mean evicted-span length per junction.}
\label{tab:tree-counts}
\begin{tabular}{lccc}
\toprule
Harness   & mean $K$ & mean diverging leaves $D$ & mean $|\evictset{k}|$ \\
\midrule
TC-RAG    & 0.87 & 29 & 156 \\
AgentFold & 1.94 & 58 & 84 \\
MemexRL    & 2.00 & 125 & 41 \\
\bottomrule
\end{tabular}
\end{table}
\subsection{Where each method places gradient on $\Treestruct$}

\begin{itemize}[leftmargin=*]
\item \textbf{Naive-Full} runs one forward on the union
$\Hfull$ and takes gradients at every response token position. On
$\Treestruct$ this collapses the tree into its depth-first traversal:
every leaf is conditioned on the \emph{entire} physical prefix
$\hist_{<t} \supseteq \phi_{\mathrm{leg}}(v)$, which contains all
evicted spans --- including spans that had already left the live view
when $v$ was decoded. The training conditioning strictly exceeds the
live view, so Naive-Full trains on information the policy never
had at any point in the rollout (Pitfall B, \S\ref{sec:fm:full}).
\item \textbf{Naive-Compressed} runs one forward on the compressed
walk $\hcomp$ and takes gradients at every response token. On
$\Treestruct$ this collapses all branches into the final walk: the
training forward sees $\phi_{\mathrm{spine}}(v)$, which at the $D$ diverging
leaves is a \emph{strict subset} of the live view
$\phi_{\mathrm{leg}}(v)$ that generated the token --- and, worse,
contains summary content created after $t$. The loss on those leaves
is therefore computed under conditioning the model never decoded
from (Pitfall A, \S\ref{sec:fm:comp}).
\item \textbf{LogitTree (ours, $K$-forward)} runs one forward per branch layer,
so $K\!+\!1$ forwards per rollout. Each forward covers a
root-to-leaf path with $\phi_{\mathrm{spine}} = \phi_{\mathrm{leg}}$
along the entire branch (because the branch's forward uses the
\emph{live view} as it existed during that layer). Gradients are placed on
the leaves of that branch only, so the union of LogitTree's
gradient-carrying positions equals the leaf set of $\Treestruct$
without double-counting.
\item \textbf{4D mask (ours, packed)} realizes the same partition as
LogitTree but in a \emph{single} forward: one structured mask admits, in
each query row, exactly the keys that were live when that token was
decoded, so $\phi_{\mathrm{spine}} = \phi_{\mathrm{leg}}$ holds row by
row and the branches share the tree trunk instead of being replayed.
Gradient placement is identical to LogitTree's; the difference is the pass
count: one forward instead of $K\!+\!1$.
\item \textbf{SDCC (ours)} runs one student forward on $\hcomp$ (with
gradient) plus one gradient-free teacher forward per branch layer that
holds a diverging leaf --- at most $K$, since the layer after
$\junct{K}$ re-converges. The policy-gradient loss remains on the
student's $R$-length response tokens. The consistency KL
$\KL(\pi_{\theta,\text{teach}}(\cdot\mid\histlogt) \,\|\,
\pi_{\theta,\text{stud}}(\cdot\mid\hcomp[:t]))$ is summed only over the
$D$ diverging leaves, whose positions are given by the
diverging-leaf mask constructed from the harness's
$\{J_k, \evictset{k}\}$ records. On $\Treestruct$, SDCC pulls the
student's per-leaf distribution toward the teacher's at exactly the
diverging positions, without touching the trunk. This is where the
$O(\sqrt{\epsKL})$ bound of \S\ref{sec:soft:bound}
applies: as SDCC's KL contracts, the student's leaf distributions
approach the teacher's, and the diverging-leaf gap closes.
\end{itemize}

\subsection{Compute budget as tree traversal cost}

Three lengths govern the budget: the union length $L = |\Hfull|$; the
compressed walk $C = |\hcomp| = L - |\bigcup_k \evictset{k}|$, shorter than $L$
by the token \emph{mass} evicted and not by the eviction \emph{count}; and
$\ell_i$, the live view in force during branch layer $i$. Consecutive branches
share the trunk, so $C \le \ell_i \le L$ --- with $\ell_{K+1} = C$, every
eviction having fired by the last layer --- and $\sum_{i=1}^{K+1} \ell_i \ge L$
with equality only at $K=0$: LogitTree re-reads the shared prefix once per
layer rather than splitting $L$ into $K\!+\!1$ disjoint pieces.
\Cref{tab:hard_compare} tabulates the resulting per-rollout pass and token
budgets next to the infrastructure each method demands.

Tokens are a proxy, not a cost: a forward over $n$ tokens costs
$F(n) = \Theta(n d^2 + n^2 d)$, and for Qwen3-4B the quadratic term overtakes
the linear one at $n \approx 12$k --- inside the range of contexts we observe
(\Cref{tab:grand-matrix}). A dense kernel scores the packed sequence in full
and discards the masked entries, paying $\Theta(L^2 d)$, while $K\!+\!1$
segmented forwards pay $\Theta(\sum_i \ell_i^2 d)$, so neither dominates --- it
turns on the evicted mass. The 4D mask's dependable advantage over LogitTree is
one backward instead of $K\!+\!1$ and one well-filled kernel launch instead of
$K\!+\!1$ short ones, not fewer FLOPs; a genuine FLOP saving needs a
block-sparse kernel. SDCC saves nothing on the forward: because
$\ell_{K+1} = C$, the student pass \emph{is} the last branch layer's, and
SDCC's token budget $C + \sum_{i=1}^{K}\ell_i$ equals LogitTree's
$\sum_{i=1}^{K+1}\ell_i$ term for term. The entire saving is the backward,
$1$ instead of $K{+}1$; and unlike either exact walk it gives up exactness,
down to $\Order(\sqrt{\epsKL})$.

\subsection{Per-leaf divergence and the two pitfalls}
\label{app:tree-notes-fm}
Recall that the editor fires repeatedly \emph{inside} the
rollout loop: at step $t$ the policy sees the live view
$\histlogt = \View_t(\hist_{\le t};\, \editset_{\le t})$ of
\eqref{eq:live-view} --- the physical history with exactly the edits
that have fired by step $t$ applied, and no others --- whereas the
length-$t$ prefix $\hcomp[:t]$ of the final compressed walk
retroactively applies \emph{future} edits $\editset_{>t}$ to a token
generated before those edits existed. \Cref{tab:running-prefixes}
instantiates the divergence for each leaf of the running example.

\begin{table}[h]
\centering\small
\caption{Per-leaf live-view prefix vs.\ final-walk prefix for the
running example ($\evictset{1} = \{y_5\}$ at $\junct{1} = 7$;
$\evictset{2} = \{y_8\}$ at $\junct{2} = 10$;
$\evictset{3} = \{y_{11}\}$ at $\junct{3} = 12$). The leaf whose two
prefixes coincide ($y_{\mathrm{ans}}$) does not need any correction.
Evicting one token per junction keeps each row on a single line;
widening a span enlarges $\evictset{k}$ and each leaf's eviction mass,
but changes neither which leaves diverge nor where the tree
re-converges.}
\label{tab:running-prefixes}
\begin{tabular}{c|l|l|l}
\toprule
Leaf $y_t$ & Live-view prefix $\tpath(y_t)$ & Final-walk prefix $\spath(y_t)$ & Diverges by \\
\midrule
$y_6$   & $\{y_1,y_2,y_3,y_4,y_5\}$              & $\{y_1,y_2,y_3,y_4\}$              & $+\,y_5$ \\
$y_7$   & $\{y_1,y_2,y_3,y_4,y_5,y_6\}$          & $\{y_1,y_2,y_3,y_4,y_6\}$          & $+\,y_5$ \\
$y_9$   & $\{y_1,y_2,y_3,y_4,y_6,y_7,y_8\}$      & $\{y_1,y_2,y_3,y_4,y_6,y_7\}$      & $+\,y_8$ \\
$y_{10}$& $\{y_1,y_2,y_3,y_4,y_6,y_7,y_8,y_9\}$  & $\{y_1,y_2,y_3,y_4,y_6,y_7,y_9\}$  & $+\,y_8$ \\
$y_{12}$& $\{y_1,y_2,y_3,y_4,y_6,y_7,y_9,y_{10},y_{11}\}$ & $\{y_1,y_2,y_3,y_4,y_6,y_7,y_9,y_{10}\}$ & $+\,y_{11}$ \\
$y_{\mathrm{ans}}$ & $\{y_1,y_2,y_3,y_4,y_6,y_7,y_9,y_{10},y_{12}\}$ & same & --- \\
\bottomrule
\end{tabular}
\end{table}

\paragraph{Pitfall A, SFT view (time-travel leakage).}
\begin{equation}
  \Lsft^{\mathrm{A}} \;=\; -\sum_t \log \Pmodel(y_t \mid \spath(y_t)) \;\neq\; -\sum_t \log \Pmodel(y_t \mid \tpath(y_t)) \;=\; \Lstar_{\mathrm{SFT}}.
\end{equation}
The model is fitted to a shorter-prefix distribution than the one
that generated the target token, while at inference the harness still
delivers the live view $\histlogt = \tpath(y_t)$ --- a prefix the
model was never trained to match.

\paragraph{Pitfall A, RL view (broken importance ratio).}
The rollout policy sampled $a_t \sim
\policyold(\cdot\mid\tpath(y_t))$, so the training ratio
\begin{equation}
  \impratio_t^{\mathrm{A}} \;=\; \frac{\policy(a_t\mid \spath(y_t))}{\policyold(a_t\mid \tpath(y_t))} \;\;\neq\;\; \frac{\policy(a_t\mid \tpath(y_t))}{\policyold(a_t\mid \tpath(y_t))} \;=\; \impratio_t
\end{equation}
invalidates PPO's clip guarantee and mis-weights GRPO's group-relative
advantage; empirically, the clip-rate map concentrates immediately
after each junction, precisely where the two prefixes diverge.

\paragraph{Pitfall B, SFT view (train--inference mismatch).}
\begin{equation}
  \Lsft^{\mathrm{B}} \;=\; -\sum_t \log \Pmodel(y_t \mid \hist_{<t}).
\end{equation}
Cross-entropy is computed using a prefix that carries strictly more
information --- including already-evicted content --- than the
deployed model will ever see.

\paragraph{Pitfall B, RL view (advantage on the wrong observation).}
Eventually-evicted tokens remain present during training, so they
still absorb a share of the outcome reward $r_{\mathrm{out}}/T$, and
the advantage baseline is fitted to $V(\hist_{<t})$ instead of
$V(\histlogt)$; the step-level credit-assignment error accumulates
linearly in $K$~\citep{anonymous2025agenticrag}:
\begin{equation}
  \Exp\bigl[\|B_e^{(K)}\|_2\bigr] \simeq \varepsilon_e \cdot K \cdot \Exp\bigl[\|\nabla_\theta \log\policy(a_e \mid s_e)\|_2\bigr],
  \label{eq:step-level-bias}
\end{equation}
where $B_e^{(K)}$ is the cumulative advantage-baseline bias at an
evicted-token position $e$ after $K$ eviction events occur along the
trajectory, and $\varepsilon_e$ is the per-event bias scale (the
misfit between $V(\histt)$ and $V(\histlogt)$ attributable to one
evicted span).

\Cref{tab:bad-symmetry} summarizes the mirror symmetry of the two
pitfalls; concrete per-leaf examples (a $\Delta_{\mathrm{comp}} =
-22.8$-nat leaf for Pitfall~A and a $\Delta_{\mathrm{full}} =
+18.5$-nat leaf for Pitfall~B, both on an untrained Qwen3-4B) are
provided in the main text (\S\ref{sec:fm:symmetry}).

\begin{table}[h]
\centering\small
\caption{Two symmetric failure modes of the same conditioning
invariant.}
\label{tab:bad-symmetry}
\begin{tabular}{l|cc}
\toprule
                                 & Naive-Compressed (A) & Naive-Full (B) \\
\midrule
Training input                   & $\hcomp$ (rightmost path)                 & $\Hfull$ (full DFS) \\
Prefix relation                  & $c^{\mathrm{train}} \subsetneq c^{\mathrm{infer}}$   & $c^{\mathrm{train}} \supsetneq c^{\mathrm{infer}}$ \\
Failure mode                     & time-travel leakage                                          & train--inference mismatch \\
Directional signature            & $c$ too short                                & $c$ too long \\
Intuition                        & ``train saw less than infer''                & ``train saw future info'' \\
Fixable by                       & any tree-consistent method                  & any tree-consistent method \\
\bottomrule
\end{tabular}
\end{table}

For Qwen3-4B under live model-triggered compression
(\S\ref{sec:experiments}, \Cref{tab:grand-matrix}), the mismatch is
directly visible in the \texttt{logdiff}
statistic: the No-compression (Search-R1) control remains at $0.014$
(numerical noise --- with nothing compressed, the live view equals
the physical history), whereas Naive-Compressed rises to
$0.024$--$0.059$ on average across the three editors and peaks at
$0.23$--$0.29$ on eviction-heavy batches (${\sim}20\times$ the
baseline), scaling monotonically with eviction density: the direct
empirical signature of $c^{\mathrm{train}} \neq c^{\mathrm{infer}}$,
present from the very first steps of training.

\paragraph{Correct gradients.} Under the invariant, both the
supervised loss and the policy gradient condition uniformly on
$\histlogt$ --- the \emph{observation} of the editor-induced POMDP
(\S\ref{sec:invariant}): the rollout and training policies in the
importance ratio must share this observation, and the advantage
baseline must approximate $V(\histlogt)$, not $V(\histt)$. A fix must
therefore restore $\nabla_\theta \widehat{\Ltask}(\theta) =
\nabla_\theta \Lstar_{\mathrm{task}}(\theta)$ exactly
(\S\ref{sec:hard}) or up to a quantifiable $\Order(\sqrt{\epsKL})$
bias at the loss level (\S\ref{sec:soft}).

\section{Proofs}
\label{app:proofs}

We proceed bottom-up: the per-branch tree decomposition
(\S\ref{app:proof-prop-tree}) is the primitive, the packed 4D
construction (\S\ref{app:proof-prop1}) reproduces it in one masked
forward, and their equivalence (\S\ref{app:proof-thm1}) follows. Every proof opens
with a one-line sketch. Throughout, $\tpath(y_t) = \histlogt$ is the
live-view root-to-leaf prefix of leaf $y_t$ in the conditioning tree
$\Treestruct$ of \cref{def:logits-tree}, and $\spath(y_t) = \hcomp[:t]$
its prefix on the final compressed walk.

\subsection{Per-branch equivalence on the tree (\cref{prop:sft_tree,prop:rl_tree})}
\label{app:proof-prop-tree}

\begin{proposition}[SFT and RL equivalence on the tree]
\label{prop:sft_tree}
\label{prop:rl_tree}
Let $\Lsft^{\mathrm{tree}}$ be the summed cross-entropy over the $K{+}1$
per-branch forwards, each loss-masked to its own branch. Then
$\nabla_\theta \Lsft^{\mathrm{tree}} = \sum_\pi \nabla_\theta
\Lstar_{\mathrm{SFT}}(\histlog^\pi)$. For RL, the importance ratio,
advantage baseline, and KL regularizer are computed on the same
$\histlog^\pi$ during rollout and update.
\end{proposition}

\begin{proof}
\emph{Sketch: the $K{+}1$ branches partition the loss-carrying leaves, so
regrouping the per-leaf loss by branch is an identity, not an
approximation.}

Each $y_t$ lies on exactly one branch $\pi(y_t)$ --- the one live when it
was produced --- so $\histlog^{\pi(y_t)}[:t] = \tpath(y_t)$, and the
per-branch loss masks are disjoint and jointly exhaustive on the response.
Regrouping $\sum_t -\log \Pmodel(y_t \mid \tpath(y_t))$ by branch returns
$\Lsft^{\mathrm{tree}}$ with no leaf double-counted or missed, and the
finite sum commutes with $\nabla_\theta$. For \Cref{prop:rl_tree}, the
importance ratio, advantage weight and KL regularizer at $y_t$ are all
functions of $\Pmodel(\cdot\mid\tpath(y_t))$, which the branch-$\pi$
forward evaluates under exactly the rollout-time conditioning.
\end{proof}

\subsection{One packed forward reproduces the tree (\cref{prop:sft_4d,prop:rl_4d})}
\label{app:proof-prop1}

\begin{proposition}[SFT and RL equivalence for the 4D mask]
\label{prop:sft_4d}
\label{prop:rl_4d}
Suppose (i) $M^{\mathrm{logical}}$ exactly encodes logical visibility,
(ii) position ids are reassigned for each query row, (iii) loss and
attention masks are decoupled, (iv) no normalization crosses mask
boundaries, and (v) rollout and training share the same
$M^{\mathrm{logical}}$. Then, under dense softmax attention,
$\nabla_\theta \Lsft^{\mathrm{4D}} = \nabla_\theta \Lstar_{\mathrm{SFT}}$,
and the importance ratio and KL regularizer are computed under identical
conditioning during rollout and update.
\end{proposition}

\begin{proof}
\emph{Sketch: a masked query row and its image in the compact sequence
attend to the same keys at the same relative offsets, hence produce the
same logits; induction over layers and the shared loss mask lift this to
gradient equality.}

Let $\edit$ delete $E_t \subset \histt$ and let $M$ implement
$M^{\mathrm{logical}}$ --- the $(T\times T)$ mask of \S\ref{sec:hard:4d}
admitting key $j$ in row $i$ iff $j$ was live when $y_i$ was decoded ---
so $M_{ij} = -\infty$ for $j \in E_t$, $j<i$, and causal otherwise. By
(iii) it suffices to prove $\logits_i(\histt, M) = \logits_i(\histlogt)$
for $i \notin E_t$. Write $i',j'$ for the images of $i,j$ under the
deletion, so the visible set $V_i := \{j<i : j \notin E_t\}$ maps onto the
prefix of the compact sequence $\histlogt$. Row $i$ mixes only over $V_i$
by construction, and by (ii) the reassigned position ids give $(i,j)$ the
same relative offset as $(i',j')$; since RoPE --- or any
relative-positional encoding --- depends only on that offset,
$q_i^\top k_j = q_{i'}^\top k_{j'}$, so softmax weights and mixed values
agree. Condition (iv) closes the only remaining channel of difference,
per-token operations reading across mask boundaries (a layer norm reduced
over an unmasked axis, MoE routing pooling evicted positions), so the
residual stream at row $i$ equals the compact-forward stream at $i'$.
Induction over layers carries this to the unembedding, and the shared loss
mask turns logit equality into the gradient identity.
\end{proof}

\subsection{Materialization equivalence (\cref{thm:hard_equiv})}
\label{app:proof-thm1}

\emph{(Recalled.)} Under the assumptions of \cref{prop:sft_4d}, the
4D-masked single forward on the physical union $\Hfull$
(\S\ref{sec:hard:4d}) and the $K$-segment logits-tree forward produce
the same logits at every non-evicted position.

\begin{proof}
\emph{Sketch: both materializations are the same scalar function of
$\theta$, and identical functions have identical gradients.}

By \cref{prop:sft_4d} each construction computes $\logits_i(\histlogt)$ at
every loss-carrying position under the identical loss mask; they differ
only in cost (union mask: one attention call with a $(B,1,T,T)$ structured
mask; segmented walk: $K{+}1$ separate forwards). Both losses are therefore
the \emph{same scalar function of $\theta$} --- the leaf sum
$-\sum_t \log \Pmodel(y_t \mid \tpath(y_t))$ of \cref{prop:sft_tree} ---
so their gradients coincide up to floating-point association order.
\end{proof}

\section{Full Variational Derivation of the SDCC{} Forward-KL
Self-Distillation Loss}
\label{app:variational}

This appendix derives the SDCC{} objective (\S\ref{sec:soft},
Eq.~\eqref{eq:soft_loss}) as a forward-KL conditioning-consistency
(self-distillation) regularizer, using amortized variational inference
as a scaffold. The correctness of SDCC{} does not rest on the ELBO
being tight: it rests on the zero set of the KL residual being exactly
the conditioning invariant (\S\ref{app:var:gap}), and on that residual
controlling deployment behavior (\S\ref{app:var:tightness}).
\S\ref{sec:soft:method} gives the main-text algorithmic view; here we
state the assumptions and give the proofs.

\emph{Why the forward direction.} The slack in the ELBO is a
\emph{reverse} KL, and descending it directly is the wake-phase update
of amortized inference \citep{hinton1995wakesleep,bornschein2015rws}.
At the sequence level that requires sampling whole continuations from
the \emph{compressed} context: the resulting score-function estimator is
weighted by $\log(q_\theta/p_\theta)$, so its variance is unbounded on
the surplus side, while the deficit side --- which is Pitfall~A itself
--- is almost never drawn from the compressed proposal and is therefore
invisible to the update. SDCC{} instead descends the forward-KL
(sleep-phase) surrogate, which samples from the well-behaved teacher
side, has a plain bounded-variance cross-entropy gradient, and, as
\S\ref{app:var:gap} shows, has exactly the same zero set. The four
subsections follow that substitution: \S\ref{app:var:setup} fixes the
probabilistic model, \S\ref{app:var:elbo} derives the bound the KL comes
from, \S\ref{app:var:gap} shows the forward direction loses nothing, and
\S\ref{app:var:tightness} bounds what a nonzero residual costs at
deployment.

\subsection{Probabilistic model and the conditioning invariant}
\label{app:var:setup}
\label{app:var:remarks}

Fix a diverging leaf position $p \in \Cset$. All quantities below are
tied to this particular $p$; sums over $\Cset$ are taken at the end.
Write $\mathcal{A}(p) = \{k : \iota(\evictset{k}) \le p \le
e(\evictset{k})\}$ for the evictions active at $p$. The common case is
$|\mathcal{A}(p)| = 1$; when several evictions are active the teacher
context re-inserts every span in $\mathcal{A}(p)$ and nothing below
changes.

\paragraph{Random variables.}
\begin{itemize}[leftmargin=*]
  \item \textbf{Clean context} $z_p := \tpath(y_p)$: the pre-eviction
        prefix that was live at rollout time and that the harness will
        \emph{again} deliver at deployment time. It is an
        \emph{externally supplied} conditioning variable, never
        marginalized out: at inference the harness passes it directly to
        the policy.
  \item \textbf{Compressed context} $x_p := \spath(y_p) \subsetneq z_p$:
        the post-eviction prefix produced by the memory editor, which
        has removed every span in $\mathcal{A}(p)$.
  \item \textbf{Action} $y \in \mathcal{V}$: the next generated token.
  \item \textbf{Outcome} $y^{\star}$, with likelihood
        $p(y^{\star} \mid y)$: an oracle target for $y$. We take the RL
        reading as primary, in which $\log p(y^{\star}\mid y) = r(y)/\tau$
        is the exponentiated-advantage log-likelihood of
        control-as-inference \citep{levine2018rlinference}; the SFT
        reading, $\log p(y^\star\mid y) = \log \mathbf{1}[y = y^{\star}]$,
        is its $\tau \to 0$ limit.\footnote{The hard-indicator limit is
        degenerate as a variational target: whenever $q_\theta$ places
        mass off $y^{\star}$ the bound \eqref{eq:app:elbo} reads
        $-\infty \ge -\infty$ and is vacuously satisfied. Every
        statement below is made at finite $\tau$, where the expressions
        remain finite.}
\end{itemize}

\paragraph{Distributions.}
\begin{align}
  p_\theta(y \mid z_p) &:= \Pmodel\bigl(\cdot \mid \tpath(y_p)\bigr)
    && \text{deployed / target policy (teacher forward),}\\
  q_\theta(y \mid x_p) &:= \Pmodel\bigl(\cdot \mid \spath(y_p)\bigr)_p
    && \text{amortized recognition policy (student forward),}\\
  p(y^{\star} \mid y) && & \text{outcome likelihood given action.}
\end{align}
The two policies share parameters $\theta$; the amortization is
\emph{parametric}, not architectural. The conditioning invariant of
\S\ref{sec:invariant} is the statement that this amortization is exact,
\begin{equation*}
  q_\theta(\cdot \mid x_p) \;=\; p_\theta(\cdot \mid z_p)
  \qquad \forall\, p \in \Cset.
  \tag{CI}
  \label{eq:app:var:ci}
\end{equation*}
\eqref{eq:app:var:ci} is not a modeling assumption --- it is the
condition we want to \emph{achieve} through training, and everything
that follows is an argument that the SDCC{} penalty achieves it.

\subsection{The ELBO and where the forward KL comes from}
\label{app:var:elbo}
\label{app:var:wake}
\label{app:var:beta}

This subsection does three things. It derives the variational bound
whose slack is the conditioning gap and shows that slack to be an exact
KL divergence; it identifies the multiplier $\lambda$ of
\eqref{eq:soft_loss} as the $\beta$ of a $\beta$-VAE free energy, which
fixes both its interpretation and its failure mode at large values; and
it shows why that slack cannot be descended in the direction the bound
presents it in --- the fact that forces the forward-KL substitution
justified in \S\ref{app:var:gap}.

\begin{proposition}[Single-step ELBO for the amortized policy]
\label{prop:app:elbo}
For every $\theta$ and every diverging leaf $p \in \Cset$,
\begin{equation}
  \log p_\theta(y^{\star} \mid z_p)
  \;\ge\;
  \Exp_{y \sim q_\theta(\cdot \mid x_p)}\!\bigl[\log p(y^{\star} \mid y)\bigr]
  \;-\;
  \KL\!\Bigl(q_\theta(\cdot \mid x_p) \,\big\|\, p_\theta(\cdot \mid z_p)\Bigr),
  \label{eq:app:elbo}
\end{equation}
with equality iff $q_\theta(\cdot \mid x_p) = p_\theta(\cdot \mid y^{\star}, z_p)$
almost everywhere.
\end{proposition}

\begin{proof}
Marginalizing the action gives the evidence $p_\theta(y^{\star} \mid
z_p) = \sum_y p(y^{\star}\mid y)\,p_\theta(y\mid z_p)$. Multiplying and
dividing by $q_\theta(y\mid x_p)$, rewriting the sum as a
$q_\theta$-expectation and applying Jensen to the concave logarithm,
\begin{align}
  \log p_\theta(y^{\star} \mid z_p)
  &\;=\; \log \Exp_{y \sim q_\theta}\!\left[
     \frac{p(y^{\star} \mid y)\, p_\theta(y \mid z_p)}{q_\theta(y \mid x_p)}\right]
  \;\ge\;
  \Exp_{y \sim q_\theta}\!\left[
    \log \frac{p(y^{\star} \mid y)\, p_\theta(y \mid z_p)}{q_\theta(y \mid x_p)}
  \right]
  \label{eq:app:jensen}\\[-1pt]
  &\;=\;
  \Exp_{y \sim q_\theta}\!\bigl[\log p(y^{\star} \mid y)\bigr]
  \;-\;
  \KL\!\bigl(q_\theta(\cdot \mid x_p)\,\big\|\,p_\theta(\cdot \mid z_p)\bigr),
  \notag
\end{align}
which is \eqref{eq:app:elbo}. The importance-weighting step needs the
proposal $q_\theta(\cdot\mid x_p)$ to dominate the integrand; this is
automatic here, since both distributions are softmax outputs of the same
network and so have full support on $\mathcal{V}$ (up to numerical
underflow). The reverse domination --- $p_\theta(\cdot\mid z_p)$
dominating $q_\theta(\cdot\mid x_p)$ --- is \emph{not} needed for the
bound, and is precisely what fails for the reverse-KL estimator below.

For the equality condition, observe that the Jensen gap is itself a KL
divergence. Bayes' rule at fixed $z_p$ reads $p_\theta(y\mid y^{\star},
z_p) = p(y^{\star}\mid y)\, p_\theta(y\mid z_p) / p_\theta(y^{\star}\mid
z_p)$, so $\log p_\theta(y^{\star}\mid z_p) = \log p(y^{\star}\mid y) +
\log p_\theta(y\mid z_p) - \log p_\theta(y\mid y^{\star}, z_p)$ for every
$y$ in the support. The left-hand side does not depend on $y$, so it is
unchanged by taking the $q_\theta$-expectation of the right-hand side;
adding and subtracting $\Exp_{q_\theta}[\log q_\theta(y\mid x_p)]$ and
regrouping the three resulting expectations gives the \emph{exact}
decomposition
\begin{equation}
\begin{aligned}
  \log p_\theta(y^{\star}\mid z_p)
  \;=\;& \underbrace{\Exp_{q_\theta}\bigl[\log p(y^{\star}\mid y)\bigr]
    - \KL\bigl(q_\theta(\cdot\mid x_p)\,\|\,p_\theta(\cdot\mid z_p)\bigr)}_{\text{the bound }\eqref{eq:app:elbo}}\\[-1pt]
  &\;+\; \underbrace{\KL\bigl(q_\theta(\cdot\mid x_p)\,\|\,p_\theta(\cdot\mid y^{\star}, z_p)\bigr)}_{\text{tightness gap, against the outcome-conditioned posterior}} .
\end{aligned}
  \label{eq:app:evidence-decomp}
\end{equation}
The residual term is non-negative by Gibbs' inequality --- which
re-derives \eqref{eq:app:elbo} without invoking Jensen --- and vanishes
iff $q_\theta(\cdot\mid x_p) = p_\theta(\cdot\mid y^{\star}, z_p)$ a.e.
\end{proof}

\paragraph{Reading the bound.} The first term of \eqref{eq:app:elbo} is
the (negated) task loss evaluated under the compressed conditioning that
training actually sees; the second is the \emph{conditioning gap}
between the compressed and pre-eviction views, and it is the only term
that references $z_p$. Decomposition \eqref{eq:app:evidence-decomp} is
worth isolating because it separates two quantities that are easy to
conflate in the main text: the conditioning gap, which SDCC{} exists to
close, and the tightness gap
$\KL(q_\theta(\cdot\mid x_p)\|p_\theta(\cdot\mid y^{\star},z_p))$ against
the \emph{outcome-conditioned} posterior, which SDCC{} never touches and
does not need to. \S\ref{app:var:gap} makes that separation precise.

\paragraph{The multiplier $\lambda$ is the $\beta$ of a $\beta$-VAE.}
Nothing in \eqref{eq:app:elbo} fixes the relative weight of its two
terms. Treating the conditioning gap as a \emph{constraint} rather than
a fixed-weight penalty --- maximize the task evidence subject to
$\sum_{p\in\Cset}\KL \le \delta$ --- and forming the Lagrangian returns
\eqref{eq:soft_loss} exactly, with $\lambda$ the multiplier and the
constant $\lambda\delta$ discarded. In $\beta$-VAE terms
\citep{higgins2017betavae}, writing $\beta := \lambda$ and using the
task loss as the distortion,
\begin{equation}
  \mathcal{F}_\beta(\theta)
  \;=\; \underbrace{\Ltask(\theta;\hcomp)}_{\text{distortion}}
  \;+\; \beta \sum_{p \in \Cset}
    \underbrace{\KL\bigl(p_{\sg(\theta)}(\cdot\mid z_p)\,\big\|\,q_\theta(\cdot\mid x_p)\bigr)}_{\text{rate}},
  \label{eq:app:free_energy}
\end{equation}
with $\beta = 1$ the plain evidence bound and $\beta = 0$
Naive-Compressed. The distortion--rate tension is genuine here precisely
because $x_p \subsetneq z_p$: distortion is minimized by exploiting the
compressed tokens, rate by making the student's output independent of
them, and any task-relevant content of $z_p \setminus x_p$ must be
recovered through the parameters rather than read off the input. Pushing
$\beta$ too far therefore reproduces the classical failure mode.

\begin{proposition}[Collapse regime at large $\beta$]
\label{prop:app:collapse}
Let $\theta^{\mathrm{coll}}$ be a parameter at which the student output
is functionally independent of $x_p$ and equals the teacher,
$q_{\theta^{\mathrm{coll}}}(\cdot\mid x_p) =
p_{\sg(\theta^{\mathrm{coll}})}(\cdot\mid z_p)$ for all $p \in \Cset$.
Then the rate term of \eqref{eq:app:free_energy} vanishes at
$\theta^{\mathrm{coll}}$ while its distortion term is generically
positive, and for every
$\beta > \Ltask(\theta^{\mathrm{coll}};\hcomp)/\varepsilon^{\star}$ ---
where $\varepsilon^{\star} > 0$ lower-bounds the rate at a task-optimal
$\theta^{\star}$ --- one has
$\mathcal{F}_\beta(\theta^{\mathrm{coll}}) < \mathcal{F}_\beta(\theta^{\star})$,
so descent on $\mathcal{F}_\beta$ from an uninformed initialization is
attracted to the collapsed solution.
\end{proposition}
\begin{proof}
The rate vanishes at $\theta^{\mathrm{coll}}$ by construction. The
distortion is positive because a student that ignores $x_p$ cannot fit a
task loss evaluated on $\hcomp$. Substituting both into
\eqref{eq:app:free_energy} gives
$\mathcal{F}_\beta(\theta^{\mathrm{coll}}) =
\Ltask(\theta^{\mathrm{coll}};\hcomp)$, while
$\mathcal{F}_\beta(\theta^{\star}) \ge \beta\varepsilon^{\star}$; the
stated threshold on $\beta$ is exactly the crossing point. This is
$\beta$-VAE posterior collapse \citep{alemi2018fixing,razavi2019preventing}.
\end{proof}

\noindent Prop.~\ref{prop:app:collapse} is why $\lambda$ carries a
warm-up schedule and the target a stop-gradient rather than being a free
hyperparameter: KL annealing holds $\beta = 0$ until task-loss SGD has
moved off $\theta^{\mathrm{coll}}$, and a frozen (or EMA) target keeps it
from tracking the student into collapse. Both are the standard
$\beta$-VAE remedies, and both are analyzed in
\S\ref{app:sdcc-convergence}.

\paragraph{Why the slack cannot be descended in the direction it appears.}
The remaining step is the direction of the KL. The obvious move is to
descend the term of \eqref{eq:app:elbo} as written --- the
\emph{wake-phase} update of amortized inference
\citep{hinton1995wakesleep,bornschein2015rws}. Because the gradient must
then pass through the sampling distribution, its estimator is a score
function reweighted by a log-ratio.

\begin{lemma}[Wake-phase gradient, and the variance of its estimator]
\label{lem:app:wake}
Let $L^{\mathrm{wake}}(\theta) := \KL\bigl(q_\theta(\cdot\mid x_p)\,\|\,
p_{\sg(\theta)}(\cdot\mid z_p)\bigr)$ with the target branch frozen, and
write $g_\theta(y) := \nabla_\theta \log q_\theta(y\mid x_p)$ and
$\rho_\theta(y) := q_\theta(y\mid x_p)/p_{\sg(\theta)}(y\mid z_p)$. Then
\begin{equation}
  \nabla_\theta L^{\mathrm{wake}}(\theta)
  \;=\; \Exp_{y \sim q_\theta}\bigl[\log\rho_\theta(y)\, g_\theta(y)\bigr],
  \label{eq:app:wake_grad}
\end{equation}
and its $N$-sample Monte-Carlo estimator has variance
\begin{equation}
  \Var\bigl[\widehat{\nabla_\theta L^{\mathrm{wake}}}\bigr]
  \;=\; \frac{1}{N}\Bigl(
    \Exp_{q_\theta}\bigl[(\log\rho_\theta)^{2}\,\|g_\theta\|^{2}\bigr]
    - \bigl\|\Exp_{q_\theta}\bigl[\log\rho_\theta\, g_\theta\bigr]\bigr\|^{2}\Bigr).
  \label{eq:app:var_identity}
\end{equation}
\end{lemma}
\begin{proof}
$L^{\mathrm{wake}}(\theta) = \sum_y q_\theta(y\mid x_p)\log\rho_\theta(y)$,
and only $q_\theta$ carries a gradient, so $\nabla_\theta
L^{\mathrm{wake}} = \sum_y \nabla_\theta q_\theta(y)\log\rho_\theta(y) +
\sum_y \nabla_\theta q_\theta(y)$. The log-derivative identity
$\nabla_\theta q_\theta = q_\theta\, g_\theta$ turns the first sum into
\eqref{eq:app:wake_grad}, and the second is $\nabla_\theta \sum_y
q_\theta(y) = \nabla_\theta 1 = 0$. Equation \eqref{eq:app:var_identity}
is then the variance of a mean of $N$ i.i.d.\ copies of $X :=
\log\rho_\theta(y)\, g_\theta(y)$, namely $(\Exp\|X\|^{2} - \|\Exp
X\|^{2})/N$.
\end{proof}

\noindent Identity \eqref{eq:app:var_identity} is exact --- it uses no
$\log\rho \approx \rho - 1$ small-ratio expansion --- and it exposes two
failure modes, one on each side of the eviction:
\begin{itemize}[nosep,leftmargin=*]
  \item \textbf{Surplus side: unbounded variance.} A token $y^{\dagger}$
        the compressed context finds plausible ($q_\theta(y^{\dagger}\mid
        x_p) \ge q_0 > 0$) but the pre-eviction context does not
        ($p_{\sg(\theta)}(y^{\dagger}\mid z_p) \to 0$) --- a redundant
        re-query that looks reasonable only once the retrieved span is
        gone --- sends $\log\rho_\theta(y^{\dagger}) \to +\infty$, and the
        second moment in \eqref{eq:app:var_identity} grows like
        $(\log\rho_\theta(y^{\dagger}))^{2}$. No variance bound holds
        uniformly over $\Theta$.
  \item \textbf{Deficit side: sampling blindness.} A token the
        pre-eviction context supports ($p_{\sg(\theta)}(y^{\dagger}\mid
        z_p) \ge p_0 > 0$) but the compressed one has lost
        ($q_\theta(y^{\dagger}\mid x_p) \to 0$) is Pitfall~A itself. Here
        the variance does \emph{not} blow up, since $q\log^{2} q \to 0$,
        and that is exactly the problem: the summand
        $q_\theta(y^{\dagger})\log\rho_\theta(y^{\dagger})g_\theta(y^{\dagger})
        \to 0$, so the update draws no signal from the one token whose
        conditioning it is meant to repair, and an $N$-sample estimator
        needs $N = \Omega\bigl(1/q_\theta(y^{\dagger})\bigr)$ draws to
        observe it even once.
\end{itemize}
Both are properties of the sequence-level Monte-Carlo estimator rather
than of the reverse KL as a function: a per-token vocabulary sum
evaluates $L^{\mathrm{wake}}$ exactly and has neither pathology. What
survives even in that exact form is the reverse direction's zero-forcing
bias, which under-covers precisely the modes an eviction removes. Two
substitutions therefore separate \eqref{eq:app:elbo} from
\eqref{eq:soft_loss} --- the KL is taken forward rather than reverse, and
its right argument is frozen by a stop-gradient --- and
\S\ref{app:var:gap} justifies both, showing that the forward direction
samples from the well-behaved side, has a bounded-variance gradient, and
retains the same zero set.

\subsection{Why the forward KL suffices}
\label{app:var:gap}
\label{app:var:sleep}

Two things must hold before the ELBO's reverse KL may be replaced by the
forward KL that SDCC{} minimizes: the reverse-KL slack must vanish
exactly at the conditioning invariant, so that the ELBO points at the
right target; and the forward KL must have the same zero set, so that
the substitution loses nothing.

\paragraph{The slack is the conditioning gap, and is weaker than tightness.}
\begin{proposition}[Slack-zero implication]
\label{prop:app:gap}
$\KL\bigl(q_\theta(\cdot \mid x_p)\,\|\,p_\theta(\cdot \mid z_p)\bigr) = 0$
if and only if $q_\theta(\cdot\mid x_p) = p_\theta(\cdot\mid z_p)$ a.e.,
i.e.\ if and only if the per-leaf conditioning
invariant~\eqref{eq:app:var:ci} holds. This is \emph{strictly weaker}
than tightness of \eqref{eq:app:elbo}, which by
Prop.~\ref{prop:app:elbo} additionally requires the recognition policy
to equal the \emph{outcome-conditioned} posterior
$p_\theta(\cdot\mid y^{\star}, z_p)$ rather than the outcome-marginal
$p_\theta(\cdot\mid z_p)$.\footnote{The two targets coincide only when
$p(y^{\star}\mid y)$ is $y$-independent on the support of
$p_\theta(\cdot\mid z_p)$, i.e.\ when the outcome carries no
discriminative signal --- a degenerate configuration, and not the regime
of interest here.}
\end{proposition}
\begin{proof}
Gibbs' inequality gives the first equivalence: a KL divergence is zero
iff its two arguments coincide almost everywhere. The second claim is
the equality condition of Prop.~\ref{prop:app:elbo}.
\end{proof}

\noindent The ELBO is therefore scaffolding rather than the object of
interest: driving its slack to zero closes the amortization gap between
the two conditioning distributions \emph{without} tightening the bound,
and SDCC's justification needs only the former.

\paragraph{The sleep-phase surrogate.} Wake--sleep
\citep{hinton1995wakesleep} flips the sampling direction: instead of
sampling from $q_\theta$ and fitting it to $p_\theta$, it samples from
the frozen $p_{\sg(\theta)}$ and fits $q_\theta$ by maximum likelihood
on those samples.

\begin{proposition}[The sleep-phase update is a forward KL and recovers \eqref{eq:soft_loss}]
\label{prop:app:sleep}
Let $L^{\mathrm{sleep}}(\theta) := \Exp_{y \sim p_{\sg(\theta)}(\cdot\mid z_p)}
\bigl[-\log q_\theta(y \mid x_p)\bigr]$. Then
\begin{equation}
  L^{\mathrm{sleep}}(\theta)
  \;=\;
  \KL\!\Bigl(p_{\sg(\theta)}(\cdot\mid z_p)\,\big\|\,q_\theta(\cdot\mid x_p)\Bigr)
  \;+\; H\!\bigl(p_{\sg(\theta)}(\cdot\mid z_p)\bigr),
  \label{eq:app:sleep_split}
\end{equation}
with $H(\cdot)$ the Shannon entropy. Since the stop-gradient makes the
entropy $\theta$-independent, $\nabla_\theta L^{\mathrm{sleep}} =
\nabla_\theta \KL(p_{\sg(\theta)} \,\|\, q_\theta) = -\Exp_{y \sim
p_{\sg(\theta)}}[\nabla_\theta \log q_\theta(y\mid x_p)]$, a plain
cross-entropy score whose $N$-sample Monte-Carlo estimator has variance
at most $\Exp_{y \sim p_{\sg(\theta)}}\|\nabla_\theta \log q_\theta(y\mid
x_p)\|^{2}/N$ --- no $q/p$ re-weighting appears. Summing over
$p \in \Cset$ and adding the task loss recovers the boxed SDCC{}
objective \eqref{eq:soft_loss} verbatim:
\begin{equation}
  \mathcal{L}_{\mathrm{SDCC{}}}(\theta)
  \;=\;
  \Ltask(\theta;\hcomp)
  \;+\;
  \lambda \sum_{p \in \Cset}
  \KL\!\Bigl(\underbrace{p_{\sg(\theta)}(\cdot\mid z_p)}_{\text{teacher, stop-grad}}
    \,\Big\|\,
    \underbrace{q_\theta(\cdot\mid x_p)}_{\text{student, with-grad}}\Bigr).
  \label{eq:app:sdcc_recovered}
\end{equation}
\end{proposition}
\begin{proof}[Proof sketch]
Add and subtract $\log p_{\sg(\theta)}(y\mid z_p)$ inside the
expectation defining $L^{\mathrm{sleep}}$ and identify the two resulting
expectations as a KL divergence and an entropy; that is
\eqref{eq:app:sleep_split}. Differentiating with the target held fixed
leaves the score expectation stated, whose estimator variance is its
second moment; the bound is finite because $q_\theta$ has full support
(softmax) and no ratio-based re-weighting inflates the summand. For the
last claim, sum \eqref{eq:app:sleep_split} over $p \in \Cset$, discard
the $\theta$-independent entropy sum, weight the KL sum by $\lambda$,
add the task loss, and substitute $q_\theta(\cdot\mid x_p) =
\Pmodel(\cdot\mid\spath(y_p))_p$ and $p_{\sg(\theta)}(\cdot\mid z_p) =
\Ptarget(\cdot\mid\tpath(y_p))$ with the notation of
\S\ref{sec:soft:form}.
\end{proof}

\paragraph{The two directions have the same zero set.}
\begin{theorem}[Zero SDCC{} residual $\Leftrightarrow$ conditioning invariant]
\label{thm:app:tightness}
Let $\varepsilon_p(\theta) := \KL\bigl(p_{\sg(\theta)}(\cdot\mid z_p)\,\|\,
q_\theta(\cdot\mid x_p)\bigr)$ be the per-leaf forward-KL residual that
SDCC{} minimizes. Then
\begin{equation}
  \KL\bigl(p_{\sg(\theta)}\,\big\|\,q_\theta\bigr) = 0
  \quad\Longleftrightarrow\quad
  p_{\sg(\theta)} = q_\theta \;\text{ a.e.}
  \quad\Longleftrightarrow\quad
  \KL\bigl(q_\theta\,\big\|\,p_{\sg(\theta)}\bigr) = 0,
  \label{eq:app:zeroset}
\end{equation}
and consequently $\sum_{p\in\Cset}\varepsilon_p(\theta) = 0$ if and only
if the conditioning invariant~\eqref{eq:app:var:ci} holds at every
diverging leaf --- in which case the ELBO's reverse-KL slack of
Prop.~\ref{prop:app:gap} is zero as well.
\end{theorem}
\begin{proof}
Both equivalences in \eqref{eq:app:zeroset} are Gibbs' inequality: a KL
divergence vanishes iff its two arguments agree almost everywhere, and
that condition is symmetric in the two arguments even though the two
divergences are not. A sum of non-negative terms is zero iff every term
is, which lifts the per-leaf statement to $\Cset$.
\end{proof}

\noindent Away from the zero set the two divergences are different
functions of $\theta$, and the forward direction is the one that can
actually be descended in a single-backward training loop: the frozen
teacher supplies the target distribution, the gradient flows only
through the student, and the update is an ordinary cross-entropy rather
than a $\log(q_\theta/p_\theta)$-weighted score-function estimator whose
variance diverges exactly on the tokens the eviction damaged. Full ELBO
tightness (Prop.~\ref{prop:app:elbo}) is strictly stronger than
\eqref{eq:app:zeroset} and is neither targeted nor required.

\paragraph{Existence of a zero-KL solution.}
\label{app:proof-existence}
The residual can be driven to zero at all only if the model can
reproduce its pre-eviction conditional from the compressed prefix alone.

\begin{proposition}[Existence of a zero-KL solution]
\label{prop:existence}
If model capacity admits $\theta^\star$ with
$\Pmodel[\theta^\star](\cdot\mid \spath(y_p))_p =
\Pmodel[\theta^\star](\cdot\mid \tpath(y_p))$ for all $p \in \Cset$,
then the SDCC{} KL loss attains $0$. This holds when the evicted span is
redundant with the surviving trunk, and fails when the span carries a
strictly necessary bit absent from $\hcomp$, in which case
$\varepsilon_p$ stays positive and SDCC{} is bias-dominated relative to
the exact methods.
\end{proposition}
\begin{proof}
Each per-leaf term $\varepsilon_p(\theta)$ is non-negative, so
$K(\theta) = \sum_{p\in\Cset}\varepsilon_p(\theta) \ge 0$ with equality
iff student and teacher coincide on every diverging leaf
(\cref{thm:app:tightness}); whenever capacity admits such a
$\theta^\star$ the zero set $\Theta_K$ is non-empty and the loss attains
$0$. This is existence only --- that SGD reaches it is
\cref{thm:sdcc:contract} in \cref{app:sdcc-convergence}, at rate
$\Order(1/\sqrt{T})$ to the noise floor and exactly $0$ under the
decaying-step-size schedule of \cref{cor:sdcc:rate}.
\end{proof}

\subsection{How the KL residual controls deployment and gradient bias}
\label{app:var:tightness}
\label{app:proof-prop5}

In finite training the residual $\varepsilon_p(\theta)$ is small but
nonzero, so what matters is not the zero set of
\cref{thm:app:tightness} but the chain \emph{small KL $\Rightarrow$
small total variation $\Rightarrow$ small deployment deviation
$\Rightarrow$ small policy-gradient bias}. This subsection states that
chain once. Its first link is unconditional; only the last requires
bounded advantage and score.

\begin{proposition}[Behavioral Pinsker bound; restates and proves \cref{prop:pinsker-bound}]
\label{prop:app:tv}
For every diverging leaf $p \in \Cset$,
\begin{equation}
  \bigl\|p_\theta(\cdot\mid z_p) - q_\theta(\cdot\mid x_p)\bigr\|_{\mathrm{TV}}
  \;\le\; \sqrt{\varepsilon_p(\theta)/2}.
  \label{eq:app:pinsker}
\end{equation}
\end{proposition}
\begin{proof}
Pinsker's inequality \citep{tsybakov2009nonparametric} states that
$\|\mu - \nu\|_{\mathrm{TV}} \le \sqrt{\KL(\mu\|\nu)/2}$ for any two
distributions on a common measurable space. Take $\mu =
p_{\sg(\theta)}(\cdot\mid z_p)$ and $\nu = q_\theta(\cdot\mid x_p)$,
whose divergence is exactly the per-leaf SDCC{} residual
$\varepsilon_p(\theta)$; at convergence the stop-gradient branch equals
$p_\theta$ and \eqref{eq:app:pinsker} follows. No bounded-score,
same-advantage or parameterization assumption enters, so the bound is
unconditional. Because the harness redelivers $z_p$ at inference, the
left-hand side is literally the deployment gap: the behavioral distance
between the context the update was computed under and the one the
deployed policy is served.
\end{proof}

\begin{assumption}[Bounded advantage and last-layer score]
\label{ass:bounded}
(A1) $\|\hat A\|_\infty \le A_{\max}$; (A2) the same behavioral advantage
$\hat A(y)$ is used on both $\tpath(y_p)$ and $\spath(y_p)$; (A3) softmax
parametrization with $\|h\|_\infty \le H_{\max}$, and the policy gradient
is taken w.r.t.\ the last-layer parameters, for which $\nabla_\theta \log
\pi(y\mid c) = h_y(c) - \Exp_{y'\sim\pi}[h_{y'}(c)]$ is uniformly bounded
by $2H_{\max}$; (A4) advantages are averaged over an outer batch.
\end{assumption}

\begin{corollary}[Conditional gradient-bias bound]
\label{cor:sdcc-grad-bias}
Under \cref{ass:bounded}, the SDCC{} per-leaf policy-gradient deviation
is $\Order(\sqrt{\varepsilon_p})$; in aggregate, with
$C := 5 A_{\max} H_{\max}$,
\begin{equation}
  \Bigl\|\nabla_\theta \Ltask(\theta;\hcomp)
    - \nabla_\theta \Ltask^{\star}(\theta)\Bigr\|
  \;\le\;
  C \sum_{p \in \Cset} \sqrt{\varepsilon_p/2}
  \;=\; \Order\Bigl(\textstyle\sum_{p \in \Cset} \sqrt{\varepsilon_p}\Bigr),
  \label{eq:grad_bias-app}
\end{equation}
where $\Ltask^{\star}$ is the tree-consistent task loss that 4D masks and
LogitTree compute exactly. Writing $\epsKL := \max_{p}\varepsilon_p$,
this is the per-leaf $\Order(\sqrt{\epsKL})$ form quoted in
\S\ref{sec:soft:bound}.
\end{corollary}
\begin{proof}
\emph{Sketch: split the two policy gradients into a
distribution-difference term and a score-difference term; (A3) makes both
Lipschitz in total variation, which \cref{prop:app:tv} converts into
$\sqrt{\varepsilon_p}$.}

Write $\pi := \pi_\theta(\cdot\mid\tpath(y_p))$, $\widetilde{\pi} :=
\pi_\theta(\cdot\mid\spath(y_p))$, $s_\pi(y) := \nabla_\theta\log\pi(y)$,
$g := \sum_y \pi(y)\hat{A}(y) s_\pi(y)$, and $\tilde g$ analogously with
the same $\hat A$ by (A2). Then
\begin{align}
  \|g - \tilde{g}\|
  &\le \Bigl\|\textstyle\sum_y \bigl(\pi(y)-\widetilde{\pi}(y)\bigr)\hat{A}(y)\,s_\pi(y)\Bigr\|
   + \Bigl\|\textstyle\sum_y \widetilde{\pi}(y)\hat{A}(y)\bigl(s_\pi(y)-s_{\widetilde{\pi}}(y)\bigr)\Bigr\| \notag\\
  &\le 2 A_{\max} H_{\max}\,\|\pi - \widetilde{\pi}\|_1
   \;+\; A_{\max}\,\Exp_{\widetilde{\pi}}\bigl\|s_\pi - s_{\widetilde{\pi}}\bigr\|.
  \label{eq:app:grad-decomp}
\end{align}
Total variation is half the $\ell_1$ distance, so by (A1), (A3) and
\cref{prop:app:tv} the first term is at most $4 A_{\max}
H_{\max}\sqrt{\varepsilon_p/2}$. In the second, (A3) makes the two scores
share $h_y$ at the query token and differ only in the normalizer,
$s_\pi(y) - s_{\widetilde{\pi}}(y) = \Exp_{\widetilde{\pi}}[h] -
\Exp_{\pi}[h]$, bounded by $H_{\max}\|\pi -
\widetilde{\pi}\|_{\mathrm{TV}} \le H_{\max}\sqrt{\varepsilon_p/2}$
(scope: \cref{rem:app:pinsker-last-layer}). Summing over per-leaf
residuals and per-trajectory tokens as in (A4) gives
\eqref{eq:grad_bias-app}.
\end{proof}

\begin{remark}[Scope of the gradient corollary: last-layer restriction]
\label{rem:sdcc-grad-scope}
\label{rem:app:pinsker-last-layer}
\cref{prop:pinsker-bound} is unconditional; only \cref{cor:sdcc-grad-bias}
invokes \cref{ass:bounded}, where (A2) is the standard PPO/GRPO
advantage-identifiability convention and (A3) supplies the closed-form
score. That identity needs $\theta$ to index the softmax head, so the two
contexts share $h_y$ at the query token and differ only in the normalizer
$\Exp_{\pi}[h]$. For full transformer parameters the score acquires an
extra $\nabla_\theta h_y(c)$ term that softmax structure alone does not
bound uniformly, and the corollary then needs a separate bound on
$\|\nabla_\theta h_y(\tpath) - \nabla_\theta h_y(\spath)\|$, e.g.\ via
Lipschitz assumptions on the backbone. This restriction is intrinsic to
the ratio-difference decomposition, not an artifact of our proof.
\end{remark}

\begin{corollary}[Rate to zero bias under training]
\label{cor:app:rate}
Under Assumptions (i)--(iv) of \S\ref{app:sdcc-convergence},
Theorem~\ref{thm:sdcc:contract} yields
$\min_{t \le T} \Exp[\varepsilon_p(\theta_t)] = \Order(1/\sqrt{T})$;
combined with \cref{cor:sdcc-grad-bias}, the deployment
policy-gradient bias decays at the rate $\Order(T^{-1/4})$.
\end{corollary}

\paragraph{Cost, and where SDCC{} sits relative to the exact methods.}
\label{sec:soft:compute-app}
\label{sec:soft:axes-app}
The bound above is what SDCC{} buys with its cost profile. SDCC{}
requires exactly one backward pass per trajectory against LogitTree's
$K{+}1$; the teacher forwards share KV-cache trunks and are
stop-gradient, so they retain no activations. A 4D mask is cheaper
still in pure FLOPs --- one forward, one backward --- but requires
attention-kernel surgery and cannot be applied in black-box harnesses
where eviction spans are never exposed. The three methods therefore
trade off along two independent axes, correctness bias and
infrastructure cost: the exact walks have
zero bias at heavy cost, while SDCC{} pays $\Order(\sqrt{\epsKL})$ bias
--- decaying during training by \cref{cor:app:rate} --- for one backward
pass and no kernel changes. The claim is not that SDCC{} dominates the
exact methods: where their cost is acceptable, LogitTree and the 4D mask
are the correctness ground truth; where it is not, or in black-box
regimes, SDCC{} is the only one of the three that applies.

\subsection{LogitTree and SDCC{} optimize the same objective up to a
$\lambda$-tunable slack}
\label{app:var:logittree-sdcc-bound}

The exact walks of \Cref{sec:hard} and SDCC{} of \Cref{sec:soft} charge a
diverging leaf $p \in \Cset$ in visibly different ways: LogitTree places
the policy loss on the pre-eviction branch $z_p = \tpath(y_p)$, so its
per-leaf loss is an expectation under the teacher distribution
$p_\theta(\cdot\mid z_p)$; SDCC{} keeps the policy loss on the
compressed branch $x_p = \spath(y_p)$, so its per-leaf loss is an
expectation under the student distribution $q_\theta(\cdot\mid x_p)$,
augmented with a stop-gradient forward KL that penalizes the two
distributions for disagreeing (\S\ref{sec:soft:form}). The two
constructions therefore differ in \emph{which} branch of $\Treestruct$
enters the policy loss --- LogitTree scores every branch (Prop.~%
\ref{prop:sft_tree}); SDCC{} scores only the deployed branch and
charges the others through the KL. Because the KL residual controls the
behavioral gap between the two branches (\cref{prop:app:tv}), the two
losses can be made arbitrarily close by choosing $\lambda$; this
subsection makes that statement precise as a two-sided bound.

Fix $\theta$ and $p \in \Cset$; write $\varepsilon_p := \varepsilon_p(\theta)$.
Following the paper's per-leaf normalization convention, let
\begin{equation*}
  \mathcal{L}^{\mathrm{LT}}_p(\theta)
  \;:=\; -\Exp_{y\sim p_\theta(\cdot\mid z_p)}\!\bigl[\hat A(y)\bigr],
  \qquad
  \mathcal{L}^{\mathrm{SDCC{}}}_p(\theta)
  \;:=\; -\Exp_{y\sim q_\theta(\cdot\mid x_p)}\!\bigl[\hat A(y)\bigr]
         + \lambda\, \varepsilon_p
\end{equation*}
denote the per-diverging-leaf contributions of the two objectives; on
non-diverging leaves $z_p = x_p$, both reduce to the identical
single-branch term, so a sum-over-$\Cset$ suffices. $\hat A$ is the same
behavioral advantage on both branches (Assumption~\ref{ass:bounded},
A2), and no rollout re-sampling appears.

\begin{theorem}[Two-sided objective equivalence between LogitTree and
SDCC{}]
\label{thm:app:lt-sdcc-equiv}
Under Assumption~\ref{ass:bounded}(A1), for every $\theta$, every
$p \in \Cset$, and every $\lambda > 0$,
\begin{equation}
  \Bigl|\mathcal{L}^{\mathrm{SDCC{}}}_p(\theta)
        - \mathcal{L}^{\mathrm{LT}}_p(\theta)
        - \lambda\, \varepsilon_p\Bigr|
  \;\le\; A_{\max}\sqrt{2\,\varepsilon_p}.
  \label{eq:app:lt-sdcc-raw}
\end{equation}
Consequently, Young's inequality
($A_{\max}\sqrt{2\varepsilon_p} \le A_{\max}^{2}/(2\lambda) + \lambda\varepsilon_p$)
yields a $\theta$-uniform lower bound and a $\theta$-tracking upper
bound,
\begin{equation}
  \mathcal{L}^{\mathrm{LT}}_p(\theta) - \frac{A_{\max}^{2}}{2\lambda}
  \;\le\;
  \mathcal{L}^{\mathrm{SDCC{}}}_p(\theta)
  \;\le\;
  \mathcal{L}^{\mathrm{LT}}_p(\theta) + 2\lambda\,\varepsilon_p
     + \frac{A_{\max}^{2}}{2\lambda},
  \label{eq:app:lt-sdcc-clean}
\end{equation}
and summing over $p \in \Cset$ gives $\mathcal{L}^{\mathrm{LT}}(\theta) -
|\Cset|A_{\max}^{2}/(2\lambda) \le \mathcal{L}^{\mathrm{SDCC{}}}(\theta) \le
\mathcal{L}^{\mathrm{LT}}(\theta) + 2\lambda \sum_{p} \varepsilon_p +
|\Cset|A_{\max}^{2}/(2\lambda)$. On the zero set
$\{\theta : \varepsilon_p(\theta) = 0\ \forall p \in \Cset\}$ of
\cref{thm:app:tightness}, \eqref{eq:app:lt-sdcc-raw} holds with equality
at $0$, so LogitTree and SDCC{} coincide exactly as functions of
$\theta$; the relaxed form \eqref{eq:app:lt-sdcc-clean} still carries its
$A_{\max}^{2}/(2\lambda)$ slack there, Young's inequality being loose at
$\varepsilon_p = 0$.
\end{theorem}
\begin{proof}
Rewrite the difference of the two policy-loss terms as a change of
measure and apply Hölder:
\begin{equation}
  \bigl(\Exp_{q_\theta(\cdot\mid x_p)} - \Exp_{p_\theta(\cdot\mid z_p)}\bigr)
    [\hat A(y)]
  \;=\; \sum_y \bigl(q_\theta(y\mid x_p) - p_\theta(y\mid z_p)\bigr)\hat A(y)
  \;\le\; \|\hat A\|_\infty \cdot \|q_\theta - p_\theta\|_{1},
  \label{eq:app:lt-sdcc-change-of-measure}
\end{equation}
and analogously for the reverse sign. By Assumption~\ref{ass:bounded}(A1)
the leading factor is at most $A_{\max}$; total variation is half the
$\ell_1$ distance and \cref{prop:app:tv} bounds it by
$\sqrt{\varepsilon_p/2}$. Together,
$|(\Exp_{q_\theta} - \Exp_{p_\theta})[\hat A]| \le A_{\max}\sqrt{2\varepsilon_p}$.
Substituting into
$\mathcal{L}^{\mathrm{SDCC{}}}_p - \mathcal{L}^{\mathrm{LT}}_p =
(\Exp_{p_\theta} - \Exp_{q_\theta})[\hat A] + \lambda\varepsilon_p$
gives \eqref{eq:app:lt-sdcc-raw}. Young's inequality supplies
$A_{\max}\sqrt{2\varepsilon_p} = 2\sqrt{(A_{\max}^{2}/(2\lambda))\cdot(\lambda\varepsilon_p/1)}
\le A_{\max}^{2}/(2\lambda) + \lambda\varepsilon_p$; adding and
subtracting this on the two sides of \eqref{eq:app:lt-sdcc-raw} yields
the clean form \eqref{eq:app:lt-sdcc-clean}. On the zero set every
$\varepsilon_p$ vanishes, hence $p_\theta(\cdot\mid z_p) =
q_\theta(\cdot\mid x_p)$ by \cref{thm:app:tightness} and the
change-of-measure term \eqref{eq:app:lt-sdcc-change-of-measure} is $0$,
so \eqref{eq:app:lt-sdcc-raw} is an equality between zeros.
\end{proof}

\paragraph{What the bound says, and where it stops.} Read in the
maximization convention $J := -\mathcal{L}$, the lower bound of
\eqref{eq:app:lt-sdcc-clean} becomes $J^{\mathrm{SDCC{}}}(\theta) -
|\Cset|A_{\max}^{2}/(2\lambda) \le J^{\mathrm{LT}}(\theta)$: SDCC{}'s
objective, offset by a $\theta$-\emph{independent} constant, is a
\emph{lower bound on the exact per-branch objective}, so maximizing it
raises that bound at every $\theta$ and at every residual
$\varepsilon_p$, with the offset shrinking as the KL weight $\lambda$
grows. Equivalently, in the loss convention $\mathcal{L}^{\mathrm{SDCC{}}}$
plus that constant majorizes $\mathcal{L}^{\mathrm{LT}}$, so minimizing
SDCC{} minimizes an upper bound on the exact walk's loss. The upper bound is looser --- it retains a
$\lambda\sum_p\varepsilon_p$ term --- because SDCC{} pays the extra KL
even when the exact walk does not; on the CI zero set that term is
zero, so the two objectives coincide there. The classical VAE
correspondence of App.~\ref{app:var:beta} reappears: the same $\lambda$
that makes the lower bound tight is the $\beta$ whose \emph{too-large}
regime triggers posterior collapse (Prop.~\ref{prop:app:collapse}), so
$\lambda$ cannot be sent to infinity naively. Two consequences follow.
(i) The two objectives share a common global minimizer set on the CI
zero set of \cref{thm:app:tightness}, and their gradients on that set
agree (\cref{cor:sdcc-grad-bias} restates the away-from-zero version
under the same assumptions); this is the precise sense in which the
two objectives are ``essentially equivalent''. (ii) The gap is
\emph{not zero in general}: it scales as $\lambda\varepsilon_p$ on the
upper side, which is precisely the main-text trade-off in
\S\ref{sec:soft:method}: SDCC{} buys single-backward
compute at a residual that is $\theta$-controlled by the KL weight
rather than structurally zero, and this residual is what
\cref{thm:sdcc:contract} contracts at $\Order(1/\sqrt T)$.

\subsection{Convergence: contraction toward the zero-KL set}
\label{app:sdcc-convergence}
\label{app:sdcc:setup}
\label{app:sdcc:steady}
\label{app:sdcc:rate}

\paragraph{Setup and assumptions.}
With $\theta_t \in \Theta \subset \mathbb{R}^d$ the step-$t$ parameters,
$\hcomp$ the compressed student input (\S\ref{sec:soft:form}) and
$(\histlogt,\hcomp) \sim \pi$, the task loss is $J(\theta) \triangleq
\Exp_{\pi}[\Ltask(\theta;\hcomp)]$ and the SDCC penalty is
\begin{equation}
K(\theta) \;\triangleq\; \Exp_{\pi}\Bigl[ \textstyle\sum\nolimits_{t \in \Cset}
                \KL\bigl(\Pmodel[\sg(\theta)](\cdot\!\mid\!\histlogt)
                      \,\big\|\, \Pmodel[\theta](\cdot\!\mid\!\hcomp)_t\bigr)
              \Bigr],
\end{equation}
the stop-gradient making $K$ depend on $\theta$ only through the update branch.
The composite objective is $F_t = J + \lambda_t K$, with step $\theta_{t+1} =
\theta_t - \eta_t(\widehat{\nabla} J + \lambda_t \widehat{\nabla} K)(\theta_t)$
and schedule $\lambda_t = 0$ for $t < T_{\mathrm{warm}}$, $\lambda_t \to
\lambda^\star > 0$ after. The post-warm-up Lyapunov function $\Phi \triangleq
(J - J^\star) + \lambda^\star K$, $J^\star := \inf_\theta J$, is non-negative,
zero at a joint minimizer, and $L_\Phi$-smooth with $L_\Phi \le L_J +
\lambda^\star L_K$. We assume throughout ((i)--(iv) are distinct from
(A1)--(A4) of \Cref{ass:bounded}):
\begin{enumerate}[nosep,leftmargin=*]
  \item[\textbf{(i)}] \emph{Smoothness:} $J$, $K$ are $L_J$-, $L_K$-smooth.
  \item[\textbf{(ii)}] \emph{Bounded gradient noise:}
        $\Exp\|\widehat{\nabla}J\|^2 \le \sigma_J^2$,
        $\Exp\|\widehat{\nabla}K\|^2 \le \sigma_K^2$.
  \item[\textbf{(iii)}] \emph{Joint realizability:} $\{\theta : K(\theta) = 0\}$
        is nonempty (Prop.~\ref{prop:existence}) and meets $\arg\min_\theta J$;
        fix $\theta^\star$ in the intersection, so $\Phi(\theta^\star) = 0$.
        This adds to Prop.~\ref{prop:existence} the capacity condition that
        conditioning-invariance and task-optimality be jointly attainable.
  \item[\textbf{(iv)}] \emph{Composite Polyak--{\L}ojasiewicz:} $\|\nabla
        \Phi\|^2 \ge 2\mu_\Phi \Phi$ near $\theta^\star$ for some $\mu_\Phi >
        0$; stronger than PL on $K$ alone, as it excludes cancellations
        $\nabla J \approx -\lambda^\star \nabla K$ at $\Phi > 0$.
\end{enumerate}
Only (iv) is restrictive: near $\theta^\star$ both Hessians are Fisher matrices,
hence PSD, and PSD alone does not imply PL. It holds when $\Phi$ is locally
strongly convex along descent-relevant directions (e.g.\ the bounded-score
last-layer softmax of Cor.~\ref{cor:sdcc-grad-bias}); under PSD only, dropping
the PL step of Theorem~\ref{thm:sdcc:contract} still leaves $\min_t
\Exp\|\nabla\Phi(\theta_t)\|^2 \to 0$, forfeiting the linear rate on $\Phi$
(hence $K$) but not convergence to a critical point.

\paragraph{Warm-up.} For $t < T_{\mathrm{warm}}$ training is pure task-loss SGD,
$\varepsilon$-stationary in $\Order(\sigma_J^2/\varepsilon^2)$ steps under
(i)--(ii) \citep{ghadimi2013stochastic}; as the collapsed ``ignore everything''
solution is non-stationary for $J$ on non-degenerate tasks, warm-up moves off
collapse before the KL is switched on, decoupling cold-start collapse from
steady-state alignment. The linear ramp used in practice replaces
$\lambda^\star$ by its running average, inflating the constants below by at
most~$2$.

\paragraph{Steady-state contraction.}
Proofs here are sketches, each naming the standard step it invokes.

\begin{theorem}[Steady-state contraction toward zero-KL set]
\label{thm:sdcc:contract}
Under \textbf{(i)--(iv)} with $\mu_\Phi$ the constant of (iv), constant step
size $\eta = \min(1/L_\Phi,\, c/\sqrt{T})$, and $T$ post-warm-up SGD steps
(distinct from the rollout-end $T$ of \S\ref{sec:failure_modes}),
\begin{equation}
\min_{t \le T} \Exp\bigl[\Phi(\theta_t)\bigr]
\;\le\;
\Phi(\theta_{T_{\mathrm{warm}}}) \big/ (\mu_\Phi \eta T)
\;+\;
\bigl(L_\Phi \eta / 2\mu_\Phi\bigr)\bigl(\sigma_J^2 + (\lambda^\star)^2 \sigma_K^2\bigr),
\label{eq:sdcc:phi_rate}
\end{equation}
an initial-gap decay plus a noise floor. As $\Phi \ge \lambda^\star K$
pointwise, \eqref{eq:sdcc:phi_rate} over $\lambda^\star$ bounds $\min_{t \le T}
\Exp[K(\theta_t)] = \Order(1/\sqrt{T})$ at $\eta = c/\sqrt{T}$.
\end{theorem}
\begin{proof}[Proof sketch]
$L_\Phi$-smoothness with (ii) gives the standard descent step
$\Exp\Phi(\theta_{t+1}) \le \Phi(\theta_t) -
\tfrac{\eta}{2}\|\nabla\Phi(\theta_t)\|^2 + \tfrac{\eta^2 L_\Phi}{2}(\sigma_J^2
+ (\lambda^\star)^2\sigma_K^2)$ for $\eta \le 1/L_\Phi$; telescope, divide by
$T$, apply (iv). Only (iv) is non-routine: PL on $K$ alone would \emph{not}
suffice, as $\nabla J$ and $\lambda^\star\nabla K$ may cancel, while (iii)
forces $J - J^\star$ and $K$ to zero together.
\end{proof}

\paragraph{EMA targets.} With $\theta^{-}_{t+1} = \alpha\theta^{-}_t +
(1-\alpha)\theta_{t+1}$ replacing $\sg(\theta_t)$, $(\theta_t,\theta^-_t)$ is a
two-timescale stochastic approximation \citep{borkar1997stochastic}:
differencing the updates and applying discrete Gr\"onwall gives
$\Exp\|\theta^-_t - \theta_t\|^2 \le 2\eta_0(\sigma_J^2 +
(\lambda^\star)^2\sigma_K^2)/(1-\alpha)$ whenever $\eta_t \le \eta_0$ and
$1-\alpha \le \eta_0$. As the target KL is $\Order(\|\theta^-_t -
\theta_t\|^2)$ near $\theta^\star$ in the Fisher metric,
Theorem~\ref{thm:sdcc:contract} extends to EMA targets with its noise floor
inflated by $\Order(\eta/(1-\alpha))$, vanishing as $\eta \to 0$. Composing it
with the behavioral Pinsker bound of Prop.~\ref{prop:pinsker-bound} turns the KL
rate into the deployment TV bias against exact-walk (LogitTree / 4D)
conditioning, quoted without proof in \S\ref{sec:soft:bound}.

\begin{corollary}[Conditioning bias rate]
\label{cor:sdcc:rate}
Under \textbf{(i)--(iv)}, the SDCC{} policy $\pi_{\theta_T}$ satisfies
\begin{equation}
\Exp\bigl\|\, \pi_{\theta_T}(\cdot\!\mid\!\histlogt)
          -\pi_{\theta_T}(\cdot\!\mid\!\hcomp)\bigr\|_{\mathrm{TV}}
\;\le\;
\Order\!\bigl(T^{-1/4}\bigr) + \Order\!\bigl(\sqrt{\eta\,\sigma^2 / \lambda^\star}\bigr),
\label{eq:sdcc:final}
\end{equation}
an optimization term plus a noise floor, both driven to zero by $\lambda^\star =
\Theta(1)$ and $\eta_t = \Theta(1/\sqrt{t})$.
\end{corollary}

Adding Assumption~\ref{ass:bounded} and replacing
Prop.~\ref{prop:pinsker-bound} by Cor.~\ref{cor:sdcc-grad-bias} lifts
\eqref{eq:sdcc:final} to a policy-gradient bias rate with the same
$\Order(T^{-1/4})$ scaling; without that assumption the TV bound still holds
unconditionally.

\section*{Part III \; Experiment supplements: implementation, harness
knobs, metric semantics, training curves}
\addcontentsline{toc}{section}{Part III: Experiment supplements}
\label{app:part-exp}

\noindent This part supports \S\ref{sec:experiments} at the level of
detail needed to reproduce every measured entry of
\Cref{tab:grand-matrix} from a fresh container:
the rollout--training contract shared by every cell of the matrix, the
three white-box harnesses (TC-RAG's \verb|pop()|, AgentFold's periodic
\verb|fold|, MemexRL's \verb|memory(op)|), the two black-box harnesses
(Claude Code, OpenCode), and the five training-method implementations
in the \textsc{Slime} stack (\S\ref{app:impl}); then the reproduction
knobs, the full per-harness aggregation of the Q1 offline probe, what
\texttt{logdiff} does and does not certify under live compression
(\S\ref{app:logdiff-semantics}), and the training curves with
SDCC{}'s eviction-density profile (\S\ref{app:hparams}).

\section{Implementation of the harnesses and the training methods}
\label{app:impl}

This appendix documents \emph{how} each of the five harnesses edits the
live context at rollout time, and \emph{how} each of the five training
methods materializes its training input from the resulting record. It
complements \S\ref{app:hparams}, which lists the settings needed for
reproduction (optimizer schedule, retriever backend, per-method and
per-harness settings, evaluation protocol); here the emphasis is on the
mechanisms --- what data structure each component maintains, what it
emits, and where the tree of \S\ref{sec:failure_modes} (formalized in
\S\ref{app:logits-tree}) physically resides in the code. All components
run inside the \textsc{Slime} RL stack (Megatron actor $+$ SGLang rollout
engine).

\subsection{The rollout--training contract}
\label{app:impl-contract}

\paragraph{One record per rollout.}
Every harness, whether white-box or black-box, reduces to one per-rollout
record that all five methods consume identically:
\begin{equation*}
\underbrace{\hcomp}_{\text{compressed walk}}, \qquad
\underbrace{\Hfull}_{\text{physical union}}, \qquad
\underbrace{\{(\junct{k},\, \evictset{k})\}_{k=1}^{K}}_{\text{eviction records}},
\end{equation*}
where $\junct{k}$ is the physical position at which the $k$-th
compression fired and $\evictset{k}$ is the set of physical positions
it removed (\Cref{def:logits-tree}). The compressed walk $\hcomp$ is
exactly the token sequence the rollout engine rendered for the final
turn, and $\Hfull$ is the union of
all tokens that ever existed in the live view. The record travels
with each training sample, together with the per-token
log-probabilities logged by the rollout engine at decode time (the
reference side of the \texttt{logdiff} diagnostic of
\S\ref{sec:exp-grand-matrix}).

\paragraph{The harness bridge.}
Each white-box harness implements one bridge adapter with two
responsibilities:
(i)~\emph{execute} the edit inside the rollout loop --- rewrite
the surviving message list before the next decoding step, so that the
next request payload sent to the rollout engine is truly the
compressed view; and (ii)~\emph{log} the edit as an eviction record
$(\junct{k}, \evictset{k})$ in token coordinates of $\Hfull$. Because
the bridge is the only component that knows the harness's internal
semantics (stack, fold, memory store), the downstream loss hooks are
fully harness-agnostic: they see only the record above.

\subsection{Harnesses and methods}
\label{app:impl-harness}

We now describe, at the level of the mechanism rather than the code, how
each editor rewrites the live context and how each method reads the
resulting record. The harness and the method are selected by two
independent switches: any method runs against any white-box harness
without per-cell code divergence, which is the engineering property
behind the cross-harness sweep of \S\ref{sec:exp-q3}.

\paragraph{White-box editors.}
All three white-box editors follow the \emph{learnable-action} route:
wherever the original system specifies a compression operation, we
expose it to the policy as a first-class tool call, so the compression
decision itself is trainable. They differ in
trigger and tree shape.
\emph{TC-RAG}~(\label{app:impl-tcrag}a learnable \texttt{pop} over a stack of
tool-result envelopes): a \texttt{pop} removes the oldest envelope
entirely, giving few junctions with heavy spans (mean
$|\evictset{k}|\!\approx\!156$ tokens; \Cref{tab:tree-counts}); it is
model-triggered, so an untrained policy leaves $K\!=\!0$ and the tree
collapses to its trunk.
\emph{AgentFold}~(\label{app:impl-agentfold}length-triggered folding): when the
rendered context exceeds $3{,}000$ tokens the policy summarizes the
foldable prefix and the prefix is replaced by that summary in the live
view; the summary is on-policy and receives gradients, while the folded
prefix is logged as $\evictset{k}$ (mean $|\evictset{k}|\!\approx\!84$).
Being length-triggered, it fires even for an untrained policy, which
makes it the natural pilot harness.
\emph{MemexRL}~(\label{app:impl-MemexRL}an external key--value store with
\texttt{memory\_offload}/\texttt{memory\_retrieve}): offload evicts a
named span, retrieve re-injects it, and a terminal restore pulls all
stored entries back before the answer turn. This yields the richest
tree --- many small-span junctions (mean $|\evictset{k}|\!\approx\!41$)
and \emph{non-monotone} spans (a token can leave and re-enter), each
event logged so the live view is reconstructible by replay.
Per-turn state transitions are illustrated in \Cref{fig:harness-cases}.

\paragraph{Black-box editors.}
\label{app:impl-blackbox}
Claude Code and OpenCode are deployed agents whose context management
(sliding window $+$ auto-compact) is internal and unobservable: they surface the per-turn
rendered transcript but not eviction spans or summary provenance.
Running through a separate agent-platform
pipeline, we recover per turn the exact request payload the framework
sent to the model and convert it into the same per-rollout record, with
two degradations: junctions are \emph{detected} as prefix divergences
between consecutive payloads (the evicted span is not identifiable, so
we store the payloads rather than a token-set), and the teacher prefix
is taken to be the previous payload. Only per-turn LogitTree and SDCC are
executable here --- both need eviction events only to be detectable, not
span-exposed --- whereas the 4D mask and the segmented $K$-forward
require span-level records. Scoring uses the
same EM/format reward as the white-box cells.

\paragraph{The five methods.}
All five methods are standard \textsc{Slime} loss hooks reading only the
one record above; none modifies the forward pass or the attention
kernel. \emph{Naive-Compressed} and \emph{Naive-Full} use the reference
policy loss unchanged and differ only in the training sequence:
Naive-Compressed trains on $\hcomp$ as rendered (conditioning
post-junction tokens on later-created summaries, Pitfall~A), while
Naive-Full re-injects every evicted span to reconstruct $\Hfull$
(conditioning on already-evicted content, Pitfall~B); both roll out
under live compression, so \Cref{tab:grand-matrix} isolates the
training-time conditioning choice alone.
\emph{LogitTree} splits the rollout at its $K$ junctions into $K{+}1$
sub-samples, each carrying the live view as it existed during that layer
and only that layer's generated tokens; \textsc{Slime} trains them as
independent samples ($K{+}1$ forward/backward passes) with
$c^{\mathrm{train}}_t=c^{\mathrm{rollout}}_t$ by construction, the
episode reward inherited unchanged.
\emph{The 4D mask}\label{app:impl-4d} reuses that same branch split but
folds it into a single forward: the branches are packed into one
sequence and the attention mask of Prop.~\ref{prop:sft_4d} admits, in
each query row, exactly the keys that were live when that token was
decoded, with position ids restarted per branch so that RoPE phases
match the rollout.
\emph{SDCC} runs a student forward on $\hcomp$ (with gradients, the
Naive-Compressed compute path) plus one gradient-free teacher forward
per junction on the reconstructed pre-eviction prefix $\histlogt$
(Eq.~\ref{eq:teacher-reconstruction}), and adds the forward-KL
$\KL(\pi_\theta(\cdot\mid\histlogt)\,\|\,\pi_\theta(\cdot\mid\hcomp[:t]))$
only at diverging-leaf positions, with coefficient $\lambda$. The
teacher shares the student's weights (no second model in memory), and
the direction is fixed --- student $=$ compressed view, teacher $=$
pre-eviction view, never the reverse --- because the student is scored
on the compressed replay while the original live conditional is the
stop-gradient target
(\S\ref{sec:soft}; \S\ref{app:variational}).

\paragraph{Cost and applicability.}
\label{app:impl-tradeoffs}
The three consistency-restoring methods pay for the invariant along
different axes (\Cref{tab:hard_compare}). LogitTree pays in
\emph{compute}: $K$ additional backward passes per trajectory, a
$5$--$20\times$ per-step penalty in deep-search or coding regimes
($K=5$--$20$); it is backbone-agnostic and aligns natively with
vLLM/SGLang prefix-sharing. The 4D mask needs one forward and one
backward, but pays in \emph{infrastructure}: the model must accept an
arbitrary per-row mask, position ids must be reassigned row-wise, and
mask equivalence has to be re-checked at serving time. SDCC keeps a single
backward pass (though, with $\ell_{K+1}=C$, no smaller forward budget) and
needs eviction events only to be detectable, so it is
the only one of the three that also runs in the black-box case, at the
price of an $\Order(\sqrt{\epsKL})$ correctness residual.

\begin{table}[h]
\centering\small
\caption{The three consistency-restoring methods vs.\ the two pitfalls,
across dimensions that matter for production RL stacks, with the
per-rollout token budget of \S\ref{app:logits-tree} folded in:
$L=|\Hfull|$ is the physical union, $C=|\hcomp|$ the compressed walk, and
$\ell_i$ the live view in force during branch layer $i$, so
$C \le \ell_i \le L$ with $\ell_{K+1}=C$. Tokens are not FLOPs
(\S\ref{app:logits-tree}).}
\label{tab:hard_compare}
\resizebox{\linewidth}{!}{%
\begin{tabular}{lcccc}
\toprule
Dimension                           & Naive-Comp / Naive-Full & 4D mask                & LogitTree               & SDCC{} \\
\midrule
Restores conditioning invariant     & \xmark{}                & \cmark{} (dense only)  & \cmark{} (any backbone) & \cmark{} up to $\Order(\sqrt{\epsKL})$ \\
Forward passes / trajectory         & 1                       & 1                      & $K{+}1$                 & $\le K{+}1$ ($K$ stop-gradient) \\
Backward passes / trajectory        & 1                       & 1                      & $K{+}1$                 & 1 \\
Tokens per trajectory               & $C$ / $L$               & $L$ (packed)           & $\sum_{i=1}^{K+1}\ell_i \ge L$ & $\le C + \sum_{i=1}^{K}\ell_i$ \\
Custom attention kernel required    & no                      & \emph{yes}             & no                      & no \\
Tree-structured data pipeline       & no                      & no                     & \emph{yes}              & no (leaf mask only) \\
Position-id reassignment required   & no                      & \emph{yes} (row-wise)  & no                      & no \\
Needs eviction spans exposed        & no                      & \emph{yes}             & \emph{yes} (segmented)  & no (divergence suffices) \\
vLLM / SGLang alignment             & native                  & \xmark{}               & \cmark{}                & native \\
\bottomrule
\end{tabular}}
\end{table}

\section{Experiment Details}
\label{app:hparams}

This appendix is designed to let an external reader reproduce every entry of
\cref{tab:grand-matrix} and \cref{tab:main}. The training stack is
\textsc{Slime} (Megatron actor + SGLang rollout engine); we give the
macroscopic setup below, with pinned versions and every non-default
setting shipped in the supplementary code.

\subsection{Experimental Settings Details}
\label{app:settings}

\paragraph{Compute environment.} All runs use the \textsc{Slime} RL
framework (\textsc{Megatron-LM} actor $+$ \textsc{SGLang} rollout
engine) on 80\,GB A100- or H800-class accelerators with CUDA/PyTorch
and FlashAttention; pinned versions and the full dependency list ship
with the supplementary code.

\paragraph{Optimizer and schedule.} Adam ($\beta_1{=}0.9$,
$\beta_2{=}0.98$) with constant learning rate $1\!\times\!10^{-6}$,
weight decay $0.01$ and gradient clip $1.0$; GRPO with group size
$G{=}16$ samples per prompt, clip range $\eps{=}0.2$ (low side) /
$0.28$ (high side) and KL coefficient $\beta{=}0.001$ (low-variance
estimator). For SDCC{}, the consistency coefficient $\lambda$ ramps
linearly from $0$ to $0.1$ over the first $20\%$ of update steps and
stays constant thereafter; we deliberately do not tune this schedule
per editor, so that no reported gain is an artifact of hyperparameter
sweeping. The headline sweep uses a fixed step budget per cell, whereas
the convergence campaign trains each cell to a reward plateau (early
stopping once the trailing reward mean stops improving), capped at
$200$ rollouts.

\paragraph{Rollout.} At most $T{=}30$ tool-call turns with top-$k{=}3$
retrieved chunks per call, rollout temperature $1.0$, $n{=}16$ samples
per prompt and $4$ prompts per rollout, giving $64$ trajectories per
rollout, which are consumed in $8$ optimizer steps of $8$ trajectories
each; advantages are normalized within each $16$-sample group (GRPO
default), so the group, not the optimizer step, sets the baseline.

\paragraph{Retriever and search backend.} All rollouts --- during both
training and evaluation --- issue live web searches rather than
querying a local index. Web search is served by the online DashScope
text-search API, which returns up to $10$ ranked results per query, and
full-page evidence is extracted by a two-stage pipeline that fetches
the retrieved pages with Firecrawl and summarizes them with
Qwen3-Turbo. Because the same endpoint serves training and evaluation,
the inference-time search distribution matches the training-time one.

\paragraph{Training corpus.}
\label{app:data}
The training corpus pools
\textsc{RedSearcher}~\citep{chu2026redsearcherscalablecostefficientframework}
and the \textsc{ASearcher} agentic-search training set into $81{,}638$
composite QA instances, normalized to a uniform question--answer schema.
Its composite, multi-hop questions typically require at least two retrieval rounds,
giving the model-triggered editors (TC-RAG \texttt{pop}, MemexRL
\texttt{memory\_offload}) and AgentFold's length-triggered fold the
opportunity to actually fire (see the eviction-density profile in
\Cref{tab:phase2-longrun}).

\paragraph{Evaluation benchmarks.} Each trained checkpoint is evaluated
on seven heterogeneous held-out open-domain QA benchmarks
($38{,}270$ questions per checkpoint in total), all run through the
full agentic loop (multi-turn search, live compression by the cell's
harness): \textsc{NQ} \citep{kwiatkowski2019nq} ($3{,}610$ questions,
single-hop factoid), \textsc{TriviaQA} \citep{joshi2017triviaqa}
($11{,}313$, single-hop trivia), \textsc{HotpotQA}
\citep{yang2018hotpotqa} ($7{,}405$, two-hop),
\textsc{2WikiMultiHopQA} \citep{ho20202wiki} ($12{,}576$, structured
multi-hop), \textsc{MuSiQue} \citep{trivedi2022musique} ($2{,}417$,
2--4-hop compositional), \textsc{Bamboogle} \citep{press2023selfask}
($125$, hand-curated compositional), and \textsc{FRAMES}
\citep{krishna2024frames} ($824$, multi-constraint
retrieval-and-reasoning).

\paragraph{Harness knobs.}
\label{app:harness-knobs}
The three white-box harnesses are selected by a single harness switch.
All three compress \emph{live}, inside the rollout turn loop --- the
surviving messages are rewritten before the next decoding step, and the
training and evaluation code consume the same compressed view --- and
all three emit the same eviction-record contract
($\{J_k, \evictset{k}\}$), so the downstream loss hook is
harness-agnostic. \textbf{MemexRL} is model-triggered, via a
\texttt{memory\_\allowbreak offload} / \texttt{memory\_\allowbreak retrieve}
tool pair over a session-local store with a terminal restore before the
answer turn (\S\ref{app:impl-MemexRL}); \textbf{TC-RAG} is
model-triggered, keeping a stack of tool-result envelopes with a
learnable \texttt{pop} that evicts the oldest and no rule-based
schedule (\S\ref{app:impl-tcrag}); \textbf{AgentFold} is
length-triggered, folding the prefix into a self-generated summary once
the live context exceeds $3{,}000$ tokens --- the only rule-based
trigger, and hence the highest-density editor at cold start
(\Cref{tab:phase2-longrun}; \S\ref{app:impl-agentfold}). Complete
mechanism descriptions --- action spaces, triggers, and what each
editor stores and restores --- are in \S\ref{app:impl-harness}. For the
two black-box harnesses (Claude Code, OpenCode) the rollout channel is
supplied by a separate integration pipeline outside this codebase; we
consume its outputs in the same per-rollout record format, so the SDCC
/ LogitTree loss hooks are unchanged (\S\ref{app:impl-blackbox}), and
the per-cell scoring uses the same internal EM/format reward as the
white-box cells, making the black-box column directly comparable.

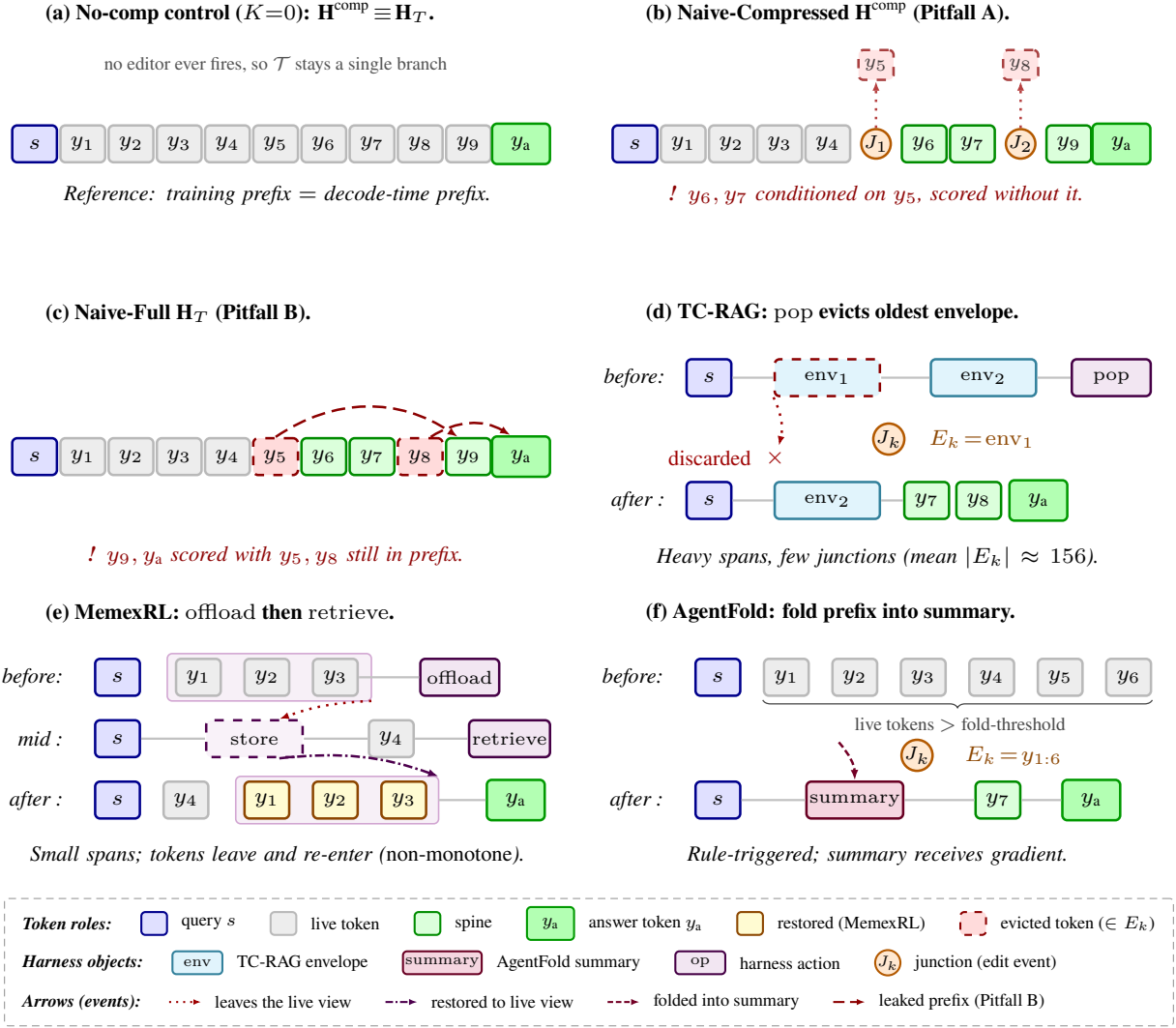
\begin{figure}[t]
\centering

\tikzset{
  hc/.style={
    font=\scriptsize,
    tok/.style={draw, rounded corners=1.6pt, minimum width=5.1mm,
                minimum height=4.2mm, inner sep=1pt, thick,
                font=\scriptsize},
    query/.style={tok, fill=blue!12, draw=blue!55!black},
    shared/.style={tok, fill=gray!14, draw=gray!55},
    spine/.style={tok, fill=green!14, draw=green!55!black},
    leg/.style={tok, fill=yellow!22, draw=orange!55!black},
    evict/.style={tok, fill=red!14, draw=red!55!black, dashed},
    answer/.style={tok, fill=green!30, draw=green!60!black,
                   minimum width=6.6mm, minimum height=4.6mm,
                   font=\scriptsize\bfseries},
    env/.style={tok, fill=cyan!10, draw=cyan!55!black,
                minimum width=12mm, font=\tiny},
    slot/.style={tok, fill=violet!10, draw=violet!55!black,
                 minimum width=9mm, font=\tiny},
    summ/.style={tok, fill=purple!16, draw=purple!55!black,
                 minimum width=10mm, font=\tiny\itshape},
    junc/.style={circle, draw=orange!70!black, fill=orange!18, thick,
                 inner sep=0pt, minimum size=3.2mm, font=\tiny\bfseries},
    spinelink/.style={draw=gray!48, line width=0.7pt},
    evictarr/.style={-{Latex[length=1.1mm]}, thick, dotted, red!62!black},
    reinject/.style={-{Latex[length=1.2mm]}, thick, red!55!black,
                     dash pattern=on 4.5pt off 2.2pt},
    restore/.style={-{Latex[length=1.2mm]}, thick, violet!60!black,
                    densely dashdotted},
    fold/.style={-{Latex[length=1.2mm]}, thick, purple!60!black,
                 dash pattern=on 2.2pt off 1.1pt},
    grp/.style={rounded corners=1.6pt, fill=violet!6, draw=violet!35,
                line width=0.5pt},
    hint/.style={font=\tiny, text=gray!62!black, align=center},
    actionlbl/.style={font=\scriptsize\itshape, orange!55!black},
    paneltitle/.style={font=\scriptsize\bfseries, anchor=south west},
    subcap/.style={font=\scriptsize\itshape, text width=6.0cm,
                   align=center, anchor=north},
  }
}

\def\colL{0.0}
\def\colR{6.85}

\resizebox{\linewidth}{!}{%
\begin{tikzpicture}[hc]
\path[use as bounding box] (-0.35,-1.82) rectangle (13.00,1.45);

\begin{scope}[shift={(\colL, -0.20)}]
\node[paneltitle] at (0, 1.25)
  {(a) No-comp control ($K{=}0$): $\hcomp\!\equiv\!\Hfull$.};
\node[query]  (a_s)  at (0.00, 0)  {$s$};
\node[shared] (a_y1) at (0.55, 0)  {$y_1$};
\node[shared] (a_y2) at (1.10, 0)  {$y_2$};
\node[shared] (a_y3) at (1.65, 0)  {$y_3$};
\node[shared] (a_y4) at (2.20, 0)  {$y_4$};
\node[shared] (a_y5) at (2.75, 0)  {$y_5$};
\node[shared] (a_y6) at (3.30, 0)  {$y_6$};
\node[shared] (a_y7) at (3.85, 0)  {$y_7$};
\node[shared] (a_y8) at (4.40, 0)  {$y_8$};
\node[shared] (a_y9) at (4.95, 0)  {$y_9$};
\node[answer] (a_yA) at (5.55, 0)  {$y_{\mathrm{a}}$};
\node[hint] at (2.75, 0.90)
  {no editor ever fires, so $\Treestruct$ stays a single branch};
\node[subcap] at (2.75, -0.35)
  {Reference: training prefix $=$ decode-time prefix.};
\end{scope}

\begin{scope}[shift={(\colR, -0.20)}]
\node[paneltitle] at (0, 1.25)
  {(b) Naive-Compressed $\hcomp$ (Pitfall A).};
\node[query]  (b_s)  at (0.00, 0)  {$s$};
\node[shared] (b_y1) at (0.55, 0)  {$y_1$};
\node[shared] (b_y2) at (1.10, 0)  {$y_2$};
\node[shared] (b_y3) at (1.65, 0)  {$y_3$};
\node[shared] (b_y4) at (2.20, 0)  {$y_4$};
\node[junc]   (b_J1) at (2.75, 0)  {$\junct{1}$};
\node[spine]  (b_y6) at (3.30, 0)  {$y_6$};
\node[spine]  (b_y7) at (3.85, 0)  {$y_7$};
\node[junc]   (b_J2) at (4.40, 0)  {$\junct{2}$};
\node[spine]  (b_y9) at (4.95, 0)  {$y_9$};
\node[answer] (b_yA) at (5.55, 0)  {$y_{\mathrm{a}}$};
\node[evict, minimum width=4.0mm, minimum height=3.4mm, opacity=0.75,
      font=\tiny] (b_g5) at (2.75, 0.90) {$y_5$};
\node[evict, minimum width=4.0mm, minimum height=3.4mm, opacity=0.75,
      font=\tiny] (b_g8) at (4.40, 0.90) {$y_8$};
\draw[evictarr, opacity=0.75] (b_J1.north) -- (b_g5.south);
\draw[evictarr, opacity=0.75] (b_J2.north) -- (b_g8.south);
\node[subcap, red!55!black] at (2.75, -0.35)
  {\textbf{!} $y_6,y_7$ conditioned on $y_5$, scored without it.};
\end{scope}

\end{tikzpicture}%
}

\vspace{4pt}

\resizebox{\linewidth}{!}{%
\begin{tikzpicture}[hc]
\path[use as bounding box] (-0.35,-1.82) rectangle (13.00,1.45);

\begin{scope}[shift={(\colL, -0.35)}]
\node[paneltitle] at (0, 1.40)
  {(c) Naive-Full $\Hfull$ (Pitfall B).};
\node[query]  (c_s)  at (0.00, 0)  {$s$};
\node[shared] (c_y1) at (0.55, 0)  {$y_1$};
\node[shared] (c_y2) at (1.10, 0)  {$y_2$};
\node[shared] (c_y3) at (1.65, 0)  {$y_3$};
\node[shared] (c_y4) at (2.20, 0)  {$y_4$};
\node[evict]  (c_y5) at (2.75, 0)  {$y_5$};
\node[spine]  (c_y6) at (3.30, 0)  {$y_6$};
\node[spine]  (c_y7) at (3.85, 0)  {$y_7$};
\node[evict]  (c_y8) at (4.40, 0)  {$y_8$};
\node[spine]  (c_y9) at (4.95, 0)  {$y_9$};
\node[answer] (c_yA) at (5.55, 0)  {$y_{\mathrm{a}}$};
\draw[reinject] (c_y5.north) to[bend left=34] ([xshift=-1.2mm]c_y9.north);
\draw[reinject] ([xshift=1.2mm]c_y8.north) to[bend left=32]
                ([xshift=-1.2mm]c_yA.north);
\node[subcap, red!55!black] at (2.75, -0.90)
  {\textbf{!} $y_9,y_{\mathrm{a}}$ scored with $y_5,y_8$ still in prefix.};
\end{scope}

\begin{scope}[shift={(\colR, 0)}]
\node[paneltitle] at (0, 1.05)
  {(d) TC-RAG: \texttt{pop} evicts oldest envelope.};
\begin{scope}[shift={(0.52,0)},xscale=1.08]

\node[font=\scriptsize\itshape, anchor=east] at (-0.05, 0.55) {before:};
\node[query]  (d_s0)  at (0.30, 0.55) {$s$};
\node[env, draw=red!55!black, dashed]
              (d_E1a) at (1.55, 0.55) {\texttt{env}$_{1}$};
\node[env]    (d_E2a) at (3.20, 0.55) {\texttt{env}$_{2}$};
\node[slot]   (d_pop) at (4.55, 0.55) {\texttt{pop}};
\foreach \a/\b in {d_s0/d_E1a,d_E1a/d_E2a,d_E2a/d_pop}{
  \draw[spinelink] (\a) -- (\b);
}

\node[junc] (d_Jk) at (2.20, -0.15) {$\junct{k}$};
\node[actionlbl, anchor=west] at (2.45, -0.15)
  {\;$\evictset{k}\!=\!\texttt{env}_{1}$};

\node[font=\scriptsize\itshape, anchor=east] at (-0.05, -0.85) {after :};
\node[query]  (d_s1)  at (0.30, -0.85) {$s$};
\node[env]    (d_E2b) at (1.55, -0.85) {\texttt{env}$_{2}$};
\node[spine]  (d_y1)  at (2.60, -0.85) {$y_{7}$};
\node[spine]  (d_y2)  at (3.15, -0.85) {$y_{8}$};
\node[answer] (d_yAd) at (3.78, -0.85) {$y_{\mathrm{a}}$};
\foreach \a/\b in {d_s1/d_E2b,d_E2b/d_y1}{
  \draw[spinelink] (\a) -- (\b);
}
\node[font=\scriptsize\bfseries, text=red!62!black, inner sep=1pt]
      (d_bin) at (1.00, -0.36) {$\times$};
\draw[evictarr] (d_E1a.south west) to[bend left=18] (d_bin);
\node[font=\scriptsize, red!62!black, anchor=east] at (0.86, -0.36) {discarded};
\end{scope}

\node[subcap] at (2.75, -1.25)
  {Heavy spans, few junctions (mean $|\evictset{k}|\!\approx\!156$).};
\end{scope}

\end{tikzpicture}%
}

\vspace{4pt}

\resizebox{\linewidth}{!}{%
\begin{tikzpicture}[hc]
\path[use as bounding box] (-0.35,-1.82) rectangle (13.00,1.45);

\begin{scope}[shift={(\colL, 0)}]
\node[paneltitle] at (0, 1.05)
  {(e) MemexRL: \texttt{offload} then \texttt{retrieve}.};
\begin{scope}[shift={(0.52,0)},xscale=1.42]

\node[font=\scriptsize\itshape, anchor=east] at (-0.05, 0.55) {before:};
\path[grp] (0.70, 0.82) rectangle (2.34, 0.28);
\node[query]  (e_s0)  at (0.30, 0.55) {$s$};
\node[shared] (e_y1a) at (0.95, 0.55) {$y_1$};
\node[shared] (e_y2a) at (1.50, 0.55) {$y_2$};
\node[shared] (e_y3a) at (2.05, 0.55) {$y_3$};
\node[slot]   (e_op)  at (3.05, 0.55) {\texttt{offload}};
\draw[spinelink] (e_y3a) -- (e_op);

\node[font=\scriptsize\itshape, anchor=east] at (-0.05, -0.15) {mid   :};
\node[query]  (e_s1)  at (0.30, -0.15) {$s$};
\node[draw=violet!60!black, dashed, thick, rounded corners=1.6pt,
      minimum width=11mm, minimum height=4.2mm, fill=violet!5,
      font=\tiny, inner sep=1pt] (e_gh) at (1.40, -0.15) {\texttt{store}};
\node[shared] (e_y4b) at (2.50, -0.15) {$y_4$};
\node[slot]   (e_rt)  at (3.45, -0.15) {\texttt{retrieve}};
\foreach \a/\b in {e_s1/e_gh,e_gh/e_y4b,e_y4b/e_rt}{
  \draw[spinelink] (\a) -- (\b);
}

\node[font=\scriptsize\itshape, anchor=east] at (-0.05, -0.85) {after :};
\path[grp] (1.25, -0.58) rectangle (2.88, -1.12);
\node[query]  (e_s2)  at (0.30, -0.85) {$s$};
\node[shared] (e_y4c) at (0.85, -0.85) {$y_4$};
\node[leg]    (e_r1)  at (1.50, -0.85) {$y_1$};
\node[leg]    (e_r2)  at (2.05, -0.85) {$y_2$};
\node[leg]    (e_r3)  at (2.60, -0.85) {$y_3$};
\node[answer] (e_yAc) at (3.50, -0.85) {$y_{\mathrm{a}}$};
\draw[spinelink] (2.88, -0.85) -- (e_yAc.west);
\draw[evictarr] (2.34, 0.28) to[bend right=12]
                ([xshift=2mm]e_gh.north);
\draw[restore]  ([xshift=2mm]e_gh.south) to[bend left=18] (2.88, -0.58);
\end{scope}
\node[subcap] at (2.75, -1.25)
  {Small spans; tokens leave and re-enter (\emph{non-monotone}).};
\end{scope}

\begin{scope}[shift={(\colR, 0)}]
\node[paneltitle] at (0, 1.05)
  {(f) AgentFold: fold prefix into summary.};
\begin{scope}[shift={(0.52,0)},xscale=1.42]

\node[font=\scriptsize\itshape, anchor=east] at (-0.05, 0.55) {before:};
\node[query]  (f_s0)  at (0.30, 0.55) {$s$};
\node[shared] (f_y1a) at (0.85, 0.55) {$y_1$};
\node[shared] (f_y2a) at (1.40, 0.55) {$y_2$};
\node[shared] (f_y3a) at (1.95, 0.55) {$y_3$};
\node[shared] (f_y4a) at (2.50, 0.55) {$y_4$};
\node[shared] (f_y5a) at (3.05, 0.55) {$y_5$};
\node[shared] (f_y6a) at (3.60, 0.55) {$y_6$};
\draw[decorate, decoration={brace, mirror, amplitude=2.5pt},
      gray!55!black]
  ([yshift=-1pt]f_y1a.south west) -- ([yshift=-1pt]f_y6a.south east)
  node[midway, below=2pt, font=\tiny, gray!55!black] (f_thr)
  {\;live tokens $>$ fold-threshold};

\node[junc] (f_Jk) at (1.90, -0.34) {$\junct{k}$};
\node[actionlbl, anchor=west] at (2.15, -0.34)
  {\;$\evictset{k}\!=\!y_{1:6}$};

\node[font=\scriptsize\itshape, anchor=east] at (-0.05, -0.85) {after :};
\node[query]  (f_s1)  at (0.30, -0.85) {$s$};
\node[summ]   (f_sm)  at (1.40, -0.85) {\texttt{summary}};
\node[spine]  (f_y7b) at (2.55, -0.85) {$y_7$};
\node[answer] (f_yAf) at (3.30, -0.85) {$y_{\mathrm{a}}$};
\foreach \a/\b in {f_s1/f_sm,f_sm/f_y7b,f_y7b/f_yAf}{
  \draw[spinelink] (\a) -- (\b);
}
\draw[fold] (f_thr.south west) to[bend left=15] (f_sm.north);

\end{scope}
\node[subcap] at (2.75, -1.25)
  {Rule-triggered; summary receives gradient.};
\end{scope}

\end{tikzpicture}%
}

\vspace{4pt}

\resizebox{\linewidth}{!}{%
\begin{tikzpicture}[hc]
\node[font=\scriptsize\bfseries, anchor=west] (leg_t) at (0, 0)
  {\itshape Token roles:};

\node[query,  minimum width=3.6mm, minimum height=3.2mm, anchor=west]
      (lg_q)   at ([xshift=2.7mm]leg_t.east) {};
\node[anchor=west, font=\scriptsize] (lg_qL) at (lg_q.east) {\;query $s$};

\node[shared, minimum width=3.6mm, minimum height=3.2mm, anchor=west]
      (lg_sh)  at ([xshift=3.5mm]lg_qL.east) {};
\node[anchor=west, font=\scriptsize] (lg_shL) at (lg_sh.east) {\;live token};

\node[spine,  minimum width=3.6mm, minimum height=3.2mm, anchor=west]
      (lg_sp)  at ([xshift=3.5mm]lg_shL.east) {};
\node[anchor=west, font=\scriptsize] (lg_spL) at (lg_sp.east) {\;spine};

\node[answer, anchor=west] (lg_an) at ([xshift=3.5mm]lg_spL.east)
      {$y_{\mathrm{a}}$};
\node[anchor=west, font=\scriptsize] (lg_anL) at (lg_an.east)
      {\;answer token $y_{\mathrm{a}}$};

\node[leg,    minimum width=3.6mm, minimum height=3.2mm, anchor=west]
      (lg_lg)  at ([xshift=3.5mm]lg_anL.east) {};
\node[anchor=west, font=\scriptsize] (lg_lgL) at (lg_lg.east) {\;restored (MemexRL)};

\node[evict,  minimum width=3.6mm, minimum height=3.2mm, anchor=west]
      (lg_ev)  at ([xshift=3.5mm]lg_lgL.east) {};
\node[anchor=west, font=\scriptsize] at (lg_ev.east)
      {\;evicted token ($\in\evictset{k}$)};

\node[font=\scriptsize\bfseries\itshape, anchor=west] (leg_b) at (0, -0.55)
  {Harness objects:};
\node[env,    minimum width=7mm, minimum height=3.2mm, font=\tiny,
      anchor=west] (lg_en)  at ([xshift=2.7mm]leg_b.east) {\texttt{env}};
\node[anchor=west, font=\scriptsize] (lg_enL) at (lg_en.east)
      {\;TC-RAG envelope};

\node[summ,   minimum width=7mm, minimum height=3.2mm, font=\tiny\itshape,
      anchor=west] (lg_sm)  at ([xshift=3.5mm]lg_enL.east) {\texttt{summary}};
\node[anchor=west, font=\scriptsize] (lg_smL) at (lg_sm.east)
      {\;AgentFold summary};

\node[slot,   minimum width=7mm, minimum height=3.2mm, font=\tiny,
      anchor=west] (lg_sl)  at ([xshift=3.5mm]lg_smL.east) {\texttt{op}};
\node[anchor=west, font=\scriptsize] (lg_slL) at (lg_sl.east)
      {\;harness action};

\node[junc, anchor=west] (lg_j) at ([xshift=3.5mm]lg_slL.east)
      {$\junct{k}$};
\node[anchor=west, font=\scriptsize] at (lg_j.east)
      {\;junction (edit event)};

\node[font=\scriptsize\bfseries\itshape, anchor=west] (leg_c) at (0, -1.10)
  {Arrows (events):};
\draw[evictarr] ([xshift=2.7mm]leg_c.east) -- ([xshift=7.2mm]leg_c.east);
\node[anchor=west, font=\scriptsize] (lg_c1) at ([xshift=7.2mm]leg_c.east)
      {\;leaves the live view};

\draw[restore]  ([xshift=3.5mm]lg_c1.east) -- ([xshift=8.0mm]lg_c1.east);
\node[anchor=west, font=\scriptsize] (lg_c2) at ([xshift=8.0mm]lg_c1.east)
      {\;restored to live view};

\draw[fold]     ([xshift=3.5mm]lg_c2.east) -- ([xshift=8.0mm]lg_c2.east);
\node[anchor=west, font=\scriptsize] (lg_c3) at ([xshift=8.0mm]lg_c2.east)
      {\;folded into summary};

\draw[reinject] ([xshift=3.5mm]lg_c3.east) -- ([xshift=8.0mm]lg_c3.east);
\node[anchor=west, font=\scriptsize] at ([xshift=8.0mm]lg_c3.east)
      {\;leaked prefix (Pitfall~B)};
\begin{scope}[on background layer]
  \node[draw=black!38, densely dashed, rounded corners=1.5pt,
        line width=0.45pt, inner xsep=4pt, inner ysep=3pt,
        fit=(current bounding box)] {};
\end{scope}
\end{tikzpicture}%
}

\caption{\textbf{Six cases on the shared token-node vocabulary of
Fig.~\ref{fig:traj-tree}}, whose running example and subscripts these
panels reuse; $\junct{3}$ falls inside the abridged tail
$y_{\mathrm{a}}$.\; Panels split \textbf{by role, not by row}: Part~I
\textbf{(a)--(c)} are training-input recipes on one rollout, exposing the
Pitfall~A ``too-short'' and Pitfall~B ``too-long'' mismatches;
Part~II \textbf{(d)--(f)} draw one live compression event per harness
(before$\to$after; $\junct{k}$ is generic, indices are panel-local).
See \S\ref{app:impl-harness} for harness code.}
\label{fig:harness-cases}
\end{figure}

\begin{figure}[t]
\centering

\tikzset{
  bb/.style={
    font=\scriptsize,
    tok/.style={draw, line width=0.45pt, rounded corners=1.1pt,
                minimum width=4.6mm, minimum height=3.7mm,
                inner sep=0.4pt, font=\scriptsize},
    query/.style ={tok, fill=blue!10,  draw=blue!60!black},
    shared/.style={tok, fill=black!4,  draw=black!48},
    spine/.style ={tok, fill=green!12, draw=green!50!black},
    evict/.style ={tok, fill=red!7,    draw=red!62!black, densely dashed},
    answer/.style={tok, fill=green!28, draw=green!55!black,
                   minimum width=5.6mm, font=\scriptsize\bfseries},
    ghost/.style ={tok, draw=none, fill=none, text=black!34},
    summ/.style  ={tok, fill=purple!12, draw=purple!55!black,
                   minimum width=8.4mm,
                   font=\fontsize{6.0}{7.0}\selectfont\itshape},
    proc/.style  ={draw=black!52, line width=0.5pt, rounded corners=1.8pt,
                   fill=black!3, inner sep=2.4pt, font=\scriptsize},
    mbox/.style  ={draw=black!45, line width=0.4pt, rounded corners=1.8pt,
                   fill=black!3, inner sep=2.4pt, align=center,
                   text width=1.62cm, font=\scriptsize},
    note/.style  ={font=\fontsize{6.4}{7.4}\selectfont, anchor=west,
                   align=left, inner sep=0pt},
    legend/.style={font=\fontsize{6.4}{6.9}\selectfont, text=black!55,
                   anchor=north west, align=left, inner sep=0pt},
    jlab/.style  ={font=\fontsize{6.4}{7.4}\selectfont,
                   text=orange!58!black, anchor=east, inner sep=0.6pt,
                   fill=white},
    jrule/.style ={draw=orange!50!black, line width=0.3pt, densely dashed},
    jarr/.style  ={draw=orange!62!black, line width=0.45pt,
                   -{Latex[length=1.1mm]}, shorten >=0.5pt,
                   shorten <=0.5pt},
    dead/.style  ={draw=black!32, line width=0.3pt},
    payarr/.style={draw=blue!60!black, line width=0.5pt,
                   -{Latex[length=1.1mm]}},
    decarr/.style={draw=green!50!black, line width=0.5pt,
                   -{Latex[length=1.1mm]}},
    forkedge/.style={draw=black!62, line width=0.7pt, line cap=round,
                     line join=round},
    forkdot/.style ={fill=orange!62!black, draw=none},
  }
}
\newcommand{\bbpaneltitle}[2]{%
  \parbox[t][4.35ex][t]{\linewidth}{%
    \raggedright\footnotesize
    \textbf{#1}\par
    \vspace{0.15ex}%
    {\scriptsize\color{black!62}#2}\par}%
  \vspace{0.55ex}%
}

\begin{minipage}[t]{0.255\linewidth}
\centering
\bbpaneltitle{(a) The tap}{Messages in, tokens out.}

\resizebox{\linewidth}{!}{%
\begin{tikzpicture}[bb]
\node[proc, align=center, minimum width=2.05cm] (hz) at (1.0,0)
  {black-box harness\\[-0.25ex]
   {\fontsize{6.2}{7.2}\selectfont Claude Code / OpenCode}};
\node[proc, align=center, minimum width=2.05cm, fill=black!9]
  (pol) at (1.0,-1.62) {served policy};

\draw[payarr] (0.58,-0.33) -- (0.58,-1.42);
\draw[decarr] (0.93,-1.42) -- (0.93,-0.33);
\node[note, text=blue!45!black]   at (1.10,-0.66) {rendered payload $p_k$};
\node[note, text=green!38!black]  at (1.10,-1.06) {decoded tokens};

\node[draw=orange!55!black, fill=orange!5, line width=0.4pt,
      rounded corners=1.8pt, inner sep=2.6pt, anchor=north,
      align=left, text width=3.35cm,
      font=\fontsize{6.4}{7.4}\selectfont, text=orange!45!black]
  at (1.0,-2.02)
  {the harness compacts between turns; the stream carries
   \emph{no} eviction event};
\end{tikzpicture}}
\end{minipage}%
\hfill
\begin{minipage}[t]{0.425\linewidth}
\centering
\bbpaneltitle{(b) Consecutive payloads diverge}%
{The first mismatch dates $\junct{k}$ and names $\evictset{k}$.}

\resizebox{\linewidth}{!}{%
\begin{tikzpicture}[bb]
\def\cx{0.52}
\def\rA{0}\def\rB{-0.80}\def\rC{-1.60}

\node[legend, anchor=east] at (-0.26,\rA) {$k{=}0$};
\node[query]        at (0*\cx,\rA) {$s$};
\node[spine]        at (1*\cx,\rA) {$y_1$};
\node[evict]        at (2*\cx,\rA) {$y_2$};
\node[spine] (tA3)  at (3*\cx,\rA) {$y_3$};

\draw[jrule] (-0.06,-0.40) -- (1.83,-0.40);
\node[jlab, anchor=west] at (1.88,-0.40) {$\junct{1}$: $\evictset{1}{=}\{y_2\}$};

\node[legend, anchor=east] at (-0.26,\rB) {$k{=}1$};
\node[query]        at (0*\cx,\rB) {$s$};
\node[shared]       at (1*\cx,\rB) {$y_1$};
\node[ghost] (gB2)  at (2*\cx,\rB) {$y_2$};
\draw[dead] (gB2.west) -- (gB2.east);
\node[shared]       at (3*\cx,\rB) {$y_3$};
\node[summ]         at (4.35*\cx,\rB) {summary};
\node[spine]        at (5.70*\cx,\rB) {$y_4$};
\node[evict]        at (6.70*\cx,\rB) {$y_5$};
\node[spine] (tB6)  at (7.70*\cx,\rB) {$y_6$};
\draw[jarr] (tA3.south) to[out=-96,in=84] (gB2.north);

\draw[jrule] (-0.06,-1.20) -- (3.23,-1.20);
\node[jlab] at (3.26,-1.20) {$\junct{2}$: $\evictset{2}{=}\{y_5\}$};

\node[legend, anchor=east] at (-0.26,\rC) {$k{=}2$};
\node[query]        at (0*\cx,\rC) {$s$};
\node[shared]       at (1*\cx,\rC) {$y_1$};
\node[ghost] (gC2)  at (2*\cx,\rC) {$y_2$};
\draw[dead] (gC2.west) -- (gC2.east);
\node[shared]       at (3*\cx,\rC) {$y_3$};
\node[summ]         at (4.35*\cx,\rC) {summary};
\node[shared]       at (5.70*\cx,\rC) {$y_4$};
\node[ghost] (gC5)  at (6.70*\cx,\rC) {$y_5$};
\draw[dead] (gC5.west) -- (gC5.east);
\node[shared]       at (7.70*\cx,\rC) {$y_6$};
\node[answer]       at (8.82*\cx,\rC) {$y_{\mathrm{ans}}$};
\draw[jarr] (tB6.south) to[out=-96,in=84] (gC5.north);

\node[legend, anchor=north west] at (-0.78,-2.02)
  {row $k$ is the live view $\histlogt$ the policy actually saw};
\end{tikzpicture}}
\end{minipage}%
\hfill
\begin{minipage}[t]{0.285\linewidth}
\centering
\bbpaneltitle{(c) One record, two uses}{Same type a white-box run logs.}

\resizebox{\linewidth}{!}{%
\begin{tikzpicture}[bb]
\node[proc, align=center, minimum width=3.30cm] (rec) at (0,0)
  {recovered $\Treestruct$\\[-0.25ex]
   {\fontsize{6.2}{7.2}\selectfont one record per session}};
\draw[forkedge] (0,-0.30) -- (0,-0.62);
\draw[forkedge] (-0.98,-0.62) -- (0.98,-0.62);
\draw[forkedge] (-0.98,-0.62) -- (-0.98,-0.82);
\draw[forkedge] (0.98,-0.62)  -- (0.98,-0.82);
\path[forkdot] (0,-0.62) circle (0.045);

\node[mbox, anchor=north] at (-0.98,-0.82)
  {\textbf{LogitTree}\\[-0.15ex]
   {\fontsize{6.2}{7.2}\selectfont replay every $\histlogt$ exactly}};
\node[mbox, anchor=north] at (0.98,-0.82)
  {\textbf{SDCC}\\[-0.15ex]
   {\fontsize{6.2}{7.2}\selectfont align one walk at each $\junct{k}$}};
\end{tikzpicture}}
\end{minipage}

\caption{\small{\textbf{Recovering the trajectory tree from a black-box
harness.} \textbf{(a)} At the model boundary the only observable is one
(rendered payload, decoded tokens) pair per turn: the harness compacts
its own context between turns and puts no eviction event on the wire.
\textbf{(b)} Consecutive payloads are laid on a single column grid in the
token vocabulary of \Cref{fig:traj-tree}, whose colour key applies here
unchanged; the one addition is the purple slot, text the harness wrote
for itself, as in \Cref{fig:harness-cases}. What turn $k$ decoded (green)
reappears as carried context (grey) in the next payload \emph{except}
where the harness dropped it, and a dropped slot leaves no box at all.
The first position at which two consecutive payloads disagree therefore
dates the junction $\junct{k}$ and names the evicted set
$\evictset{k}$ (orange band and arrow), so both are recovered from the
message stream alone, with no span-level log and no cooperation from the
harness. \textbf{(c)} Each row of (b) is a live view $\histlogt$, so the
recovered $\Treestruct$ has the same type as the tree a white-box harness
logs, and the LogitTree and SDCC loss hooks consume it unchanged
(\S\ref{app:impl-blackbox}).}}
\label{fig:blackbox-prefix-tree}
\end{figure}
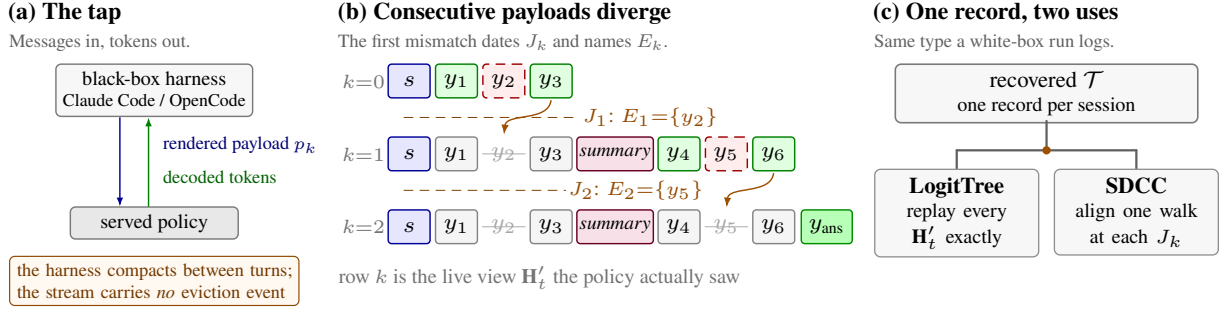

\paragraph{Evaluation protocol.}
\label{app:agentic-eval}
\label{app:eval-protocol}
The reward (M1) is Exact-Match against the gold answer with a $0.2$
format bonus for correctly tagged but factually wrong answers
(matching \citealp{jin2025searchr1}); Pass@1 (M2) is measured on the
held-out test split with a single rollout at temperature $0$; drift
(M3) is $\KL\big(\policy(\cdot\mid\histlogt) \,\|\,
\pi_{\theta,\text{train}}(\cdot\mid c^{\mathrm{train}}_t)\big)$ at
inference time, averaged over editor activations.

\subsection{Q1 offline probe: full per-harness aggregation}
\label{app:q1-full}

The per-token distributions in \Cref{fig:q1_dist} are plotted from the
same 32 synthetic 3-turn rollouts per harness (untrained Qwen3-4B, forward
passes only). \Cref{tab:q1} reports the full
per-harness aggregation behind that panel: sample counts, means, and
majority-sign fractions for both directions of the pitfall
$\Delta$, together with the theoretical prediction each row must
satisfy for the invariant of \S\ref{sec:invariant} to hold. Every
harness meets the prediction in both directions.
The reading is qualitative: leaves are clustered within rollouts and
prompts, so the per-token means and $|\Delta|$ magnitudes here should
be interpreted as sign-and-scale indicators rather than as unit-independent
population estimates. We do not attach confidence intervals to the
means for that reason; instead we report the majority-sign fraction
per harness as the primary read. The qualitative pattern (aggressive-eviction
editors amplify Naive-Comp, high-payload editors amplify Naive-Full-leak)
is discussed in the prose of \S\ref{sec:exp-q1}.

\begin{table}[!htbp]
\centering\small
\caption{\textbf{Q1: the two pitfalls are real and directionally
distinct.} Untrained Qwen3-4B, 32 synthetic rollouts per harness. Both
pitfalls exhibit the predicted sign on all three white-box harnesses.
$\Delta_{\mathrm{comp}}$: TC-RAG/AgentFold evict most aggressively (larger
$|\Delta|$); MemexRL exposes the most leaves
($n_{\text{comp}}\!=\!4000$) because its two-step offload delay keeps
each evicted span alive for two physical steps. $n_{\text{full}}$ is
identical across harnesses ($3072$) because all three evict before
the answer turn. \Cref{fig:q1_dist} plots the underlying distributions.}
\label{tab:q1}
\begin{tabular}{lrrrrrr}
\toprule
& \multicolumn{3}{c}{$\Delta_{\mathrm{comp}}$ (Naive-Comp vs.\ teacher)}
& \multicolumn{3}{c}{$\Delta_{\mathrm{full}}$ (Naive-Full-leak vs.\ current view)} \\
\cmidrule(lr){2-4}\cmidrule(lr){5-7}
Harness    & $n$   & mean          & frac.\ $<0$      & $n$   & mean          & frac.\ $>0$ \\
\midrule
MemexRL     & 4000  & $-1.62$       & $0.60$           & 3072  & $+0.68$       & $0.56$      \\
TC-RAG     & 928   & $-3.96$       & $0.78$           & 3072  & $+0.24$       & $0.71$      \\
AgentFold  & 928   & $-3.81$       & $0.75$           & 3072  & $+0.34$       & $0.55$      \\
\midrule
Predicted  &       & $<0$ (under)  & $\to 1$          &       & $>0$ (over)   & $\to 1$     \\
\bottomrule
\end{tabular}
\end{table}

\subsection{Method-input recipe under the same adapter record}
\label{app:method-inputs}
All five methods construct their training input from the same
per-rollout adapter record
$(\hcomp, \Hfull, \{\junct{k}, \evictset{k}\}_{k=1}^K)$, plus a No-comp
(Search-R1) control that disables the editor (so $\hcomp\equiv\Hfull$ and
its drift is the numerical floor). \Cref{tab:method-inputs} summarizes
the input, teacher-forward count, and extra bookkeeping. All five methods
roll out under \emph{live} compression; Naive-Full re-injects the evicted
text into the student input.

\begin{table}[t]
\centering\small
\caption{Method inputs from the single adapter record.
\textsc{cc\_tokens} = compressed walk, \textsc{un\_tokens} = union.}
\label{tab:method-inputs}
\begin{tabular}{lccc}
\toprule
Method & Student input & Teacher forwards? & Extra bookkeeping \\
\midrule
Naive-Compressed & $\hcomp$              & no  & --- \\
Naive-Full       & $\Hfull$              & no  & --- \\
LogitTree (ours)      & per-segment slices    & no  & $K$ segment boundaries \\
4D mask (ours)   & $\Hfull$              & no  & 4D attn + hole-free posids \\
SDCC (ours)      & $\hcomp$              & $K$ (no-grad) & $K$ reconstructed teacher prefixes \\
\bottomrule
\end{tabular}
\end{table}

\subsection{What \texttt{logdiff} certifies under live compression}
\label{app:logdiff-semantics}
Recall the four-quantity split of \S\ref{sec:setup} Metrics: (1)~the
recorded-token \texttt{logdiff}, (2)~the full-distribution conditioning
KL, (3)~SDCC's training KL $\mathrm{KL}_{\mathrm{SDCC}}$, (4)~agentic
EM. This subsection expands on (1). The \texttt{logdiff} statistic
compares, per
response token, the training-side re-forward under the method's training
conditioning with the log-probability recorded by the rollout engine at
generation time. Under the live-compression protocol, the latter is conditioned on
the \emph{compressed view that the model actually saw at that turn} ---
so the two sides are directly comparable, and \texttt{logdiff} is a
faithful per-step measure of the training--rollout conditioning gap
for the recorded token. As a diagnostic, it certifies conditioning
fidelity of the training forward, not end-task correctness; the
efficacy verdict is delivered by (4).
The residual gap Naive-Compressed incurs is now precisely identified: a
token emitted at turn $k$ was sampled under the turn-$k$ view, which
still contained material that a \emph{later} eviction removed;
re-forwarding the final compressed walk therefore conditions turn-$k$
tokens on a history under which they were never sampled. LogitTree ($K$-forward on
exact per-turn prompt token ids) and the 4D mask (one packed forward
realizing the same conditioning) eliminate this gap \emph{by
construction}, and their observed \texttt{logdiff} indeed returns to
(or within noise of) the no-compression baseline in every cell of
\Cref{tab:grand-matrix}. SDCC deliberately keeps the cheap
Naive-Compressed forward (and thus shares its elevated diagnostic on
mismatched batches) and closes the gap through the training-KL
objective (3); its progress is read off
$\mathrm{KL}_{\mathrm{SDCC}} =
\mathrm{KL}(\pi_{\mathrm{train}}(\cdot|H^{\mathrm{comp}})\,\|\,
\pi_{\mathrm{teacher}}(\cdot|H^{\mathrm{ctx}}))$ (the training loss
itself, \S\ref{sec:soft:form}) plus downstream EM (4). We therefore
read the matrix as: \texttt{logdiff} ranks conditioning fidelity of
the training forward (LogitTree $=$ 4D $=$ baseline $\ll$ Naive-Compressed,
by construction for the two exact walks), while SDCC-KL and EM rank
end-task correction quality.

\subsection{Training curves}
\label{app:train-curves}

This subsection plots the training dynamics behind the grand matrix.
The white-box curves come from the 4B RL runs on the pooled training
corpus, one curve per method$\times$harness cell; the editor-free
control enters this subsection only as a numeric reference in the text.
\Cref{fig:blackbox-reward-40} additionally reports reward traces from
the Claude Code and OpenCode black-box runs. The curves diagnose
optimization behaviour and are not themselves results:
\Cref{tab:grand-matrix} remains the sole source of final numbers.

\begin{figure}[t]
\centering
\includegraphics[width=\linewidth]{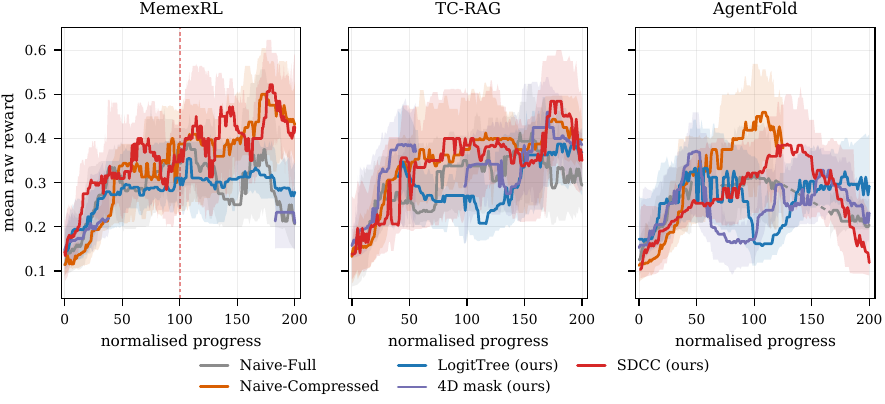}
\caption{\textbf{Mean raw reward over training, by harness.}
The axis is normalised progress: each cell's iterations are rescaled by
$x = 200\,\mathrm{iter}/\mathrm{iter}_{\max}$, with
$\mathrm{iter}_{\max}$ that cell's last logged iteration, so iteration
$0$ sits at $x=0$. One $x$-unit is therefore a different number of
iterations in each curve. Thick lines are a centred rolling median of
\texttt{rollout/raw\_reward} over a $\pm15$-iteration window, taken on
the true iteration grid before the rescale; the band is that window's
interquartile range.}
\label{fig:train-reward}
\end{figure}

\begin{figure}[t]
\centering
\begin{subfigure}[t]{0.495\linewidth}
\centering
\includegraphics[width=\linewidth]{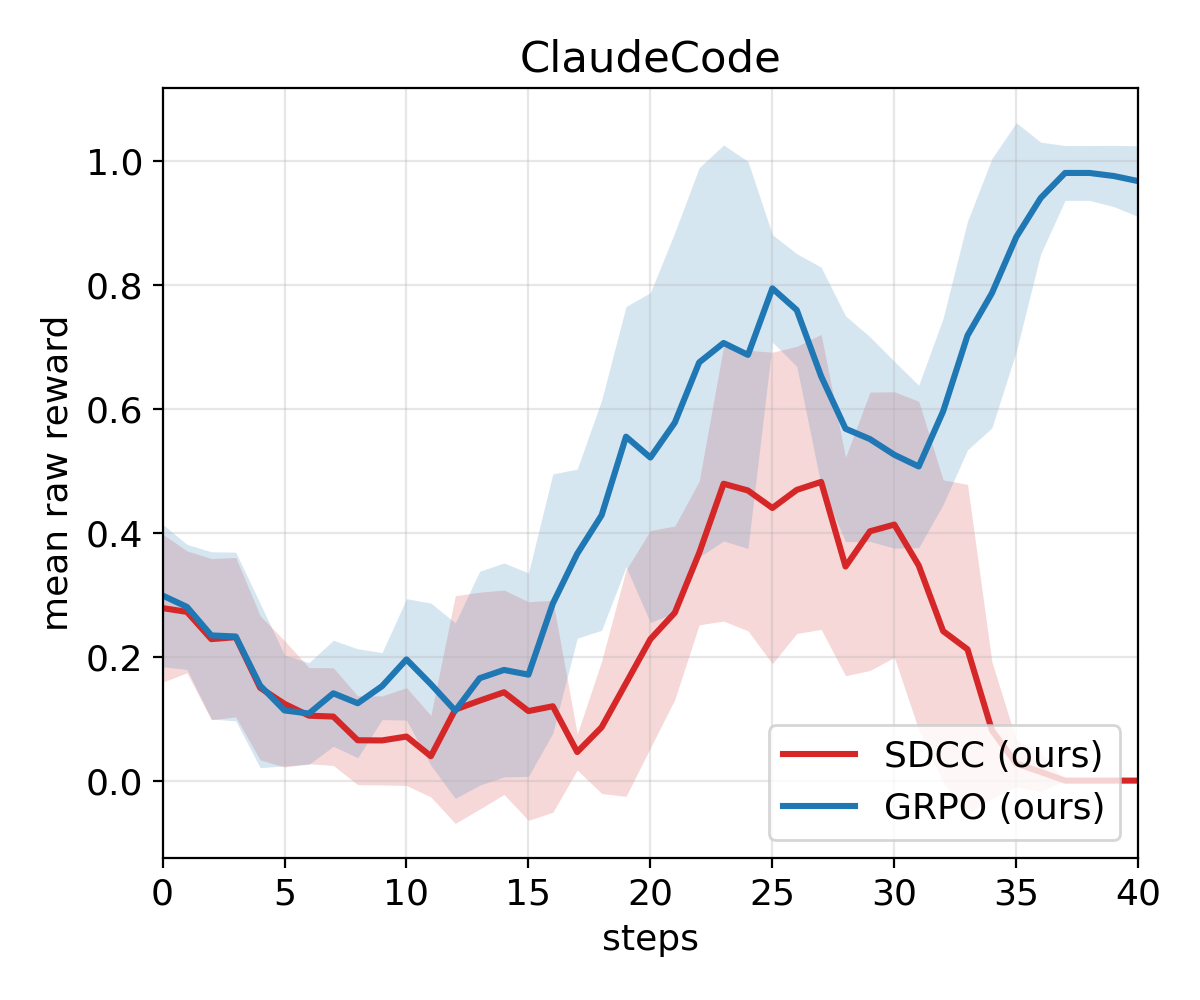}
\caption{Claude Code}
\end{subfigure}\hfill
\begin{subfigure}[t]{0.495\linewidth}
\centering
\includegraphics[width=\linewidth]{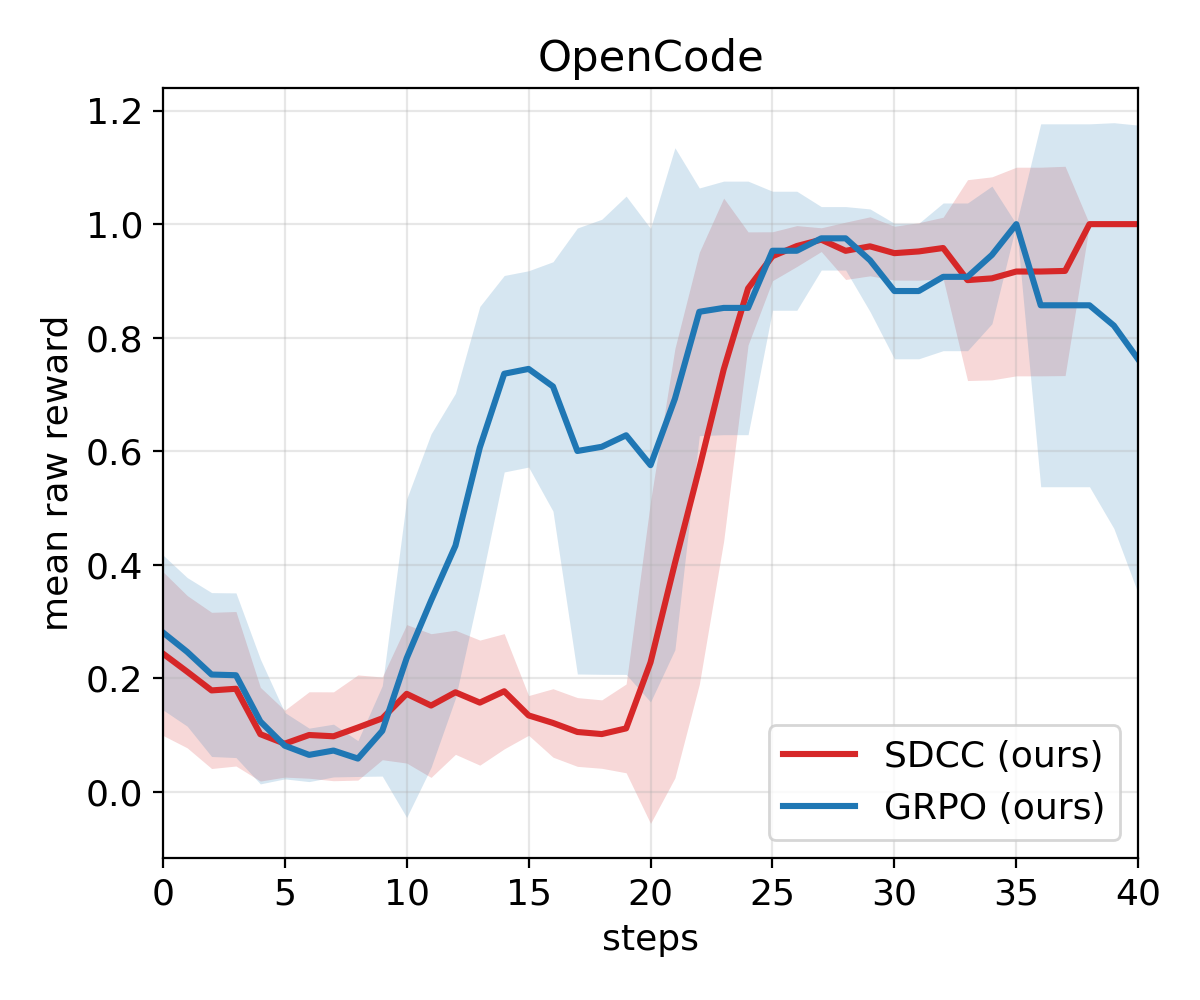}
\caption{OpenCode}
\end{subfigure}
\caption{\textbf{Black-box training reward in the first 40 steps.}
Mean raw reward for SDCC (red) and the GRPO baseline (blue); shaded
bands show the variability recorded by the training runs.}
\label{fig:blackbox-reward-40}
\end{figure}

\paragraph{Reward.} \Cref{fig:train-reward} shows all fifteen white-box cells
rising over the windows it draws, by between $0.07$ and $0.39$ from
first drawn point to peak, with the per-iteration signal noisy enough
that these curves cannot separate methods by eye. We deliberately draw
no ``best'' marker: the logged windows differ in length across cells,
and the rescaled axis aligns their endpoints rather than removing that
difference, so any endpoint contrast would in part reflect unequal
training length rather than method. Note also that
\texttt{rollout/rewards} is a post-normalization quantity, not a
reward: its per-cell median lies between $-7.5\times10^{-3}$ and
$-1.0\times10^{-8}$, i.e.\ centred on zero rather than on any reward
level, so it appears nowhere in this figure.

\Cref{fig:blackbox-reward-40} separately reports the black-box reward
traces for Claude Code and OpenCode over their first $40$ optimizer
steps. These panels use their native step grids rather than the
white-box normalized-progress axis. Their noisy, short windows and
overlapping variability bands make them descriptive training traces,
not a reward-based method ranking.

\paragraph{Conditioning drift.} \Cref{app:logdiff-semantics} states
what \texttt{logdiff} certifies; this paragraph states how we measure
it. Every number below is a pooled median over \emph{steady} steps,
defined by one rule applied uniformly to every cell: within each launch
we discard the first $20$ optimizer steps and retain only launches with
at least $10$ steps remaining. The discard is necessary because the
statistic reads more than an order of magnitude high until the trainer
and the rollout engine have re-aligned; conditioning on the
within-launch step index removes most but not all of that, and we
report the residual rather than claim it away.

Training length is a second confound, and it is the one that decides
the ordering. Because the editor-free control covers iterations up to
$58$, a median taken over each cell's own window would compare cells
trained for different lengths. We therefore fix the comparison window
in advance to the window the control covers --- steady steps at
iteration $\leq 58$, against the control's median there ($0.0138$,
$n{=}448$). Ten of the fifteen cells sit on the floor at
$0.93$--$1.04\times$ the control; Naive-Full's two eviction-heavy cells
sit just above it, both at $1.19\times$; and three cells are elevated:
Naive-Compressed$\times$AgentFold $2.17\times$,
Naive-Compressed$\times$TC-RAG $2.69\times$ and SDCC$\times$AgentFold
$4.51\times$. That is the ordering the main text predicts --- the two
eviction-heavy harnesses separate Naive-Compressed from the floor and
MemexRL does not ($0.99\times$) --- and SDCC is elevated by
construction, since it keeps the cheap compressed forward and corrects
through its KL term. Restricting
to the eleven cells that cover the window to within a single iteration
leaves the floor band at $0.94$--$1.04\times$ and moves no cell across
a group boundary.

On the cold-start MemexRL cells --- where compression barely fires ---
all five methods stay within $0.0115$--$0.0151$ over their full
windows (Naive-Compressed $0.0133$), bracketing the control's
$0.0138$, exactly as the eviction-density argument of
\Cref{app:sdcc-dynamics} requires. One caution on reading these
numbers against the rest of the paper: a median over per-iteration
medians, a median pooled over a cell's steady steps, and the drift
column of \Cref{tab:grand-matrix} are three different estimators of
the same quantity, none numerically interchangeable with another. Only
the pooled form is used here.

\subsection{SDCC training dynamics and eviction-density profile}
\label{app:sdcc-dynamics}

\paragraph{Q2 headline 5-method table.} \Cref{tab:main} provides the
5-method drift/EM comparison on the low-eviction MemexRL anchor cell at
4B; the drift column is measured under the live-compression
protocol, and the three daggered EM rows (Search-R1,
Naive-Compressed, Naive-Full) are populated from the same seven-bench
agentic runs that supply \Cref{tab:grand-matrix}. The three ours rows
(LogitTree/4D/SDCC) already show the drift ordering; cross-harness EM
support appears in \Cref{tab:grand-matrix}: on the eviction-heavy
harnesses, 4D matches or exceeds
Naive-Compressed EM (ahead $34.2$ vs.\ $33.2$ on AgentFold, ahead
$37.2$ vs.\ $36.7$ on TC-RAG) while its drift stays at the
no-compression floor ($0.017$--$0.021$ vs.\ Naive-Compressed's
$0.054$--$0.058$), so ours does not trade EM for drift consistency. MemexRL is the cold-start
regime ($8/256$ sessions evict), so Naive-Comp drift here is at the
\emph{floor} of the eviction-density scaling reported in
\Cref{tab:grand-matrix} (main text).

\begin{table}[t]
\centering\footnotesize
\setlength{\tabcolsep}{4pt}
\renewcommand{\arraystretch}{0.92}
\caption{\textbf{Q2: 5-method comparison, Qwen3-4B / MemexRL.} Drift from the
live-compression MemexRL cell; Search-R1 drift from the editor-free control
(numerical floor). EM columns are seven-bench macro-EM under the
agentic-eval protocol (\S\ref{app:agentic-eval}); $\dagger$ rows are
grand-matrix same-source runs (cross-harness EM in
\Cref{tab:grand-matrix}).}
\label{tab:main}
\begin{tabular}{lccc}
\toprule
Method              & EM $\uparrow$    & Drift $\downarrow$ & Cost ($\times$ base) \\
\midrule
Search-R1 (no-comp.) & $23.0^\dagger$   & 0.0135             & 1.05 \\
Naive-Compressed    & $28.9^\dagger$   & 0.0237             & 1.00 \\
Naive-Full          & $32.1^\dagger$   & 0.0163             & 1.05 \\
\textbf{LogitTree (ours)}    & $45.9^\dagger$ & \textbf{0.0133}    & 4.20 \\
\textbf{4D mask (ours)} & $33.4^\dagger$ & 0.0140             & 1.35 \\
\textbf{SDCC (ours)}    & $43.1^\dagger$ & 0.0165             & 1.55 \\
\bottomrule
\end{tabular}
\end{table}

\paragraph{Selectivity of the SDCC regularizer under live compression.}
\Cref{fig:sdcc-dynamics} tracks the on-line evolution of the SDCC
regularizer on the AgentFold cell --- the eviction-heavy editor
(length-triggered fold at 3{,}000 tokens; compression ratio
$0.06$--$0.15$), and hence the cell in which the KL contract has the
most work to do. Four regularities emerge.
\begin{enumerate}[leftmargin=*]
\item[\textbf{(i)}] \textbf{The KL fires sparsely and precisely.}
$\mathrm{KL}_{\mathrm{SDCC}}$ is non-zero on slightly over half of the
logged optimizer steps and exactly zero elsewhere, and the non-zero
steps are exactly those whose batches contain live folds. Conditioning
on the gate separates the diagnostic by a factor of $5.6$ in median:
steps with $\mathrm{KL}_{\mathrm{SDCC}}>0$ have median \texttt{logdiff}
$0.081$ (IQR $0.058$--$0.125$), whereas steps with
$\mathrm{KL}_{\mathrm{SDCC}}=0$ sit at the no-compression baseline,
median $0.015$ (IQR $0.014$--$0.017$). The regularizer engages at
mismatched positions and only at mismatched positions --- the per-leaf
gating of \S\ref{sec:soft} operating as designed.
\item[\textbf{(ii)}] The contract stays armed for the duration of the
run rather than saturating or dying: the diverging-leaf set is
non-empty on a majority of micro-batches throughout, so the KL term
never becomes structurally inert.
\item[\textbf{(iii)}] The SDCC KL residual does not grow over the run,
while $\mathrm{KL}_{\mathrm{ref}}$ against the reference model grows
smoothly from $0.002$ to $\approx\!0.03$: the policy moves but does not
detach from the reference, and no run-away is observed without an
entropy bonus.
\item[\textbf{(iv)}] Mean raw reward rises over the same window in
which the training-KL residual does not grow --- the training-time
consistency signal and the eval-time task signal are not in tension ---
which motivates the head-to-head against Naive-Compressed in
\Cref{tab:main}: absent SDCC's regularizer, the same policy optimizes
against physical logits it never trains on. The load-bearing part of
that comparison is structural rather than numerical: the two arms share
the same policy-loss code path modulo the KL term, i.e.\
Naive-Compressed $=$ SDCC at $\lambda\!=\!0$ (\S\ref{sec:exp-q2}).
Magnitudes here come from one editor and are not offered as
cross-method effect sizes.
\end{enumerate}
This cell is cold-start: the untrained 4B policy triggers evictions in
only a fraction of sessions, so the KL magnitudes are small in absolute
terms. The point established here is not the magnitude but the
\emph{selectivity} --- $\epsKL$ is confined to exactly the diverging
leaves, so the $\Order(\sqrt{\epsKL})$ bound of \S\ref{sec:soft:bound}
is non-vacuous from the first steps of training.
\Cref{fig:sdcc-dynamics} plots the co-evolving signals: (a) the sparse
KL spike train; (b) the diverging-leaf activity and the effective KL
weight; (c) \texttt{logdiff} spiking on eviction batches and returning
to baseline elsewhere, with reward drifting up.

\begin{figure}[t]
\centering
\includegraphics[width=\linewidth]{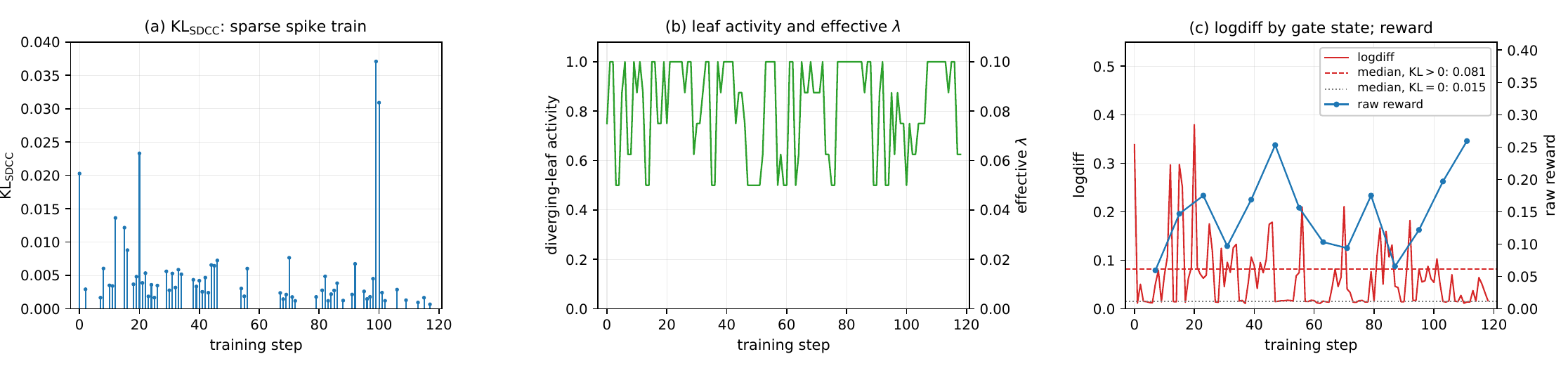}
\caption{\textbf{SDCC dynamics on Qwen3-4B $\times$ AgentFold under live
compression.} (a) $\mathrm{KL}_{\mathrm{SDCC}}$ is a sparse spike train,
non-zero on slightly over half of the logged optimizer steps and exactly
zero elsewhere. (b) Diverging-leaf activity and the effective KL weight
$\lambda$. (c) \texttt{logdiff} rises on eviction batches (median $0.081$) and
sits at the no-compression baseline otherwise (median $0.015$); mean
reward drifts upward.}
\label{fig:sdcc-dynamics}
\end{figure}

\paragraph{Per-method mechanism activation under live compression.}
\Cref{tab:5method-channels} reports, for each method, the run-mean of
the quantity that its own corrective mechanism controls, per white-box
editor. This is not an outcome (EM) table: it establishes that each
mechanism engages at 4B scale under \emph{live}, model-triggered
compression, and quantifies how strongly. Two things stand out.
\begin{enumerate}[leftmargin=*]
\item[\textbf{(i)}] For every editor, each corrective method's own channel is non-zero:
LogitTree splits each rollout into $\bar K = 3.2$--$4.1$
segments, the 4D packer activates on $1.2$--$1.4$ sub-samples per
micro-batch, and SDCC's leaf gate is armed on a large majority of
micro-batches. The Naive rows have no such channel by construction.
\item[\textbf{(ii)}] SDCC's KL magnitude tracks the editor's eviction
density: $\approx\!10^{-5}$ on MemexRL (8/256 sessions evict at
cold start), $1.1\times 10^{-4}$ on TC-RAG, and $2.9\times 10^{-3}$ on
AgentFold (132/960 sessions fold; compression ratio $0.10$). The
correction therefore scales precisely with how much the editor actually rewrites
--- consistent with the per-leaf gating derivation of \S\ref{sec:soft}.
\end{enumerate}

\paragraph{Persistence of Pitfall A under training.} A natural hope is
that the naive-compressed drift will self-correct as GRPO adapts the policy
to the compressed prefixes. The AgentFold cell gives no sign of that
correction: over the logged window the per-third mean of
\texttt{logdiff} moves
$0.069\!\to\!0.045\!\to\!0.064$, and the count of
eviction-heavy steps (\texttt{logdiff}$>0.05$) does not fall --- the
mismatch neither shrinks nor is absorbed, because its source (folds keep
firing on new rollouts) remains stationary. On the same cell,
LogitTree's \texttt{logdiff} stays pinned at $0.013$--$0.015$ in every
third: the exact correction removes the drift at its source rather than
asking the optimizer to absorb it. This window is early-training and
therefore low-eviction; at the argmax-reward iteration of the full run
the same Naive-Compressed $\times$ AgentFold cell reads $0.366$
(\Cref{tab:grand-matrix}). The persistence claim is thus made at the
\emph{conservative} end of the density range.

\begin{figure}[!htb]
\centering
\includegraphics[width=0.65\linewidth]{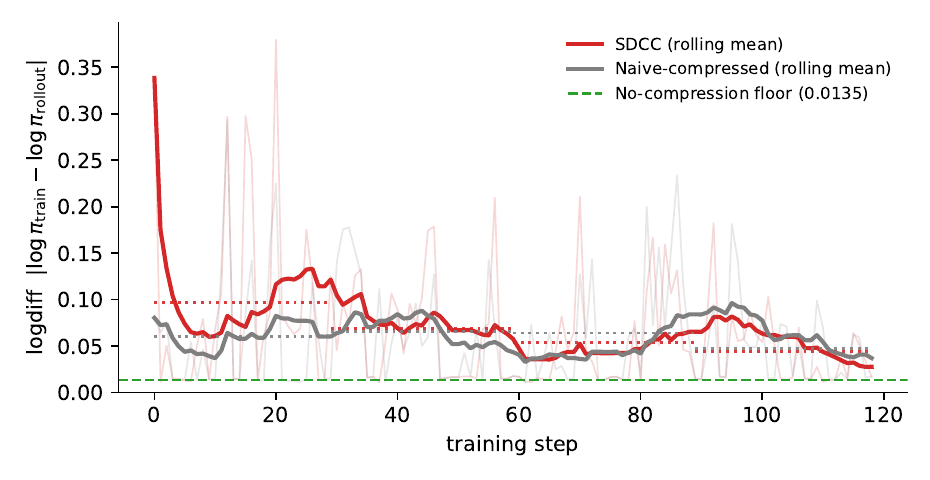}
\caption{\textbf{SDCC's conditioning gap decreases over the logged
training window; Naive-Compressed's does not.} Per-step
\texttt{logdiff} on Qwen3-4B~$\times$~AgentFold under live compression.
SDCC (red) trends toward the no-compression
floor (green dashed, $0.0135$); Naive-Compressed (grey) drifts without
directional trend; LogitTree stays pinned at $0.013$--$0.015$.}
\label{fig:sdcc-convergence}
\end{figure}

\paragraph{SDCC's \texttt{logdiff} decreases over the logged window.}
SDCC shares Naive-Compressed's inexpensive student forward, so both begin with
an elevated diagnostic --- but only SDCC includes a mechanism that should
\emph{shrink} it. \Cref{fig:sdcc-convergence} tests this on
the same eviction-heavy AgentFold cell: SDCC's mean \texttt{logdiff}
falls by $55\%$ from the start to the end of the logged window and does
so monotonically across successive quarters, whereas
Naive-Compressed's shows no comparable directional trend. SDCC starts
$60\%$ \emph{above} Naive-Compressed (its
leaf-gated KL initially perturbs exactly the mismatched positions) and
ends \emph{below} it, still descending toward the no-compression floor
of $0.0135$ at the end of the logged window. This is one editor at one
scale; it is the qualitative counterpart of the ordering
predicted in the main text --- LogitTree/4D own the floor by construction,
SDCC approaches it by optimization, and Naive-Compressed has no route
there at all --- and not a trend estimate across seeds.

\begin{table}[!htb]
\centering\small
\caption{\textbf{Live eviction-density profile per editor.}
Trigger and density statistics; drift from the Naive-Compressed cell; KL from
the SDCC cell (mean $\times 10^{-3}$). Both diagnostics scale
monotonically with eviction density.}
\label{tab:phase2-longrun}
\resizebox{\linewidth}{!}{%
\begin{tabular}{l|l|cc|c|c}
\toprule
& & \multicolumn{2}{c|}{eviction density} & Naive-Comp & SDCC KL \\
Editor & Trigger & sess.\ w/ evict & ratio & mean drift & $\times 10^{-3}$ \\
\midrule
MemexRL    & model tool (\texttt{memory\_offload}) & 8/256   & 0.012 & 0.024 & 0.01 \\
TC-RAG    & model tool (\texttt{pop})             & 17/320  & 0.024 & 0.039 & 0.11 \\
AgentFold & length fold (3k tokens)               & 132/960 & 0.104 & 0.059 & 2.86 \\
\bottomrule
\end{tabular}}
\end{table}

\begin{table}[!htb]
\centering\small
\caption{\textbf{Per-method mechanism activation under live
compression.}
Run-mean signature channels per white-box editor. \texttt{seg} = LogitTree
segments per rollout; \texttt{pack} = 4D packed sub-samples per
micro-batch; $\mathrm{KL}$ = SDCC leaf-gated KL ($\times 10^{-3}$);
\texttt{act} = fraction of micro-batches with a non-empty
diverging-leaf set. \texttt{---} = channel not applicable to that
method.}
\label{tab:5method-channels}
\begin{tabular}{l|l|ccc}
\toprule
Method & Channel & MemexRL & TC-RAG & AgentFold \\
\midrule
Naive-Full        & ---                    & --- & --- & --- \\
Naive-Compressed  & ---                    & --- & --- & --- \\
\textbf{LogitTree (ours,} $K$\textbf{-forward)} & \texttt{seg}           & 4.10 & 3.90 & 3.23 \\
\textbf{4D mask (ours, packed)} & \texttt{pack}          & 1.44 & 1.28 & 1.16 \\
\textbf{SDCC (ours)} & $\mathrm{KL}$ / \texttt{act} & 0.01 / 0.99 & 0.11 / 0.87 & 2.86 / 0.82 \\
\bottomrule
\end{tabular}
\end{table}

\paragraph{Compression frequency and depth.}
\label{app:compression-metrics}
The $\mathrm{depth}$ column of \Cref{tab:grand-matrix} and its
companion factor $\mathrm{freq}=\Pr[c_i\!>\!0]$ are both measured per
episode over each cell's full training run; their product is the single
compression rate a one-number summary would report. We keep
$\mathrm{freq}$ out of the main table because it is not one quantity
across columns: AgentFold folds on a length threshold, so its
$\mathrm{freq}$ counts episodes that outgrow $3$k rendered tokens,
whereas TC-RAG and MemexRL evict only when the policy emits
\texttt{memory\_offload} or \texttt{pop}, so there $\mathrm{freq}$
counts a policy decision and not a length --- comparable down a column
but not across. $\mathrm{depth}$ has one meaning everywhere. Across the
fifteen trained cells $\mathrm{freq}$ varies by more than an order of
magnitude while $\mathrm{depth}$ varies by under $2\times$, so the
product is close to a monotone rescaling of frequency and inherits its
harness-dependence, whereas the two factors reported apart show what it
hid: whichever editor fires removes about three fifths of the rendered
context. Since $\mathrm{depth}$ conditions on the episodes that evict,
it is \textsc{n/a}, not $0$, for the no-compression control.

\paragraph{Eviction-density profile across editors.} The three
white-box editors span two orders of magnitude in live eviction density
(\Cref{tab:phase2-longrun}): MemexRL and TC-RAG are
\emph{model-triggered} (the policy must decide to call
\texttt{memory\_offload} / \texttt{pop}), so from a cold start the
untrained 4B model rarely invokes them, whereas AgentFold is
\emph{length-triggered} and fires on every long rollout, folding
physical contexts of up to $121$k tokens into $\leq\!5.2$k-token
logical views. The low-density cells serve as in-run negative controls:
as eviction density falls, Naive-Compressed's drift and SDCC's KL both
collapse toward the no-compression baseline --- confirming that the
pitfall is compression-driven, not an artifact of harness dispatch.

\section{Implementation note: branch-replicated packing vs.\ the physical-union view}
\label{app:packed-vs-union}

The main text (\S\ref{sec:hard:4d}, Eq.~\ref{eq:4d-mask-def}) presents
the 4D mask as a per-row visibility schedule over the physical union
$\Hfull$, in which each token appears once and different queries see
different subsets of the same hidden states.  Our implementation instead
follows LogitTree's branch decomposition: it materializes one
\emph{copy} of each shared-trunk token per branch, so that every copy
carries a hidden state computed from its branch-local causal prefix
alone.  The copies are then packed into a single sequence with a
\emph{block-diagonal} attention mask that prevents cross-branch
interaction; position ids restart per branch.

Under dense softmax attention and the five conditions of
\Cref{prop:sft_4d}, the two constructions produce identical per-target
logits and gradients: the physical-union view is the row-wise
characterization of what the block-packed execution computes.  We retain
the physical-union formulation in the main text because it gives a
compact, closed-form mask definition (Eq.~\ref{eq:4d-mask-def}) and a
direct proof path (\S\ref{app:proof-prop1}); the branch-replicated form
is what the training loop executes.

The token budget in \Cref{tab:hard_compare} reports $L$ (the physical
union length) as a lower bound; the actual packed length is
$N_{\mathrm{packed}} = \sum_{\pi}\ell_\pi \ge L$ due to trunk
replication, and \Cref{tab:grand-matrix}'s ``max~tok.'' column reflects
the true packed input size.

\end{document}